\documentclass[11pt]{article}
\usepackage{xcolor}
\usepackage{longtable}
\usepackage{supertabular}
\usepackage{pspicture}
\usepackage{filecontents}
\usepackage{pst-bar}
\usepackage{bm,amsfonts,amssymb,amsmath,float}
\usepackage{sectsty}
\usepackage{colortbl}
\usepackage{multirow}
\sectionfont{\large}\subsectionfont{\normalsize}
\usepackage[sort&compress,square,numbers]{natbib}
\usepackage{latexsym}
\usepackage{epsfig}
\usepackage{epstopdf}
\usepackage{threeparttable}
\usepackage{caption}
\usepackage{enumerate}
\usepackage{diagbox}
\usepackage{ulem}
\usepackage{bm}
\usepackage{mathtools}
\usepackage{enumitem}
\usepackage{algorithm}
\usepackage{algpseudocode}
\usepackage{afterpage}
\usepackage{bbm}
\usepackage{algpseudocodex}
\usepackage{placeins}
\usepackage{makecell}
\usepackage{appendix}

\newenvironment{proof}{\noindent{\bf Proof:}}{\hfill\fbox{}\vspace*{1mm}}
\newtheorem{theorem}{Theorem}[section]
\newtheorem{example}[theorem]{Example}
\newtheorem{lemma}[theorem]{Lemma}

\newtheorem{proposition}[theorem]{Proposition}
\newtheorem{corollary}[theorem]{Corollary}

\DeclareMathOperator*{\argmax}{argmax}
\DeclareMathOperator*{\argmin}{argmin}

\makeatletter
\newcommand\nextitem[1]{%
  \setcounter{\@enumctr}{#1}%
  \addtocounter{\@enumctr}{-1}%
}
\makeatother

\usepackage{ulem} 
\usepackage[mathlines, displaymath]{lineno} 

\begin{document}
\renewcommand{\labelenumi}{\arabic{enumi}.}
\renewcommand{\labelenumii}{\arabic{enumi}.\arabic{enumii}}
\renewcommand{\labelenumiii}{\arabic{enumi}.\arabic{enumii}.\arabic{enumiii}}
\renewcommand{\labelenumiv}{\arabic{enumi}.\arabic{enumii}.\arabic{enumiii}.\arabic{enumiv}}

\title{\bf A Discrete Perspective Towards the Construction of Sparse Probabilistic Boolean Networks}
\author{Christopher H. Fok 
\thanks{Advanced Modeling and Applied Computing Laboratory,
Department of Mathematics, 
The University of Hong Kong, Pokfulam Road, Hong Kong. 
E-mail: christopherfok2015@outlook.com
}
\and Chi-Wing Wong
\thanks{Advanced Modeling and Applied Computing Laboratory,
Department of Mathematics, The University of Hong Kong, Pokfulam Road, Hong Kong. 
E-mail: cwwongab@hku.hk. }
\and Wai-Ki Ching
\thanks{Advanced Modeling and Applied Computing Laboratory,
Department of Mathematics, The University of Hong Kong, Pokfulam Road, Hong Kong. 
E-mail: wching@hku.hk. Research supported in part by
the Hong Kong Research Grants Council under Grant Numbers 17301519 and 17309522.}
}
\date{12\textsuperscript{th} July 2024}
\maketitle


\begin{abstract}
\noindent
Boolean Network (BN) and its extension Probabilistic Boolean Network (PBN) are popular mathematical models for studying genetic regulatory networks. 
BNs and PBNs are also applied to model manufacturing systems, financial risk and healthcare service systems.
In this paper, we propose a novel Greedy Entry Removal (GER) algorithm
for constructing sparse PBNs.
We derive theoretical upper bounds for both existing algorithms and the GER algorithm.
Furthermore, we are the first to study the lower bound problem of the construction of sparse PBNs, and to derive a series of related theoretical results.
In our numerical experiments based on both synthetic and practical data, GER gives the best performance among state-of-the-art sparse PBN construction algorithms and outputs sparsest possible decompositions on most of the transition probability matrices being tested.
\end{abstract}


\section{Introduction}\label{section:intro}

\subsection{Boolean Networks}\label{subsection:BNs}
Boolean Network (BN) is a mathematical model that has numerous applications. 
It was first proposed by Kauffman in 1969 for studying genetic regulatory networks \cite{metabolic} (see also \cite{homeostasis, origins_of_order}). 
Since then, BN models have also been applied to study biological systems and mechanisms such as apoptosis, the yeast cell-cycle network, T cell receptor signaling (see, for instance \cite{bornholdt_cellular_regulation, yeast, 
T_cell, apoptosis}).

Concerning the formal mathematical details of BNs, 
as noted in \cite{shmulevich_PBN, from_BN_to_PBN}, a BN $(V, F)$ consists of a set of nodes
$V = \{v_{1} , v_{2} , \ldots , v_{n}\}$ and a list of Boolean functions $F = (f_{1}, f_{2}, \ldots , f_{n} )$, where each $f_{i} : \{0, 1\}^{n} \to \{0, 1\}$. 
At any time $t \in \mathbb{Z}_{\geq 0}$, each node $v_{i}$ is in either the ``on'' state (denoted by $1$) or the ``off'' state (denoted by $0$). We denote the state of $v_{i}$ at time $t$ by $v_{i}(t) \in \{ 0, 1\}$. The global state of the BN at time $t \in \mathbb{Z}_{\geq 0}$ is defined to be $\mathbf{v}(t) \coloneqq \left( v_{1}(t), v_{2}(t), \ldots, v_{n}(t) \right) \in \{0, 1\}^{n}$. We remark that there are $2^{n}$ possible global states, and that $\mathbf{v}(0)$ is the initial global state of the BN.
In a BN, at any time $t \in \mathbb{Z}_{\geq 0}$, the subsequent state of each node $v_{i}$ (i.e., $v_{i}(t + 1)$) is determined by the current global state $\mathbf{v}(t)$ of the BN according to the following rule: 
$v_{i} (t + 1) = f_{i}(\mathbf{v}(t))$. 
It is easy to see that if we fix $\mathbf{v}(0)$, then $\mathbf{v}(1), \mathbf{v}(2), \ldots$ will all be determined.

Each $n$-node BN is associated with a special $2^{n} \times 2^{n}$ matrix $A$ (called a \textit{BN matrix}). Label the $2^{n}$ possible global states of the BN as $1, 2, 3, \ldots, 2^{n}$. Then, the BN matrix $A$ is defined as follows: for all $1 \leq i, j \leq 2^{n}$, $A(i, j) = \Pr\left( \mathbf{v}(t+1) = \textup{state}\;i\;|\; \mathbf{v}(t) = \textup{state}\;j \right)$. We remark that each column of $A$ has exactly one non-zero entry and this non-zero entry equals $1$.
For any positive integer $x$, let $[x] \coloneqq \{ 1, 2, \ldots, x \}$. Also, let $B_{n}$ be the set of all $2^{n}\times 2^{n}$ BN matrices $\left\{ [\vec{e}_{k_{1}},  \vec{e}_{k_{2}}, \vec{e}_{k_{3}}, \ldots, \vec{e}_{k_{2^{n}}}] \in \mathbb{R}^{2^{n} \times 2^{n}} : k_{1}, k_{2}, k_{3}, \ldots, k_{2^n}  \in [2^{n}] \right\}$, where $\vec{e}_{m} \in \mathbb R^{2^{n}}$ is the $2^{n}$-dimensional standard unit column vector whose $m$-th entry equals $1$. We remark that there is a one-to-one correspondence between the set of all $n$-node BNs and $B_{n}$.

\begin{table}[h]
\begin{center}
\begin{tabular}{|c|cc|c|c|}
\hline BN Global States & $v_1 (t)$     & $v_2(t)$   & $f_{1}$ & $f_{2}$  \\
\hline
1 & 0 & 0 & 0 & 1 \\
2 & 0 & 1 & 1 & 0 \\
3 & 1 & 0 & 0 & 0 \\
4 & 1 & 1 & 1 & 1\\
\hline
\end{tabular}
\caption{The truth table for the BN introduced in Example \ref{example:BN}.}
\label{table:example_BN}
\end{center}
\end{table}

\FloatBarrier

\begin{example}\label{example:BN}
We give an example of a 2-node BN whose truth table is shown in Table \ref{table:example_BN}. From the table, there are four possible global states, namely $(0,0), (0,1), (1,0), (1,1)$. We label them as states $1,2,3$ and $4$ respectively.

If the current global state $\mathbf{v}(t)$ of the BN is state $1$,
then the subsequent global state $\mathbf{v}(t+1)$ will be state $2$.
If $\mathbf{v}(t)$ is state $2$, then $\mathbf{v}(t+1)$ will be state $3$.
Similarly, state $3$ moves to state $1$ while state $4$ remains unchanged.

The BN matrix associated with this BN is thus
\begin{equation}\label{eq:example_BN_matrix}
A =\left(
\begin{array}{cccc}
0 & 0 & 1 & 0 \\
1 & 0 & 0 & 0 \\
0 & 1 & 0 & 0 \\
0 & 0 & 0 & 1 \\
\end{array}
\right).
\end{equation}
\end{example}

\subsection{Probabilistic Boolean Networks}

Probabilistic Boolean Network (PBN) is a stochastic extension of the BN model. PBN was first introduced by Shmulevich et al. in 2002 for studying genetic regulatory networks \cite{from_BN_to_PBN}. 
PBN models have also been applied to a wide range of other areas including biomedicine \cite{biomedical}, credit default modeling \cite{credit_defaults}, manufacturing process modeling \cite{manufacturing_systems} and the study of brain connectivity \cite{brain_connectivity}.

Concerning the formal mathematical details of PBNs, an $n$-node PBN $\mathcal{P}$ is an $(n+2)$-tuple $(V, F_{1}, F_{2}, \ldots, F_{n}, \mathcal{D})$, where $V = \{ v_{1}, v_{2}, \ldots, v_{n}\}$ is the set of nodes of $\mathcal{P}$, $F_{i} = \left( f^{(i)}_{1}, f^{(i)}_{2}, \ldots, f^{(i)}_{l(i)} \right)$ is a non-empty list of $\ell(i)$ distinct Boolean functions mapping from $\{ 0, 1 \}^{n}$ to $\{ 0, 1 \}$, and $\mathcal{D}$ is a probability distribution on the Cartesian product of sets $\prod^{n}_{i = 1} [l(i)]$.
At any time $t \in \mathbb{Z}_{\geq 0}$, the global state of $\mathcal{P}$ at time $t \in \mathbb{Z}_{\geq 0}$ is defined to be $\mathbf{v}(t) \coloneqq \left( v_{1}(t), v_{2}(t), \ldots, v_{n}(t) \right) \in \{0, 1\}^{n}$. 
In order to generate $\mathbf v(t+1)$, a tuple $(j_{1}, j_{2}, \ldots, j_{n})$ is randomly sampled from the set $\prod^{n}_{i = 1} [l(i)] $ according to $\mathcal{D}$ and is used to define a list of Boolean functions $F(t) \coloneqq \left( f^{(1)}_{j_{1}}, f^{(2)}_{j_{2}}, \ldots, f^{(n)}_{j_{n}} \right)$. Then, the BN model $(V, F(t))$ is used to update the states of the nodes as if $\mathbf v(t)$ is in $(V,F(t))$: for each $i = 1, 2, \ldots, n$, $v_{i}(t + 1) = f^{(i)}_{j_{i}} (\mathbf{v}(t))$. We remark that the random samplings from $\prod^{n}_{i = 1} [l(i)] $ at all times $t \in \mathbb{Z}_{\geq 0}$ are mutually independent.

We label the $2^{n}$ global states of $\mathcal{P}$ in $\{ 0, 1 \}^{n}$ with integers from $1$ to $2^{n}$. Let $\mathbf{f}$ be an arbitrary element of the set $\mathcal{F} \coloneqq \prod^{n}_{i=1} \left\{ f^{(i)}_{1}, f^{(i)}_{2}, \ldots, f^{(i)}_{l(i)} \right\}$. Let $x_{\mathbf{f}}$ be the selection probability of the element $\mathbf{f}$, which is constant across all times $t \in \mathbb{Z}_{\geq 0}$. 
Then, at any time $t$, for any $i, j \in [2^{n}]$, the transition probability from the global state $j$ to the global state $i$ of the PBN $\mathcal{P}$ is given by
\begin{eqnarray}
P_{ij} & \coloneqq &
\Pr\left( \mathbf{v}(t+1) = \textrm{state}\:i\:| \: \mathbf{v}(t) = \textrm{state} \: j \right)
\nonumber \\
& = &
\sum_{\mathbf{f} \in \mathcal{F}} \Pr \left( \mathbf{v}(t+1) = \textrm{state} \: i \: | \: \mathbf{v}(t) = \textrm{state} \: j, F(t) = \mathbf{f} \right) \cdot x_{\mathbf{f}}.
\end{eqnarray}
Define the $2^{n} \times 2^{n}$ matrices $P \coloneqq (P_{ij})$ and $A_{\mathbf{f}}=\left(\Pr \left( \mathbf{v}(t+1) = \textrm{state} \: i \: | \: \mathbf{v}(t) = \textrm{state} \: j, F(t) = \mathbf{f} \right)\right)$. Note that $A_{\mathbf{f}}$ is the BN matrix of the BN $(V, \mathbf{f})$ for all $\mathbf{f} \in \mathcal{F}$ and
\begin{equation}
P = \sum_{\mathbf{f} \in \mathcal{F}} x_{\mathbf{f}} A_{\mathbf{f}}.
\end{equation}
It is important to note that all entries of $P$ are non-negative, and that in each column of $P$, the entries sum up to $1$. We call $P$ \textit{the PBN matrix} of the underlying PBN $\mathcal{P}$.

\begin{table}[h]
\begin{center}
\begin{tabular}{|c|cc|cc|cc|}
\hline PBN Global States & $v_1 (t)$     & $v_2(t)$   & 
            $f^{(1)}_{1}$ & $f^{(1)}_{2}$ & $f^{(2)}_{1}$ & $f^{(2)}_{2}$  \\
\hline
1 & 0 & 0 & 0 & 0 & 1 & 1 \\
2 & 0 & 1 & 1 & 0 & 0 & 1 \\
3 & 1 & 0 & 0 & 1 & 0 & 1 \\
4 & 1 & 1 & 1 & 1 & 0 & 0 \\
\hline
$c^{(i)}_{j}$ &  &  & 0.1 & 0.9 & 0.3 & 0.7 \\
\hline
\end{tabular}
\caption{The truth table for the PBN introduced in Example \ref{example:PBN}, together with the selection probabilities of the Boolean functions in this PBN.}
\label{table:example_PBN}
\end{center}
\end{table}

\FloatBarrier

\begin{example}\label{example:PBN}
We give an example of a $2$-node PBN, the truth table of which is shown in Table \ref{table:example_PBN}.

The PBN has four global states $(0, 0)$, $(0, 1)$, $(1, 0)$, $(1, 1)$ (labelled as 1, 2, 3, 4 respectively). 
Node $v_{1}$ is associated with the list of Boolean functions $F_{1} = ( f^{(1)}_{1}, f^{(1)}_{2} )$, whereas node $v_{2}$ is associated with the list of Boolean functions $F_{2} = ( f^{(2)}_{1}, f^{(2)}_{2} )$.
Let $\mathbf{f}_{1} \coloneqq (f^{(1)}_{1}, f^{(2)}_{1})$, $\mathbf{f}_{2} \coloneqq (f^{(1)}_{1}, f^{(2)}_{2})$, $\mathbf{f}_{3} \coloneqq (f^{(1)}_{2}, f^{(2)}_{1})$ and $\mathbf{f}_{4} \coloneqq (f^{(1)}_{2}, f^{(2)}_{2})$. 
In this PBN, at any time $t \in \mathbb{Z}_{\geq 0}$, the random selection of a Boolean function from $F_{1}$ is independent of the random selection of a Boolean function from $F_{2}$, and the probability that each $f^{(i)}_{j}$ is selected is $c^{(i)}_{j}$. 
Therefore, at any time $t \in \mathbb{Z}_{\geq 0}$, the probabilities of the selection of $\mathbf{f}_{1}$, $\mathbf{f}_{2}$, $\mathbf{f}_{3}$ and $\mathbf{f}_{4}$ are $x_{\mathbf{f}_{1}} = 0.1 \cdot 0.3 = 0.03$, $x_{\mathbf{f}_{2}} = 0.1 \cdot 0.7 = 0.07$, $x_{\mathbf{f}_{3}} = 0.9 \cdot 0.3 = 0.27$ and $x_{\mathbf{f}_{4}} = 0.9 \cdot 0.7 = 0.63$, respectively. 
In addition, the BN matrices of the BNs $(V, \mathbf{f}_{1})$, $(V, \mathbf{f}_{2})$, $(V, \mathbf{f}_{3})$ and $(V, \mathbf{f}_{4})$ are respectively given by
\begin{equation*}\label{eq:examplePBN_BN_matrices}
A_{\mathbf{f}_{1}} =\left(
\begin{array}{cccc}
0 & 0 & 1 & 0 \\
1 & 0 & 0 & 0 \\
0 & 1 & 0 & 1 \\
0 & 0 & 0 & 0 \\
\end{array}
\right),
A_{\mathbf{f}_{2}} =\left(
\begin{array}{cccc}
0 & 0 & 0 & 0 \\
1 & 0 & 1 & 0 \\
0 & 0 & 0 & 1 \\
0 & 1 & 0 & 0 \\
\end{array}
\right),
A_{\mathbf{f}_{3}} =\left(
\begin{array}{cccc}
0 & 1 & 0 & 0 \\
1 & 0 & 0 & 0 \\
0 & 0 & 1 & 1 \\
0 & 0 & 0 & 0 \\
\end{array}
\right),
A_{\mathbf{f}_{4}} =\left(
\begin{array}{cccc}
0 & 0 & 0 & 0 \\
1 & 1 & 0 & 0 \\
0 & 0 & 0 & 1 \\
0 & 0 & 1 & 0 \\
\end{array}
\right).
\end{equation*}
Therefore, the PBN matrix of the $2$-node PBN is given by
\begin{equation}\label{eq:example_PBN_matrix}
P = \sum^{4}_{i = 1} x_{\mathbf{f}_{i}} A_{\mathbf{f}_{i}} =\left(
\begin{array}{cccc}
0 & 0.27 & 0.03 & 0 \\
1 & 0.63 & 0.07 & 0 \\
0 & 0.03 & 0.27 & 1 \\
0 & 0.07 & 0.63 & 0 \\
\end{array}
\right).
\end{equation}

\end{example}

\subsection{Problem Formulation}\label{subsection:problem_formulation}
Before we describe the problem being studied in this paper, we first introduce some mathematical notations. 
Let $P$ be a $2^{n} \times 2^{n}$ non-negative matrix such that each column of $P$ sums to $1$. From now on, we will call such matrices $2^{n} \times 2^{n}$ \textit{transition probability matrices} (TPMs).
For each $j \in [2^{n}]$, let $D_{j}(P) \coloneqq \{ k \in [2^{n}] : P(k, j) > 0 \}$ and define $B_{n}(P) \coloneqq \left\{ [\vec{e}_{k_{1}},  \vec{e}_{k_{2}}, \vec{e}_{k_{3}}, \ldots, \vec{e}_{k_{2^{n}}}] \in \mathbb{R}^{2^{n} \times 2^{n}} : \forall j \in [2^{n}], k_{j} \in D_{j}(P) \right\}$, which is a subset of $B_{n}$.

Let $P$ be a $2^{n} \times 2^{n}$ TPM. We say that positive real numbers $x_{1}, x_{2}, \ldots, x_{K}$ and distinct BN matrices $A_{1}, A_{2}, \ldots, A_{K} \in B_{n}(P)$ define a decomposition of $P$ of length $K$ if $\sum^{K}_{i = 1} x_{i} A_{i} = P$ and $\sum^{K}_{i = 1} x_{i} = 1$. Further, we say that a decomposition of $P$ is sparser than another one if its length is strictly smaller.

The problem being studied in this paper is as follows: given a $2^{n} \times 2^{n}$ TPM $P$, we would like to find a decomposition of $P$ which is as sparse as possible. We call it \textit{the construction problem of sparse Probabilistic Boolean Networks}.

\subsection{Previous Work}\label{subsection:prev_work}
Many algorithms for tackling the construction problem of sparse PBNs have been proposed. A dominant modified algorithm was proposed in \cite{cui_PBN_construction}. A maximum entropy rate approach was proposed in \cite{max_entropy_rate}. A modified maximum entropy rate approach was proposed in \cite{chen_jiang_ching_PBN_construction}, which considered an objective function based on the entropy rate with an additional term of $L_{\alpha}$-norm. A projection-based gradient descent method was proposed in \cite{sparse_solution_of_nonnegative}. An alternating direction method of multipliers was proposed in \cite{alternating}. A partial proximal-type operator splitting method was proposed in \cite{partial_proximal}. Some deterministic and probabilistic algorithms that involve removing entries from $P$ sequentially were proposed in \cite{SERs_paper}. A modified orthogonal matching pursuit algorithm was proposed in \cite{MOMP_paper}.

\subsection{Our Contributions}\label{subsection:our_contributions}
The major contributions of this work are threefold.

Firstly, we propose the Greedy Entry Removal algorithm (GER), which takes as input any $2^{n} \times 2^{n}$ TPM $P$ and outputs a decomposition of $P$. Our numerical experiments demonstrate that GER gives the best performance among state-of-the-art sparse PBN construction algorithms: for a wide variety of TPMs, GER outputs decompositions with the smallest lengths among the sparse PBN construction algorithms tested. We also prove that the decompositions output by GER for most of the TPMs tested have the smallest possible length.

Secondly, we derive theoretical results related to the upper bounds of $K$. We consider three sparse PBN construction algorithms, namely Simple Entry Removal Algorithm 1 (SER 1) \cite{SERs_paper}, Simple Entry Removal Algorithm 2 (SER 2) \cite{SERs_paper} and GER. For each of these algorithms, we derive two upper bounds for $K$.

Thirdly, we are the first to study the lower bound problem of $K$, and to derive a series of theoretical results related to this problem. These results have been applied to prove that the decompositions output by GER for most of the TPMs tested in our numerical experiments have the smallest possible length.

\subsection{Organization}\label{subsection:organization}
The rest of this paper is organized as follows. 
In Section \ref{section:existing_algos_and_upper_bounds},
we present three existing sparse PBN construction algorithms and derive four new upper bounds for two of them. 
In Section \ref{section:GER_and_upper_bounds},
we propose the Greedy Entry Removal Algorithm (GER), numerically demonstrate it and derive two upper bound theorems for it. 
In Section \ref{section:lower_bound},
we derive a series of results for the lower bound problem related to the construction of sparse PBNs. 
In Section \ref{section:numerical_experiments},
we describe the numerical experiments conducted, the results of which verify the effectiveness and advantages of GER. 
In Section \ref{section:conclusion},
we give some concluding remarks and discuss some possible future directions of research.

\subsection{Some Notations and Terminologies}\label{subsection:notations_and_terms}

Recall that for any positive integer $x$, the notation $[x]$ means the set $\{ 1, 2, \ldots, x \}$. Define $\mathcal{M}_{L} \coloneqq \left\{ \vec{x} \in \mathbb{R}^{L} : \sum^{L}_{i = 1} x_{i} = 1, x_{i} \geq 0 \; \forall i \in [L] \right\}$ which is known as the probability simplex in $\mathbb{R}^{L}$.
For any vector $\vec{y}$, define $\textrm{supp}(\vec{y}) \coloneqq \left\{ j : y_{j} \neq 0 \right\}$.
For any real-valued matrix or vector $C$, we define $\mathcal{N}^{+}(C)$ to be the number of positive entries of $C$.
We use the term \textit{rational TPMs} to refer to TPMs whose entries are all rational numbers.
Recall that we have used the notation $\vec{e}_{i}$ to refer to the standard unit column vector in a Euclidean space such that the $i$-th entry of $\vec{e}_{i}$ equals $1$.
We will specify the dimension of $\vec{e}_{i}$ each time we use the notation.
We define $\vec{0}_{L}$ to be the zero column vector in $\mathbb{R}^{L}$,
and $\vec{1}_{L}$ to be the column vector in $\mathbb{R}^{L}$ consisting of ones.
For any $\vec{v} = (v_{1}, v_{2}, \ldots, v_{L})^{\top}$, $\vec{w} = (w_{1}, w_{2}, \ldots, w_{L})^{\top} \in \mathbb R^L$, we define
$\max(\vec{v}, \vec{w}) \coloneqq (\max(v_{1}, w_{1}), \max(v_{2}, w_{2}), \ldots, \max(v_{L}, w_{L}))^{\top}$.
For all positive integers $s$ and $t$, we define $\mathbf{O}_{s \times t}$ to be the $s \times t$ zero matrix.
Let $x_{1}, x_{2}, x_{3}, \ldots, x_{2^{n}}$ be arbitrary elements of $[2^{n}]$. The $2^{n} \times 2^{n}$ BN matrix $[\vec{e}_{x_{1}}, \vec{e}_{x_{2}}, \vec{e}_{x_{3}}, \ldots, \vec{e}_{x_{2^{n}}}]$ is also represented as $\langle x_{1}, x_{2}, x_{3}, \ldots, x_{2^{n}} \rangle$. For example, $\langle 4, 1, 3, 2\rangle$ represents the $4 \times 4$ BN matrix
$
\begin{pmatrix}
0	& 1	& 0	& 0	\\
0	& 0	& 0	& 1	\\
0	& 0	& 1	& 0	\\
1	& 0	& 0	& 0
\end{pmatrix}
$.
For any $m \times n$ matrix $[\vec{v}_{1}, \vec{v}_{2}, \ldots, \vec{v}_{n}]$, we define 
$\textrm{vec}([\vec{v}_{1}, \vec{v}_{2}, \ldots, \vec{v}_{n}]) \coloneqq 
\begin{pmatrix}
\vec{v}_{1} \\
\vec{v}_{2} \\
\vdots \\
\vec{v}_{n}
\end{pmatrix}
$.
We shall use MATLAB-style notation to refer to entries and sub-matrices of any matrix $C$. 
For instance, $C(i, j)$ is the $(i, j)$-entry of $C$, 
$C(:,2)$ is the 2\textsuperscript{nd} column of $C$, and 
$C([2, 4], :)$ is the sub-matrix of $C$ consisting of the 2\textsuperscript{nd} and the 4\textsuperscript{th} rows of $C$.

\section{Existing Sparse PBN Construction Algorithms and their Upper Bound Theorems}\label{section:existing_algos_and_upper_bounds}

In this section, we present three existing sparse PBN construction algorithms, namely Simple Entry Removal Algorithm 1 (SER 1) \cite{SERs_paper}, Simple Entry Removal Algorithm 2 (SER 2) \cite{SERs_paper} and the modified orthogonal matching pursuit algorithm (MOMP) \cite{MOMP_paper}. For SER 1 and SER 2, we propose three upper bound theorems (Theorems \ref{thm:entry_removal_upper_bound_SER*}, \ref{thm:rational_upper_bound_SER_1} and \ref{thm:rational_upper_bound_SER_2}). Before we continue, we need to give a frequently used proposition:

\begin{proposition}\label{prop:sum_of_xi_is_1_automatic}
Let $P = rQ$ where $r > 0$ and $Q$ is a $2^{n} \times 2^{n}$ TPM.
Suppose that $x_{1}, x_{2}, \ldots, x_{K}$ are real numbers and $A_{1}, A_{2}, \ldots, A_{K}$ are BN matrices such that 
$P = \sum^{K}_{i = 1} x_{i} A_{i}$. Then $\sum^{K}_{i = 1} x_{i} = r$.
\end{proposition}

\begin{proof}
Note that
\begin{equation}\label{eq:multiply_both_side_with_ones}
r \vec{1}^{\top}_{2^{n}}
= \vec{1}^{\top}_{2^{n}}P 
= \vec{1}^{\top}_{2^{n}} \left( \sum^{K}_{i = 1} x_{i} A_{i} \right)
= \sum^{K}_{i = 1} x_{i} \vec{1}^{\top}_{2^{n}} A_{i} 
= \sum^{K}_{i = 1} x_{i} \vec{1}^{\top}_{2^{n}} 
= \left( \sum^{K}_{i = 1} x_{i} \right) \vec{1}^{\top}_{2^{n}}.
\end{equation}
Therefore, $\sum^{K}_{i = 1} x_{i} = r$.
\end{proof}

\newpage  

\subsection{On SER 1 and SER 2}\label{subsection:SER_1_and_SER_2}

\subsubsection{SER 1}\label{subsubsection:review_and_demo_SER_1}

\begin{algorithm}
\caption{Simple Entry Removal Algorithm 1 (SER 1) \cite{SERs_paper}}\label{alg:SER_1}

\textbf{Input}: A $2^{n} \times 2^{n}$ matrix $P$ such that $P = r_{0} Q$ for some $r_{0} > 0$ and $2^{n} \times 2^{n}$ TPM $Q$. \\
\textbf{Output}: Positive real numbers $x_{1}, x_{2}, \ldots, x_{K}$ and $2^{n} \times 2^{n}$ BN matrices $A_{1}, A_{2}, \ldots, A_{K}$ such that $ P = \sum^{K}_{i = 1} x_{i} A_{i}$ and $\sum^{K}_{i = 1} x_{i} = r_{0}$. 

\begin{enumerate}
\item Set $R_{1} \gets P$ and $k \gets 0$.

\item Set $k \gets k + 1$.

\item Choose the smallest positive entry $x_{k}$ from $R_{k}$. 
If there are more than one positive entry in $R_{k}$ that has the smallest positive value, choose any one of them. 
For each of the other columns, choose the largest positive entry in that column. 
If in any one of these columns, there are more than one positive entry that has the largest value, choose any one of them.
Suppose that the chosen entries are $R_{k}(k_{1}, 1)$, $R_{k}(k_{2}, 2)$, $R_{k}(k_{3}, 3)$, \ldots, $R_{k}(k_{2^{n}}, 2^{n})$. Define the $2^{n} \times 2^{n}$ BN matrix $A_{k} \coloneqq \langle k_{1}, k_{2}, k_{3}, \ldots, k_{2^{n}} \rangle$.

We remark that for each $i \in [2^{n}]$, $R_{k}(k_{i}, i) \geq x_{k}$.

\item Compute $R_{k + 1} \coloneqq R_{k} - x_{k}A_{k}$.

\item If $R_{k+1}$ is the zero matrix, go to Step 6; otherwise, go to Step 2.

\item Set $K \gets k$. Output the positive real numbers $x_{1}, x_{2}, \ldots, x_{K}$ and BN matrices $A_{1}, A_{2}, \ldots, A_{K}$.

Note that $P = \sum^{K}_{i = 1} x_{i}A_{i}$. By Proposition \ref{prop:sum_of_xi_is_1_automatic}, $\sum^{K}_{i = 1} x_{i} = r_{0}$.
\end{enumerate}
\end{algorithm}

We now demonstrate SER 1 by presenting its step-by-step execution on a simple input matrix $P$ with $r_{0} = 10$. We also present how the outputs of this execution can be used to construct a PBN with PBN matrix equal to $\frac{1}{r_{0}}P$.

Let $P$ be the following matrix:
\begin{equation}\label{eq:4-by-4_example_PBN_matrix}
P \coloneqq \left(
\begin{array}{cccc}
1 & 5 & 6 & 0   \\
4 & 0 & 2 & 0   \\
5 & 2 & 0 & 10 \\
0 & 3 & 2 & 0   \\
\end{array}
\right) \eqqcolon R_{1}.
\end{equation}

In SER 1, in the 1\textsuperscript{st} iteration of steps 2 to 5, note that the smallest positive entry in $R_{1}$ is one, which is in the 1\textsuperscript{st} column (marked by a square bracket below). Then, we choose the largest positive entry in each of the other columns (marked by underscores below). Afterwards, we subtract one from these four chosen entries and get
\begin{eqnarray*}
R_{1} = 
\left(
\begin{array}{cccc}
[1] 	& \underline{5}	& \underline{6} 	& 0 \\
4 	& 0 			& 2 			& 0 \\
5 	& 2 			& 0 			& \underline{10} \\
0 	& 3 			& 2 			& 0 \\
\end{array}
\right) & = &
1
 \left(
\begin{array}{cccc}
1 & 1 & 1 & 0 \\
0 & 0 & 0 & 0 \\
0 & 0 & 0 & 1 \\
0 & 0 & 0 & 0 \\
\end{array}
\right)
+
\left(
\begin{array}{cccc}
0 & 4 & 5 & 0 \\
4 & 0 & 2 & 0 \\
5 & 2 & 0 & 9 \\
0 & 3 & 2 & 0 \\
\end{array}
\right) \\
& \eqqcolon & x_{1} A_{1} + R_{2}.
\end{eqnarray*}

In the 2\textsuperscript{nd} iteration of steps 2 to 5, note that the smallest positive entry of $R_2$ is two. Because in $R_{2}$, the values of three entries are equal to this minimum value, we arbitrarily choose one of them, say the one in the 2\textsuperscript{nd} column (marked by a square bracket below). Then, we choose the largest positive entry in each of the other columns (marked by underscores below). Afterwards, we subtract two from these four chosen entries and get
\begin{eqnarray*}
R_{2} = 
\left(
\begin{array}{cccc}
0 			& 4 	& \underline{5}	& 0 \\
4 			& 0 	& 2 			& 0 \\
\underline{5}	& [2] 	& 0 			& \underline{9} \\
0 			& 3 	& 2 			& 0 \\
\end{array}
\right) & = &
2
 \left(
\begin{array}{cccc}
0 & 0 & 1 & 0 \\
0 & 0 & 0 & 0 \\
1 & 1 & 0 & 1 \\
0 & 0 & 0 & 0 \\
\end{array}
\right)
+
\left(
\begin{array}{cccc}
0 & 4 & 3 & 0 \\
4 & 0 & 2 & 0 \\
3 & 0 & 0 & 7 \\
0 & 3 & 2 & 0 \\
\end{array}
\right) \\
& \eqqcolon & x_{2} A_{2} + R_{3}.
\end{eqnarray*}

Executing steps 2 to 5 of SER 1 repeatedly, we have
\begin{eqnarray*}
R_{3} = 
\left(
\begin{array}{cccc}
0 			& \underline{4}	& 3 	& 0 \\
\underline{4}	& 0 			& [2] 	& 0 \\
3 			& 0 			& 0 	& \underline{7} \\
0 			& 3 			& 2 	& 0 \\
\end{array}
\right) & = &
2
 \left(
\begin{array}{cccc}
0 & 1 & 0 & 0 \\
1 & 0 & 1 & 0 \\
0 & 0 & 0 & 1 \\
0 & 0 & 0 & 0 \\
\end{array}
\right)
+
\left(
\begin{array}{cccc}
0 & 2 & 3 & 0 \\
2 & 0 & 0 & 0 \\
3 & 0 & 0 & 5 \\
0 & 3 & 2 & 0 \\
\end{array}
\right) \\
& \eqqcolon & x_{3} A_{3} + R_{4},
\end{eqnarray*}
\begin{eqnarray*}
R_{4} = 
\left(
\begin{array}{cccc}
0 			& 2 			& \underline{3}	& 0 \\
\left[ 2 \right] 	& 0 			& 0 			& 0 \\
3 			& 0 			& 0 			& \underline{5} \\
0 			& \underline{3}	& 2 			& 0 \\
\end{array}
\right) & = &
2
 \left(
\begin{array}{cccc}
0 & 0 & 1 & 0 \\
1 & 0 & 0 & 0 \\
0 & 0 & 0 & 1 \\
0 & 1 & 0 & 0 \\
\end{array}
\right)
+
\left(
\begin{array}{cccc}
0 & 2 & 1 & 0 \\
0 & 0 & 0 & 0 \\
3 & 0 & 0 & 3 \\
0 & 1 & 2 & 0 \\
\end{array}
\right) \\
& \eqqcolon & x_{4} A_{4} + R_{5},
\end{eqnarray*}
\begin{eqnarray*}
R_{5} = 
\left(
\begin{array}{cccc}
0 			& 2 	& 1 			& 0 \\
0 			& 0 	& 0 			& 0 \\
\underline{3}	& 0 	& 0 			& \underline{3} \\
0 			& [1]	& \underline{2}	& 0 \\
\end{array}
\right) & = &
1
 \left(
\begin{array}{cccc}
0 & 0 & 0 & 0 \\
0 & 0 & 0 & 0 \\
1 & 0 & 0 & 1 \\
0 & 1 & 1 & 0 \\
\end{array}
\right)
+
\left(
\begin{array}{cccc}
0 & 2 & 1 & 0 \\
0 & 0 & 0 & 0 \\
2 & 0 & 0 & 2 \\
0 & 0 & 1 & 0 \\
\end{array}
\right) \\
& \eqqcolon & x_{5} A_{5} + R_{6},
\end{eqnarray*}
\begin{eqnarray*}
R_{6} = 
\left(
\begin{array}{cccc}
0 			& \underline{2}	& [1]	& 0 \\
0 			& 0 			& 0 	& 0 \\
\underline{2}	& 0 			& 0 	& \underline{2} \\
0 			& 0 			& 1 	& 0 \\
\end{array}
\right) & = &
1
 \left(
\begin{array}{cccc}
0 & 1 & 1 & 0 \\
0 & 0 & 0 & 0 \\
1 & 0 & 0 & 1 \\
0 & 0 & 0 & 0 \\
\end{array}
\right)
+
\left(
\begin{array}{cccc}
0 & 1 & 0 & 0 \\
0 & 0 & 0 & 0 \\
1 & 0 & 0 & 1 \\
0 & 0 & 1 & 0 \\
\end{array}
\right) \\
& \eqqcolon & x_{6} A_{6} + R_{7},
\end{eqnarray*}
\begin{eqnarray*}
R_{7} = 
\left(
\begin{array}{cccc}
0 			& \underline{1}	& 0 			& 0 \\
0 			& 0 			& 0 			& 0 \\
\left[ 1 \right]	& 0 			& 0 			& \underline{1} \\
0 			& 0 			& \underline{1}	& 0 \\
\end{array}
\right) & = &
1
 \left(
\begin{array}{cccc}
0 & 1 & 0 & 0 \\
0 & 0 & 0 & 0 \\
1 & 0 & 0 & 1 \\
0 & 0 & 1 & 0 \\
\end{array}
\right)
+
\left(
\begin{array}{cccc}
   0 &    0 &    0 & 0 \\
   0 &    0 &    0 & 0 \\
   0 &    0 &    0 & 0 \\
   0 &    0 &    0 & 0 \\
\end{array}
\right) \\
& \eqqcolon & x_{7} A_{7} + R_{8}.
\end{eqnarray*}

Because $R_{8}$ is the zero matrix, SER 1 outputs the positive real numbers $x_{1}, x_{2}, \ldots, x_{7}$ and BN matrices $A_{1}, A_{2}, \ldots, A_{7}$. It can be easily checked that each $x_{i} > 0$, $\sum^{7}_{i = 1} x_{i} = 10 = r_{0}$ and $\sum^{7}_{i = 1} x_{i} A_{i} = P$ as desired.

From the outputs $x_{1}, x_{2}, \ldots, x_{7}$ and $A_{1}, A_{2}, \ldots, A_{7}$, we can construct a $2$-node PBN $\mathcal{P} = (V, F_{1}, F_{2}, \mathcal{D})$ whose PBN matrix equals $\frac{1}{r_{0}}P$. Let $V=\{v_1,v_2\}$. The BNs represented by the BN matrices $A_{1}, A_{2}, \ldots, A_{7}$ are given in Tables~\ref{table:BN_represented_by_A1_SER_1}--\ref{table:BN_represented_by_A7_SER_1}:

\begin{table}[h]
\begin{center}
\begin{tabular}{|c|cc|c|c|}
\hline BN Global States & $v_1 (t)$     & $v_2(t)$   & $f^{(1)}_{1}$ & $f^{(2)}_{1}$  \\
\hline
1 & 0 & 0 & 0 & 0 \\
2 & 0 & 1 & 0 & 0 \\
3 & 1 & 0 & 0 & 0 \\
4 & 1 & 1 & 1 & 0 \\
\hline
\end{tabular}
\caption{The truth table for the BN represented by the BN matrix $A_{1}$.}
\label{table:BN_represented_by_A1_SER_1}
\end{center}
\end{table}

\FloatBarrier

\begin{table}[h]
\begin{center}
\begin{tabular}{|c|cc|c|c|}
\hline BN Global States & $v_1 (t)$     & $v_2(t)$   & $f^{(1)}_{2}$ & $f^{(2)}_{1}$  \\
\hline
1 & 0 & 0 & 1 & 0 \\
2 & 0 & 1 & 1 & 0 \\
3 & 1 & 0 & 0 & 0 \\
4 & 1 & 1 & 1 & 0 \\
\hline
\end{tabular}
\caption{The truth table for the BN represented by the BN matrix $A_{2}$.}
\label{table:BN_represented_by_A2_SER_1}
\end{center}
\end{table}

\FloatBarrier

\begin{table}[h!]
\begin{center}
\begin{tabular}{|c|cc|c|c|}
\hline BN Global States & $v_1 (t)$     & $v_2(t)$   & $f^{(1)}_{1}$ & $f^{(2)}_{2}$  \\
\hline
1 & 0 & 0 & 0 & 1 \\
2 & 0 & 1 & 0 & 0 \\
3 & 1 & 0 & 0 & 1 \\
4 & 1 & 1 & 1 & 0 \\
\hline
\end{tabular}
\caption{The truth table for the BN represented by the BN matrix $A_{3}$.}
\label{table:BN_represented_by_A3_SER_1}
\end{center}
\end{table}

\FloatBarrier

\begin{table}[h!]
\begin{center}
\begin{tabular}{|c|cc|c|c|}
\hline BN Global States & $v_1 (t)$     & $v_2(t)$   & $f^{(1)}_{3}$ & $f^{(2)}_{3}$  \\
\hline
1 & 0 & 0 & 0 & 1 \\
2 & 0 & 1 & 1 & 1 \\
3 & 1 & 0 & 0 & 0 \\
4 & 1 & 1 & 1 & 0 \\
\hline
\end{tabular}
\caption{The truth table for the BN represented by the BN matrix $A_{4}$.}
\label{table:BN_represented_by_A4_SER_1}
\end{center}
\end{table}

\FloatBarrier

\begin{table}[h!]
\begin{center}
\begin{tabular}{|c|cc|c|c|}
\hline BN Global States & $v_1 (t)$     & $v_2(t)$   & $f^{(1)}_{4}$ & $f^{(2)}_{4}$  \\
\hline
1 & 0 & 0 & 1 & 0 \\
2 & 0 & 1 & 1 & 1 \\
3 & 1 & 0 & 1 & 1 \\
4 & 1 & 1 & 1 & 0 \\
\hline
\end{tabular}
\caption{The truth table for the BN represented by the BN matrix $A_{5}$.}
\label{table:BN_represented_by_A5_SER_1}
\end{center}
\end{table}

\FloatBarrier

\begin{table}[h]
\begin{center}
\begin{tabular}{|c|cc|c|c|}
\hline BN Global States & $v_1 (t)$     & $v_2(t)$   & $f^{(1)}_{5}$ & $f^{(2)}_{1}$  \\
\hline
1 & 0 & 0 & 1 & 0 \\
2 & 0 & 1 & 0 & 0 \\
3 & 1 & 0 & 0 & 0 \\
4 & 1 & 1 & 1 & 0 \\
\hline
\end{tabular}
\caption{The truth table for the BN represented by the BN matrix $A_{6}$.}
\label{table:BN_represented_by_A6_SER_1}
\end{center}
\end{table}

\FloatBarrier

\begin{table}[h!]
\begin{center}
\begin{tabular}{|c|cc|c|c|}
\hline BN Global States & $v_1 (t)$     & $v_2(t)$   & $f^{(1)}_{6}$ & $f^{(2)}_{5}$  \\
\hline
1 & 0 & 0 & 1 & 0 \\
2 & 0 & 1 & 0 & 0 \\
3 & 1 & 0 & 1 & 1 \\
4 & 1 & 1 & 1 & 0 \\
\hline
\end{tabular}
\caption{The truth table for the BN represented by the BN matrix $A_{7}$.}
\label{table:BN_represented_by_A7_SER_1}
\end{center}
\end{table}

\FloatBarrier

The lists of Boolean functions associated with the nodes $v_1$ and $v_2$ are respectively $F_{1} \coloneqq \left( f^{(1)}_{1}, f^{(1)}_{2}, f^{(1)}_{3}, f^{(1)}_{4}, f^{(1)}_{5}, f^{(1)}_{6} \right)$ and $F_{2} \coloneqq \left( f^{(2)}_{1}, f^{(2)}_{2}, f^{(2)}_{3}, f^{(2)}_{4}, f^{(2)}_{5} \right)$. In addition, the probability distribution $\mathcal{D}$ on $[6] \times [5]$ is given by the probability mass function
\begin{eqnarray*}
&& \frac{x_1}{r_0} \mathbbm 1_{\{(1,1)\}} + \frac{x_2}{r_0} \mathbbm 1_{\{(2,1)\}} + \frac{x_3}{r_0} \mathbbm 1_{\{(1,2)\}} + \frac{x_4}{r_0} \mathbbm 1_{\{(3,3)\}} + \frac{x_5}{r_0} \mathbbm 1_{\{(4,4)\}} + \frac{x_6}{r_0} \mathbbm 1_{\{(5,1)\}} + \frac{x_7}{r_0} \mathbbm 1_{\{(6,5)\}}\\
&=& 0.1 \mathbbm 1_{\{(1,1),(4,4),(5,1),(6,5)\}} + 0.2 \mathbbm 1_{\{(2,1),(1,2),(3,3)\}}
\end{eqnarray*}
where $\mathbbm 1_A$ is the indicator function of the subset $A$. Consequently, $\frac{1}{r_{0}}P$ is the PBN matrix of the PBN $\mathcal{P} = (V, F_{1}, F_{2}, \mathcal{D})$.

\newpage

\subsubsection{SER 2}\label{subsubsection:review_and_demo_SER_2}

\begin{algorithm}
\caption{Simple Entry Removal Algorithm 2 (SER 2) \cite{SERs_paper}}\label{alg:SER_2}
\textbf{Input}: A $2^{n} \times 2^{n}$ matrix $P$ such that $P = r_{0} Q$ for some $r_{0} > 0$ and $2^{n} \times 2^{n}$ TPM $Q$. \\
\textbf{Output}: Positive real numbers $x_{1}, x_{2}, \ldots, x_{K}$ and $2^{n} \times 2^{n}$ BN matrices $A_{1}, A_{2}, \ldots, A_{K}$ such that $ P = \sum^{K}_{i = 1} x_{i} A_{i}$ and $\sum^{K}_{i = 1} x_{i} = r_{0}$.

\begin{enumerate}
\item Set $R_{1} \gets P$ and $k \gets 0$.

\item Set $k \gets k + 1$.

\item For each column of $R_{k}$, choose the largest positive entry in that column. 
If in any one of these columns, there are more than one positive entry that has the largest value, choose any one of them.
Suppose that the chosen entries are $R_{k}(k_{1}, 1)$, $R_{k}(k_{2}, 2)$, $R_{k}(k_{3}, 3)$, \ldots, $R_{k}(k_{2^{n}}, 2^{n})$. 
Define $x_{k} \coloneqq \min \left( R_{k}(k_{1}, 1), R_{k}(k_{2}, 2), R_{k}(k_{3}, 3), \ldots, R_{k}(k_{2^{n}}, 2^{n}) \right)$.
Define the $2^{n} \times 2^{n}$ BN matrix $A_{k} \coloneqq \langle k_{1}, k_{2}, k_{3}, \ldots, k_{2^{n}} \rangle$.

\item Compute $R_{k + 1} \coloneqq R_{k} - x_{k}A_{k}$.

\item If $R_{k+1}$ is the zero matrix, go to Step 6; otherwise, go to Step 2.

\item Set $K \gets k$. Output the positive real numbers $x_{1}, x_{2}, \ldots, x_{K}$ and BN matrices $A_{1}, A_{2}, \ldots, A_{K}$.

Note that $P = \sum^{K}_{i = 1} x_{i}A_{i}$. By Proposition \ref{prop:sum_of_xi_is_1_automatic}, $\sum^{K}_{i = 1} x_{i} = r_{0}$.
\end{enumerate}
\end{algorithm}

We now demonstrate SER 2 by presenting its step-by-step execution on the simple input matrix $P$ defined in Eq.\ (\ref{eq:4-by-4_example_PBN_matrix}). We also present how the outputs of this execution can be used to construct a PBN with PBN matrix equal to $\frac{1}{r_{0}} P$.

Let $R_{1} \coloneqq P$. In SER 2, in the 1\textsuperscript{st} iteration of steps 2 to 5, we choose the largest positive entry in each column of $R_{1}$ (marked by underscores below). Because the minimum of the four chosen entries is $5$, we subtract $5$ from these four entries and get
\begin{eqnarray*}
R_{1} = 
\left(
\begin{array}{cccc}
1 			& \underline{5}	& \underline{6}	& 0 \\
4 			& 0 			& 2 			& 0 \\
\underline{5} 	& 2 			& 0 			& \underline{10} \\
0 			& 3 			& 2 			& 0 \\
\end{array}
\right) & = &
5
\left(
\begin{array}{cccc}
0 & 1 & 1 & 0 \\
0 & 0 & 0 & 0 \\
1 & 0 & 0 & 1 \\
0 & 0 & 0 & 0 \\
\end{array}
\right)
+
\left(
\begin{array}{cccc}
1 & 0 & 1 & 0 \\
4 & 0 & 2 & 0 \\
0 & 2 & 0 & 5 \\
0 & 3 & 2 & 0 \\
\end{array}
\right) \\
& \eqqcolon & x_{1} A_{1} + R_{2}.
\end{eqnarray*}

In the 2\textsuperscript{nd} iteration of steps 2 to 5, we choose the largest positive entry in each column of $R_{2}$. Because the largest positive value in column $3$ occurs twice ($R_{2}(2, 3) = R_{2}(4, 3) = 2$), we arbitrarily choose one of them, say $R_{2}(2, 3)$. Because the minimum of the four chosen entries (marked by underscores below) is $2$, we subtract $2$ from these four entries and get
\begin{eqnarray*}
R_{2} = 
\left(
\begin{array}{cccc}
1 			& 0 			& 1 			& 0 \\
\underline{4} 	& 0 			& \underline{2} 	& 0 \\
0 			& 2 			& 0 			& \underline{5} \\
0 			& \underline{3} 	& 2 			& 0 \\
\end{array}
\right) & = &
2
\left(
\begin{array}{cccc}
0 & 0 & 0 & 0 \\
1 & 0 & 1 & 0 \\
0 & 0 & 0 & 1 \\
0 & 1 & 0 & 0 \\
\end{array}
\right)
+
\left(
\begin{array}{cccc}
1 & 0 & 1 & 0 \\
2 & 0 & 0 & 0 \\
0 & 2 & 0 & 3 \\
0 & 1 & 2 & 0 \\
\end{array}
\right) \\
& \eqqcolon & x_{2} A_{2} + R_{3}.
\end{eqnarray*}

Executing steps 2 to 5 of SER 2 repeatedly, we have
\begin{eqnarray*}
R_{3} = 
\left(
\begin{array}{cccc}
1 			& 0 			& 1 			& 0 \\
\underline{2} 	& 0 			& 0 			& 0 \\
0 			& \underline{2} 	& 0 			& \underline{3} \\
0 			& 1 			& \underline{2} 	& 0 \\
\end{array}
\right) & = &
2
\left(
\begin{array}{cccc}
0 & 0 & 0 & 0 \\
1 & 0 & 0 & 0 \\
0 & 1 & 0 & 1 \\
0 & 0 & 1 & 0 \\
\end{array}
\right)
+
\left(
\begin{array}{cccc}
1 & 0 & 1 & 0 \\
0 & 0 & 0 & 0 \\
0 & 0 & 0 & 1 \\
0 & 1 & 0 & 0 \\
\end{array}
\right) \\
& \eqqcolon & x_{3} A_{3} + R_{4},
\end{eqnarray*}
\begin{eqnarray*}
R_{4} = 
\left(
\begin{array}{cccc}
\underline{1}	& 0 			& \underline{1}	& 0 \\
0 			& 0 			& 0 			& 0 \\
0 			& 0 			& 0 			& \underline{1} \\
0 			& \underline{1}	& 0 			& 0 \\
\end{array}
\right) & = &
1
\left(
\begin{array}{cccc}
1 & 0 & 1 & 0 \\
0 & 0 & 0 & 0 \\
0 & 0 & 0 & 1 \\
0 & 1 & 0 & 0 \\
\end{array}
\right)
+
\left(
\begin{array}{cccc}
0 & 0 & 0 & 0 \\
0 & 0 & 0 & 0 \\
0 & 0 & 0 & 0 \\
0 & 0 & 0 & 0 \\
\end{array}
\right) \\
& \eqqcolon & x_{4} A_{4} + R_{5}.
\end{eqnarray*}

Because $R_{5}$ is the zero matrix, SER 2 outputs the positive real numbers $x_{1}, x_{2}, x_{3}, x_{4}$ and BN matrices $A_{1}, A_{2}, A_{3}, A_{4}$. It can be easily checked that each $x_{i} > 0$, $\sum^{4}_{i = 1} x_{i} = 10 = r_{0}$ and $\sum^{4}_{i = 1} x_{i} A_{i} = P$ as expected.

From the outputs $x_{1}, x_{2}, x_{3}, x_{4}$ and $A_{1}, A_{2}, A_{3}, A_{4}$, we can construct a $2$-node PBN $\mathcal{P} = (V, F_{1}, F_{2}, \mathcal{D})$ whose PBN matrix equals $\frac{1}{r_{0}} P$. Let $V=\{v_{1}, v_{2}\}$. The BNs represented by the BN matrices $A_{1}, A_{2}, A_{3}, A_{4}$ are given in Tables~\ref{table:BN_represented_by_A1_SER_2}--\ref{table:BN_represented_by_A4_SER_2}:

\begin{table}[h!]
\begin{center}
\begin{tabular}{|c|cc|c|c|}
\hline BN Global States & $v_1 (t)$     & $v_2(t)$   & $f^{(1)}_{1}$ & $f^{(2)}_{1}$  \\
\hline
1 & 0 & 0 & 1 & 0 \\
2 & 0 & 1 & 0 & 0 \\
3 & 1 & 0 & 0 & 0 \\
4 & 1 & 1 & 1 & 0 \\
\hline
\end{tabular}
\caption{The truth table for the BN represented by the BN matrix $A_{1}$.}
\label{table:BN_represented_by_A1_SER_2}
\end{center}
\end{table}

\FloatBarrier

\begin{table}[h!]
\begin{center}
\begin{tabular}{|c|cc|c|c|}
\hline BN Global States & $v_1 (t)$     & $v_2(t)$   & $f^{(1)}_{2}$ & $f^{(2)}_{2}$  \\
\hline
1 & 0 & 0 & 0 & 1 \\
2 & 0 & 1 & 1 & 1 \\
3 & 1 & 0 & 0 & 1 \\
4 & 1 & 1 & 1 & 0 \\
\hline
\end{tabular}
\caption{The truth table for the BN represented by the BN matrix $A_{2}$.}
\label{table:BN_represented_by_A2_SER_2}
\end{center}
\end{table}

\FloatBarrier

\begin{table}[h!]
\begin{center}
\begin{tabular}{|c|cc|c|c|}
\hline BN Global States & $v_1 (t)$     & $v_2(t)$   & $f^{(1)}_{3}$ & $f^{(2)}_{3}$  \\
\hline
1 & 0 & 0 & 0 & 1 \\
2 & 0 & 1 & 1 & 0 \\
3 & 1 & 0 & 1 & 1 \\
4 & 1 & 1 & 1 & 0 \\
\hline
\end{tabular}
\caption{The truth table for the BN represented by the BN matrix $A_{3}$.}
\label{table:BN_represented_by_A3_SER_2}
\end{center}
\end{table}

\FloatBarrier

\begin{table}[h!]
\begin{center}
\begin{tabular}{|c|cc|c|c|}
\hline BN Global States & $v_1 (t)$     & $v_2(t)$   & $f^{(1)}_{2}$ & $f^{(2)}_{4}$  \\
\hline
1 & 0 & 0 & 0 & 0 \\
2 & 0 & 1 & 1 & 1 \\
3 & 1 & 0 & 0 & 0 \\
4 & 1 & 1 & 1 & 0 \\
\hline
\end{tabular}
\caption{The truth table for the BN represented by the BN matrix $A_{4}$.}
\label{table:BN_represented_by_A4_SER_2}
\end{center}
\end{table}

\FloatBarrier

The lists of Boolean functions associated with the nodes $v_1$ and $v_2$ are respectively $F_{1} \coloneqq \left( f^{(1)}_{1}, f^{(1)}_{2}, f^{(1)}_{3} \right)$ and $F_{2} \coloneqq \left( f^{(2)}_{1}, f^{(2)}_{2}, f^{(2)}_{3}, f^{(2)}_{4} \right)$. In addition, the probability distribution $\mathcal{D}$ on $[3] \times [4]$ is given by the probability mass function
\begin{eqnarray*}
&& \frac{x_1}{r_0} \mathbbm 1_{\{(1,1)\}} + \frac{x_2}{r_0} \mathbbm 1_{\{(2,2)\}} + \frac{x_3}{r_0} \mathbbm 1_{\{(3,3)\}} + \frac{x_4}{r_0} \mathbbm 1_{\{(2,4)\}} \\
&=& 0.1 \mathbbm 1_{\{(2,4)\}} + 0.2 \mathbbm 1_{\{(2,2),(3,3)\}} + 0.5 \mathbbm 1_{\{(1,1)\}}.
\end{eqnarray*}
Consequently, $\frac{1}{r_{0}}P$ is the PBN matrix of the PBN $\mathcal{P} = (V, F_{1}, F_{2}, \mathcal{D})$.

\subsubsection{Termination of the SER 1 and SER 2 Algorithms}\label{subsubsection:main_tools_main_thms_for_SER_1_and_2}

In this section, we justify that both SER 1 and SER 2 terminate for any input matrix $P=r_0Q$ where $r_0>0$ and $Q$ is a $2^{n} \times 2^{n}$ TPM. We begin by examining steps 3 and 4 of each of the two algorithms.

Concerning SER 1, we can regard steps 3 to 4 as a procedure (called $\lambda_{\textrm{SER1}}$) which takes as input a $2^{n} \times 2^{n}$ matrix $R$, outputs a positive real number $x$ and a $2^{n} \times 2^{n}$ BN matrix $A = \langle k_{1}, k_{2}, k_{3}, \ldots, k_{2^{n}} \rangle$, and computes a new matrix $R - xA$. 
If each entry of $R$ is non-negative and $\vec{1}^{\top}_{2^{n}}R = r \vec{1}^{\top}_{2^{n}}$ for some $r \in (0, r_{0}]$ where $r_{0} > 0$, 
then one of the following situations will occur:

Case 1: there exists a BN matrix $\tilde{A} = \langle p_{1}, p_{2}, p_{3}, \ldots, p_{2^{n}} \rangle$ such that $R = r \tilde{A}$ (i.e., $R$ is a positive multiple of some BN matrix). By looking at step 3 of SER 1, we can easily deduce that all actions of step 3 can be successfully carried out when $\lambda_{\textrm{SER1}}$ is executed with $R$ as input. The outputs $x$ and $A$ will equal $r$ and $\tilde{A}$ respectively. Then $R - x A = \mathbf{O}_{2^{n} \times 2^{n}}$. Moreover, $\mathcal{N}^{+}(R) = 2^{n} > 0 = \mathcal{N}^{+}(R - x A)$. We remark that when $R=R_k$ for some $k$ during the execution of SER 1,
then $R_{k+1} = \mathbf{O}_{2^{n} \times 2^{n}}$. Step 6 of SER 1 will be executed afterwards and the SER 1 algorithm will terminate.

Case 2: $R$ is not a positive multiple of some BN matrix. Because each entry of $R$ is non-negative and there exists $r \in (0, r_{0}]$ such that each column of $R$ adds up to $r$, there is at least one positive entry in each column of $R$. Therefore, in step 3, we can successfully carry out the action of choosing the smallest positive entry from $R$ and the largest positive entry in each of the other columns of $R$ to form $x$ (a positive real number) and $A = \langle k_{1}, k_{2}, k_{3}, \ldots, k_{2^{n}} \rangle$ (a BN matrix). Note that for $i = 1, 2, 3, \ldots, 2^{n}$, $R(k_{i}, i) \geq x$. Then, all entries of $R - x A$ are non-negative, and $\vec{1}^{\top}_{2^{n}}(R - x A) = (r - x) \vec{1}^{\top}_{2^{n}}$. Because $R \neq x A$, $R - x A$ contains positive entries and hence $0 < r - x < r_{0}$. Moreover, the entries of $R$ and $R - x A$ are the same except for the $(k_{1}, 1), (k_{2}, 2), (k_{3}, 3), \ldots, (k_{2^{n}}, 2^{n})$ entries: for $i = 1, 2, 3, \ldots, 2^{n}$, $R(k_{i}, i) > 0$ and 
the $(k_{i}, i)$ entry of $R - xA$ equals $R(k_{i}, i) - x \geq 0$.
In particular, for $i^{*}$ such that $R(k_{i^{*}}, i^{*})$ is the smallest positive entry of $R$ chosen in step 3, 
the $(k_{i^{*}}, i^{*})$ entry of $R - xA$ equals $R(k_{i^{*}}, i^{*}) - x = 0$.
Therefore, $\mathcal{N}^{+}(R) > \mathcal{N}^{+}(R - x A)$.

Similarly, for SER 2, we can regard steps 3 to 4 as a procedure (called $\lambda_{\textrm{SER2}}$) which takes as input a $2^{n} \times 2^{n}$ matrix $R$, outputs a positive real number $x$ and a $2^{n} \times 2^{n}$ BN matrix $A = \langle k_{1}, k_{2}, k_{3}, \ldots, k_{2^{n}} \rangle$, and computes a new matrix $R - xA$. 
If each entry of $R$ is non-negative and $\vec{1}^{\top}_{2^{n}}R = r \vec{1}^{\top}_{2^{n}}$ for some $r \in (0, r_{0}]$ where $r_{0} > 0$, 
then one of the following situations will occur:

Case 1: there exists a BN matrix $\tilde{A} = \langle p_{1}, p_{2}, p_{3}, \ldots, p_{2^{n}} \rangle$ such that $R = r \tilde{A}$ (i.e., $R$ is a positive multiple of some BN matrix). By looking at step 3 of SER 2, we can easily deduce that all actions of step 3 can be successfully carried out when $\lambda_{\textrm{SER2}}$ is executed with $R$ as input. The outputs $x$ and $A$ will equal $r$ and $\tilde{A}$ respectively. Then $R - x A = \mathbf{O}_{2^{n} \times 2^{n}}$. Moreover, $\mathcal{N}^{+}(R) = 2^{n} > 0 = \mathcal{N}^{+}(R - x A)$. We remark that when $R=R_k$ for some $k$ during the execution of SER 2,
then $R_{k+1} = \mathbf{O}_{2^{n} \times 2^{n}}$. Step 6 of SER 2 will be executed afterwards and the SER 2 algorithm will terminate.

Case 2: $R$ is not a positive multiple of some BN matrix. Note that there is at least one positive entry in each column of $R$. Therefore, in step 3, we can successfully carry out the action of choosing the largest positive entry in each column of $R$, and subsequently define $x$ (a positive real number) and $A = \langle k_{1}, k_{2}, k_{3}, \ldots, k_{2^{n}} \rangle$ (a BN matrix). 
Note that for $i = 1, 2, 3, \ldots, 2^{n}$, 
$R(k_{i}, i) \geq x = \min \left( R(k_{1}, 1), R(k_{2}, 2), R(k_{3}, 3), \ldots, R(k_{2^{n}}, 2^{n}) \right)$.
Then, all entries of $R - x A$ are non-negative, and $\vec{1}^{\top}_{2^{n}}(R - x A) = (r - x) \vec{1}^{\top}_{2^{n}}$. Because $R \neq x A$, $R - x A$ contains positive entries and hence $0 < r - x < r_{0}$.
Moreover, the entries of $R$ and $R - x A$ are the same except for the $(k_{1}, 1), (k_{2}, 2), (k_{3}, 3), \ldots, (k_{2^{n}}, 2^{n})$ entries: 
for $i = 1, 2, 3, \ldots, 2^{n}$, $R(k_{i}, i) > 0$ and 
the $(k_{i}, i)$ entry of $R - xA$ equals $R(k_{i}, i) - x \geq 0$.
In particular, there exists $i^{*} \in [2^{n}]$ such that $R(k_{i^{*}}, i^{*}) = \min \left( R(k_{1}, 1), R(k_{2}, 2), R(k_{3}, 3), \ldots, R(k_{2^{n}}, 2^{n}) \right) = x$.  
Hence, 
the $(k_{i^{*}}, i^{*})$ entry of $R - xA$ equals $R(k_{i^{*}}, i^{*}) - x = 0$.
Therefore, $\mathcal{N}^{+}(R) > \mathcal{N}^{+}(R - x A)$.

With the above considerations, we can prove the following important theorem about SER 1 and SER 2.

\begin{theorem}\label{thm:SER*_must_terminate_and_output}
Let $P = r_{0} Q$ where $r_{0} > 0$ and $Q$ is a $2^{n} \times 2^{n}$ TPM.
When SER 1 is executed with $P$ as input, the algorithm will eventually terminate at step 6 and output positive real numbers $x_{1}, x_{2}, \ldots, x_{K}$ and distinct BN matrices $A_{1}, A_{2}, \ldots, A_{K}$ such that $P = \sum^{K}_{i = 1} x_{i} A_{i}$ and $\sum^{K}_{i = 1} x_{i} = r_{0}$.
In other words, 
$Q$ is the PBN matrix of some PBN which has the decomposition
\[ Q = \sum^K_{i=1}\ \frac{x_i}{r_0}\ A_i. \]
This statement is also true for SER 2.
\end{theorem}

\begin{proof}
We will present the proof of this theorem for SER 1 only, because the proof for SER 2 follows the same line of reasoning.

Consider the $2^{n} \times 2^{n}$ matrices $R_{1}, R_{2}, \ldots$ generated during the execution of SER 1 with $P$ as input. 
If $P=R_1$ is a positive multiple of a BN matrix, then the theorem obviously holds. So from now on we assume that $P = R_{1}$ is not a positive multiple of a BN matrix. Suppose for a contradiction that every matrix $R_k$ generated by the algorithm were not a positive multiple of a BN matrix. Then
$\mathcal{N}^{+}(R_{1}), \mathcal{N}^{+}(R_{2}), \ldots$ is an infinite strictly decreasing sequence of positive integers, which is absurd.  
Therefore, $R_K$ is a positive multiple of a BN matrix for some positive integer $K$. 
Therefore, in the execution of SER 1 with $P$ as input, steps 2 to 5 of SER 1 is executed for $K$ times and then step 6 is executed to output the positive real numbers $x_{1}, x_{2}, \ldots, x_{K}$ and BN matrices $A_{1}, A_{2}, \ldots, A_{K}$.

By considering step 4 of SER 1, we can see that
\begin{eqnarray*}
\mathbf{O}_{2^{n} \times 2^{n}} = R_{K+1} = R_{K} - x_{K} A_{K}
& = & R_{K-1} - x_{K-1} A_{K-1} - x_{K} A_{K} \\
& = & \ldots 
= R_{1} - \sum^{K}_{i = 1} x_{i} A_{i} 
= P - \sum^{K}_{i = 1} x_{i} A_{i}.
\end{eqnarray*}
Therefore, $P = \sum^{K}_{i = 1} x_{i} A_{i}$.
By Proposition \ref{prop:sum_of_xi_is_1_automatic}, $\sum^{K}_{i = 1} x_{i} = r_{0}$.

Now, we are going to prove that $A_{1}, A_{2}, \ldots, A_{K}$ are distinct. Fix arbitrary $s, t \in [K]$ such that $s < t$. Write $A_{s}$ as $\langle k^{(s)}_{1}, k^{(s)}_{2}, k^{(s)}_{3}, \ldots, k^{(s)}_{2^{n}} \rangle$ and $A_{t}$ as $\langle k^{(t)}_{1}, k^{(t)}_{2}, k^{(t)}_{3}, \ldots, k^{(t)}_{2^{n}} \rangle$.
Consider the sequence $R_{1}, R_{2}, \ldots, R_{K+1}$. Note that for all $i, j \in [2^{n}]$, $R_{1}(i, j) \geq R_{2}(i, j) \geq \ldots \geq R_{K+1}(i, j) = 0$.
Because $R_{s+1} = R_{s} - x_{s} A_{s}$ and $R_{s}(k^{(s)}_{j^{*}}, j^{*}) = x_{s}$ for some $j^{*} \in [2^{n}]$, $R_{s+1}(k^{(s)}_{j^{*}}, j^{*}) = 0$, which implies that $R_{t}(k^{(s)}_{j^{*}}, j^{*}) = 0$.
Because $R_{t}(k^{(t)}_{j^{*}}, j^{*}) \geq x_{t} > 0$, $k^{(s)}_{j^{*}} \neq k^{(t)}_{j^{*}}$. Hence, $A_{s} \neq A_{t}$. Since $s$ and $t$ are arbitrary, $A_{1}, A_{2}, \ldots, A_{K}$ are distinct.

Because $\frac{x_{1}}{r_{0}}, \frac{x_{2}}{r_{0}}, \ldots, \frac{x_{K}}{r_{0}} > 0$, $\sum^{K}_{i = 1} \frac{x_{i}}{r_{0}} = 1$, $Q = \frac{1}{r_{0}} P = \sum^{K}_{i = 1} \frac{x_{i}}{r_{0}} A_{i}$ and $A_{1}, A_{2}, \ldots, A_{K}$ are distinct BN matrices, $\sum^{K}_{i = 1} \frac{x_{i}}{r_{0}} A_{i}$ is a decomposition of $Q$.
From $\frac{x_{1}}{r_{0}}, \frac{x_{2}}{r_{0}}, \ldots, \frac{x_{K}}{r_{0}}$ and $A_{1}, A_{2}, \ldots, A_{K}$, we can construct a PBN whose PBN matrix equals $Q$. The method of construction has been explained in the numerical demonstration of SER 1 in Section \ref{subsubsection:review_and_demo_SER_1}.
\end{proof}

\subsubsection{New Upper Bound Theorems for SER 1 and SER 2}\label{subsubsection:new_upper_bound_thms_for_SER_1_and_2}

Theorem \ref{thm:SER*_must_terminate_and_output} says that for any $2^{n} \times 2^{n}$ TPM $P$, we can obtain a decomposition of $P$ by executing SER 1 or SER 2 with $r_{0} P$ as input, where $r_{0}$ is any positive real number. In this section, we are going to prove three theorems related to the sparsity of the decompositions found using SER 1 and SER 2.

\begin{theorem}\label{thm:entry_removal_upper_bound_SER*}
Let $P$ be a $2^{n} \times 2^{n}$ TPM and let $r_{0}$ be any positive real number.
If $x_{1}, x_{2}, \ldots, x_{K} > 0$ and $A_{1}, A_{2}, \ldots, A_{K}$ (distinct BN matrices) are the outputs when SER 1 is applied to the matrix $r_{0}P$, then $K \leq \mathcal{N}^{+}(P) - 2^{n} + 1$.
This statement is also true for SER 2.
\end{theorem}

\begin{proof}
We will prove this theorem for SER 1 only, because the proof for SER 2 follows the same line of reasoning.

Consider the sequence of $2^{n} \times 2^{n}$ matrices $R_{1}, R_{2}, \ldots, R_{K+1}$ generated by SER 1 when the algorithm is executed with $r_{0} P$ as input. Note that $R_{1} = r_{0} P$ and $R_{K+1} = \mathbf{O}_{2^{n} \times 2^{n}}$. Moreover, $\mathcal{N}^{+}(P) = \mathcal{N}^{+}(r_{0}P) = \mathcal{N}^{+}(R_{1}) > \mathcal{N}^{+}(R_{2}) > \ldots > \mathcal{N}^{+}(R_{K}) > \mathcal{N}^{+}(R_{K+1}) = 0$. Therefore, $\mathcal{N}^{+}(R_{k}) - \mathcal{N}^{+}(R_{k+1}) \geq 1$ for $k = 1, 2, \ldots, K$.
By step 4 of SER 1, $R_{K} - x_{K} A_{K} = R_{K+1} = \mathbf{O}_{2^{n} \times 2^{n}}$, which implies that $R_{K}$ is a positive multiple of a BN matrix. Hence, $\mathcal{N}^{+}(R_{K}) = 2^{n}$. Therefore,
\begin{eqnarray}
\mathcal{N}^{+}(P)
& = &		\sum^{K}_{k = 1} \left[ \mathcal{N}^{+}(R_{k}) - \mathcal{N}^{+}(R_{k+1}) \right] \nonumber \\
& = &		\sum^{K-1}_{k = 1} \left[ \mathcal{N}^{+}(R_{k}) - \mathcal{N}^{+}(R_{k+1}) \right]
		+ \mathcal{N}^{+}(R_{K}) - \mathcal{N}^{+}(R_{K+1}) \nonumber \\
& = &		\sum^{K-1}_{k = 1} \left[ \mathcal{N}^{+}(R_{k}) - \mathcal{N}^{+}(R_{k+1}) \right]
		+ 2^{n} \nonumber \\
& \geq &	K - 1 + 2^{n}. \label{eq:SER_1_X_geq_K-1+2_to_n}
\end{eqnarray}
Rearranging, we get $K \leq \mathcal{N}^{+}(P) - 2^{n} + 1$.
\end{proof}

\begin{theorem}\label{thm:rational_upper_bound_SER_1}
Let $P$ be a $2^{n} \times 2^{n}$ rational TPM.
Let $p$ be the smallest positive integer such that all entries of $pP$ are integers. Let $a^{*}$ be the smallest positive entry of $P$. 
Let $r_{0}$ be any positive real number.
If $x_{1}, x_{2}, \ldots, x_{K} > 0$ and $A_{1}, A_{2}, \ldots, A_{K}$ (distinct BN matrices) are the outputs when SER 1 is applied to the matrix $r_{0}P$, then
$K \leq (1 - a^{*})p + 1$.
\end{theorem}

\begin{proof}
Consider the sequence of $2^{n} \times 2^{n}$ matrices $R_{1}, R_{2}, \ldots, R_{K+1}$ generated during the execution of SER 1 with $r_{0} P$ as input.
Note that each entry of $R_{1} = r_{0} P$ is of the form $\frac{r_{0}s}{p}$ for some non-negative integer $s$.
Because $x_{1}$ is the smallest positive entry of $R_{1}$ (see step 3 of SER 1), 
$x_{1} = r_{0} a^{*} = \frac{r_{0}s_{1}}{p}$ where $s_{1}=pa^{*}$ is a positive integer.
Hence, each entry of $R_{2} = R_{1} - x_{1} A_{1}$ is also of the form $\frac{r_{0}s}{p}$ for some non-negative integer $s$.
Using the same argument, we deduce that for each $k \in [K]$, $x_{k} = \frac{r_{0}s_{k}}{p}$ for some positive integer $s_{k}$.
By Theorem \ref{thm:SER*_must_terminate_and_output},
$\sum^{K}_{k = 1} x_{k} = r_{0}$. This implies
\begin{equation*}
1 
= \sum^{K}_{k = 1} \frac{s_{k}}{p}
= a^{*} + \sum^{K}_{k = 2} \frac{s_{k}}{p}
\geq a^{*} + \sum^{K}_{k = 2} \frac{1}{p}.
\end{equation*}
Rearranging, we get $K \leq (1 - a^{*}) p + 1$.
\end{proof}

\begin{theorem}\label{thm:rational_upper_bound_SER_2}
Let $P$ be a $2^{n} \times 2^{n}$ rational TPM.
Let $p$ be the smallest positive integer such that all entries of $pP$ are integers.
Let $a_{1}, a_{2}, a_{3}, \ldots, a_{2^{n}}$ be the largest entries in columns $1$, $2$, $3$, \ldots, $2^{n}$ of $P$ respectively. Let $\hat{a} \coloneqq \min(a_{1}, a_{2}, a_{3}, \ldots, a_{2^{n}})$.
Let $r_{0}$ be any positive real number.
If $x_{1}, x_{2}, \ldots, x_{K} > 0$ and $A_{1}, A_{2}, \ldots, A_{K}$ (distinct BN matrices) are the outputs when SER 2 is applied to the matrix $r_{0}P$, then
$K \leq (1 - \hat{a})p + 1$.
\end{theorem}

\begin{proof}
Consider the sequence of $2^{n} \times 2^{n}$ matrices $R_{1}, R_{2}, \ldots, R_{K+1}$ generated during the execution of SER 2 with $r_{0} P$ as input.
Note that each entry of $R_{1} = r_{0} P$ is of the form $\frac{r_{0}s}{p}$ for some non-negative integer $s$.
Moreover, note that $r_{0} a_{1}, r_{0} a_{2}, r_{0} a_{3}, \ldots, r_{0} a_{2^{n}}$ are the largest entries in columns $1, 2, 3, \ldots, 2^{n}$ of $R_{1} = r_{0} P$ respectively.
Therefore, by step 3 of SER 2, $x_{1} = \min(r_{0}a_{1}, r_{0}a_{2}, r_{0}a_{3}, \ldots, r_{0}a_{2^{n}}) = r_{0} \hat{a} = \frac{r_{0}s_{1}}{p}$, where $s_{1} = p \hat{a}$ is a positive integer.
Hence, each entry of $R_{2} = R_{1} - x_{1} A_{1}$ is also of the form $\frac{r_{0}s}{p}$ for some non-negative integer $s$.
Using the same argument, we deduce that for each $k \in [K]$, $x_{k} = \frac{r_{0}s_{k}}{p}$ for some positive integer $s_{k}$.
By Theorem \ref{thm:SER*_must_terminate_and_output},
$\sum^{K}_{k = 1} x_{k} = r_{0}$. This implies
\begin{equation*}
1 
= \sum^{K}_{k = 1} \frac{s_{k}}{p}
= \hat{a} + \sum^{K}_{k = 2} \frac{s_{k}}{p}
\geq \hat{a} + \sum^{K}_{k = 2} \frac{1}{p}.
\end{equation*}
Rearranging, we get $K \leq (1 - \hat{a}) p + 1$.
\end{proof}

\newpage 

\subsection{Modified Orthogonal Matching Pursuit Algorithm}\label{subsection:MOMP}

The pseudocode of MOMP \cite{MOMP_paper} is shown in Algorithm \ref{alg:MOMP}. 
Before presenting MOMP, we define some notations. 

Let $P$ be any $2^{n} \times 2^{n}$ TPM. Suppose that $B_{n}(P) = \{ A_{1}, A_{2}, \ldots, A_{N} \}$ where $N=\prod^{2^{n}}_{j = 1} |D_{j}(P)|$.
Define
$A \coloneqq \left[ \textrm{vec}(A_{1}), \textrm{vec}(A_{2}), \ldots, \textrm{vec}(A_{N}) \right] \in \mathbb{R}^{2^{2n} \times N}$. Consider an arbitrary subset $U = \{u_{1}, u_{2}, \ldots, u_{m}\}$ of $[N]$, where the elements are labeled so that $u_{1} < u_{2} < \ldots < u_{m}$. 
Define $A_{U} \coloneqq \left[ \textrm{vec}(A_{u_{1}}), \textrm{vec}(A_{u_{2}}), \ldots, \textrm{vec}(A_{u_{m}}) \right] \in \mathbb{R}^{2^{2n} \times m}$.

\begin{algorithm}
\caption{Modified Orthogonal Matching Pursuit Algorithm (MOMP) \cite{MOMP_paper}}\label{alg:MOMP}
\textbf{Input}: $\vec{b} \coloneqq \textrm{vec}(P) \in \mathbb{R}^{2^{2n}}$, where $P$ is a $2^{n} \times 2^{n}$ TPM. \\
\textbf{Parameter}: $\varepsilon_{\textrm{tolerance}}$ (parameter of error tolerance) \\
\textbf{Output}: $\vec{x}^{k} \in \mathcal{M}_{N}$ (the vector of coefficients for BN matrices in $B_{n}(P)$)

\begin{enumerate}
\item Choose an initial guess $\vec{x}^{0} \in \mathcal{M}_{N}$. Set $S^0 \gets \varnothing$, $k \gets 0$ and $e^{0} \gets +\infty$.

\item While $e^{k} > \varepsilon_{\textrm{tolerance}}$:

\begin{enumerate}
\item Find $j_{k + 1} \in [N]$ such that
$j_{k+1} \in \argmax\limits_{j \in [N]} \vec{e}^\top_{j} A^\top (\vec{b} - A\vec{x}^{k})$,
where $\vec{e}_{j} \in \mathbb{R}^{N}$.

\item Set $S^{k+1} \gets S^{k} \cup \{ j_{k+1} \}$.

\item Find $\vec{x}^{k+1} \in \mathcal{M}_{N}$ such that
$\vec{x}^{k+1} \in \argmin\limits_{\vec{x} \in \mathcal{M}_{N}, \textrm{supp}(\vec{x}) \subseteq S^{k+1} } \frac{1}{2} \| \vec{b} - A\vec{x} \|^{2}_{2}$.

\item Let $\vec{r}^{k+1} \coloneqq \vec{b} - A \vec{x}^{k+1}$. Set 
\begin{eqnarray*}
\lefteqn{
e^{k+1} \gets 
\left\| 
A^{\top}_{\textrm{supp}(\vec{x}^{k+1})}\vec{r}^{k+1} - 
[(\vec{x}^{k+1})^{\top} A^{\top} \vec{r}^{k+1}] \vec{1}_{\left| \textrm{supp}(\vec{x}^{k+1}) \right|}
\right\|_{2}
+
} \\
& & \left\| \max
\left( 
A^{\top}_{[N] \setminus \textrm{supp}(\vec{x}^{k+1})}\vec{r}^{k+1} - 
[(\vec{x}^{k+1})^{\top} A^{\top} \vec{r}^{k+1}] \vec{1}_{\left| [N] \setminus \textrm{supp}(\vec{x}^{k+1}) \right|}, 
\vec{0}_{\left| [N] \setminus \textrm{supp}(\vec{x}^{k+1}) \right|} 
\right) \right\|_{2}.
\end{eqnarray*} 

\item Set $k \gets k+1$.

\end{enumerate} 

\item Return $\vec{x}^{k}$.
\end{enumerate} 

\end{algorithm}

\newpage

\section{The Greedy Entry Removal Algorithm and its Upper Bound Theorems}\label{section:GER_and_upper_bounds}

\subsection{The Greedy Entry Removal Algorithm}\label{subsection:GER_pseudocode}

In this section, we propose the Greedy Entry Removal algorithm (GER) for the construction of sparse PBNs. Before we present the pseudocode of GER, we need to define some notations and function. Consider any real-valued $n_{r} \times n_{c}$ matrix $C$.

First, for all $b \in \mathbb{R}$, define
\begin{equation*}
\textrm{Col\_indices}(b, C) \coloneqq 
\left\{ j \in [n_{c}] : \textrm{the $j$-th column of $C$ contains $b$} \right\}.
\end{equation*}
We call $\textrm{Col\_indices}(b, C)$ \textit{the set of column indices of $b$ in $C$}.
Note that $| \textrm{Col\_indices}(b, C) |$ is the number of columns of $C$ which contain $b$. We call $| \textrm{Col\_indices}(b, C) |$ \textit{the column frequency of $b$ in $C$}.

Second, for any column vector $\vec{v} = (v_{1}, v_{2}, \ldots, v_{n})^{\top} \in \mathbb{R}^{n}$ and $b \in \mathbb{R}$, define
\begin{eqnarray*}
\textrm{Occurrences}(b, \vec{v}) 
& \coloneqq &
\left\{ i \in [n] : v_{i} = b \right\}, \\
\textrm{Larger}(b, \vec{v})
& \coloneqq &
\left\{ i \in [n] : v_{i} > b \right\}.
\end{eqnarray*}
We call $\textrm{Occurrences}(b, \vec{v})$ and $\textrm{Larger}(b, \vec{v})$
\textit{the set of occurrences of $b$ in $\vec{v}$} and \textit{the set of larger occurrences of $b$ in $\vec{v}$} respectively.

We remark that in the GER algorithm, certain matrices $C$ will be updated at certain steps. After an update, $\textrm{Col\_indices}(b, C)$, $|\textrm{Col\_indices}(b, C)|$, $\textrm{Occurrences}(b, C(:, j))$ and $\textrm{Larger}(b, C(:, j))$ will change accordingly, where $b$ is any real number and $j \in [n_{c}]$.

For example, if
\[ C = \begin{pmatrix}
1 & 0 & 0 & 3\\
2 & 5 & 3 & 0\\
3 & 0 & 3 & 0\\
0 & 1 & 0 & 3
\end{pmatrix}, \]
then
\begin{itemize}
\item $\textrm{Col\_indices}(1, C) = \{ 1, 2 \}$,  
$\textrm{Col\_indices}(2, C) = \{ 1 \}$,
$\textrm{Col\_indices}(3, C) = \{ 1, 3, 4 \}$, \\
$\textrm{Col\_indices}(4, C) = \varnothing$,
$\textrm{Col\_indices}(5, C) = \{ 2 \}$.

\item Hence, the column frequencies of $1$, $2$, $3$, $4$, $5$ in $C$ are $2$, $1$, $3$, $0$, $1$ respectively. 

\item 
$\textrm{Occurrences}(3, C(:, 1)) = \{ 3 \}$,
$\textrm{Occurrences}(3, C(:, 3)) = \{ 2, 3 \}$,
$\textrm{Occurrences}(3, C(:, 4)) = \{ 1, 4 \}$.

\item 
$\textrm{Larger}(2, C(:, 1)) = \{ 3 \}$,
$\textrm{Larger}(2, C(:, 2)) = \{ 2 \}$,
$\textrm{Larger}(2, C(:, 3)) = \{ 2, 3 \}$.
\end{itemize}
Suppose that $C$ is updated in the following way:
\begin{equation*}
C \gets 
C - 
\begin{pmatrix}
0 & 0 & 0 & 1 \\
0 & 1 & 0 & 0 \\
1 & 0 & 1 & 0 \\
0 & 0 & 0 & 0
\end{pmatrix}.
\end{equation*}
After the update,
\begin{equation*}
C = 
\begin{pmatrix}
1 & 0 & 0 & 2 \\
2 & 4 & 3 & 0 \\
2 & 0 & 2 & 0 \\
0 & 1 & 0 & 3
\end{pmatrix}.
\end{equation*}
Then, the following changes (highlighted in bold) are resulted:
\begin{itemize}
\item $\textrm{Col\_indices}(1, C) = \{ 1, 2 \}$,  
$\textrm{Col\_indices}(2, C) = \mathbf{\{ 1, 3, 4 \}}$,
$\textrm{Col\_indices}(3, C) = \mathbf{\{ 3, 4 \}}$, \\
$\textrm{Col\_indices}(4, C) = \mathbf{\{ 2 \}}$,
$\textrm{Col\_indices}(5, C) = \mathbf{\varnothing}$.

\item Hence, the column frequencies of $1$, $2$, $3$, $4$, $5$ in $C$ are $2$, $\mathbf{3}$, $\mathbf{2}$, $\mathbf{1}$, $\mathbf{0}$ respectively. 

\item 
$\textrm{Occurrences}(3, C(:, 1)) = \mathbf{\varnothing}$,
$\textrm{Occurrences}(3, C(:, 3)) = \mathbf{\{ 2 \}}$,
$\textrm{Occurrences}(3, C(:, 4)) = \mathbf{\{ 4 \}}$.

\item 
$\textrm{Larger}(2, C(:, 1)) = \mathbf{\varnothing}$,
$\textrm{Larger}(2, C(:, 2)) = \{ 2 \}$,
$\textrm{Larger}(2, C(:, 3)) = \mathbf{\{ 2 \}}$.
\end{itemize}

Third, we need to prove the following lemma which the GER Entry Selection Algorithm (GERESA) relies upon:

\begin{lemma}\label{lemma:larger}
Let $R = r Q$ where $r > 0$ and $Q$ is a $2^{n} \times 2^{n}$ TPM.
Let $v$ be a positive entry of $R$ such that $v \leq \min\limits_{1\le j\le 2^{n}}\ \max\limits_{1\le i\le 2^{n}}\ R(i,j)$.
Then, for all $j^{'} \in [2^{n}] \setminus \textrm{Col\_indices}(v, R)$,
$\max\limits_{1\le i\le 2^{n}}\ R(i, j^{'}) > v$ and hence $\textrm{Larger}(v, R(:, j^{'})) \neq \varnothing$.
\end{lemma}

\begin{proof}
Note that
$v \leq \min\limits_{1\le j\le 2^{n}}\ \max\limits_{1\le i\le 2^{n}}\ R(i,j) \leq \max\limits_{1\le i\le 2^{n}}\ R(i, j^{'})$.
Moreover, because $j^{'} \in [2^{n}] \setminus \textrm{Col\_indices}(v, R)$, $R(:, j^{'})$ does not contain $v$ as a positive entry.
Because $\max\limits_{1\le i\le 2^{n}}\ R(i, j^{'})$ is a positive entry of $R(:, j^{'})$,
$\max\limits_{1\le i\le 2^{n}}\ R(i, j^{'}) \neq v$ and hence
$\max\limits_{1\le i\le 2^{n}}\ R(i, j^{'}) > v$.
Let $i^{*} \in [2^{n}]$ such that $R(i^{*}, j^{'}) = \max\limits_{1\le i\le 2^{n}}\ R(i, j^{'})$.
Then, $i^{*} \in \textrm{Larger}(v, R(:, j^{'}))$.
\end{proof}

Finally, we need to define a score function $f_{\textrm{score}}$ which is used in the GER algorithm. Let $R^{'} \in \mathbb{R}^{n_{r} \times n_{c}}$. Let $z$ be any real number greater than $1$. Let $E^{+}(R^{'})$ be the set of positive entries of $R^{'}$. Then, we define
\begin{equation*}
f_{\textrm{score}}(R^{'}, z) \coloneqq
\sum_{c \in E^{+}(R^{'})} z^{\textrm{col.\ freq.\ of $c$ in $R^{'}$}}.
\end{equation*}

We remark that if a matrix $R^{'}$ has positive entries with large column frequencies, then $R^{'}$ tends to have a high score. We will illustrate this point with the following examples.
Consider the two matrices
\begin{equation*}
R = 
\begin{pmatrix}
45	& 9	& 85	& 1	\\
81	& 52	& 16	& 36	\\
0	& 65	& 26	& 94	\\
5	& 5	& 4	& 0
\end{pmatrix},\quad
R^{'} = 
\begin{pmatrix}
45	& 9	& 84	& 1	\\
81	& 52	& 6	& 36	\\
0	& 65	& 36	& 94	\\
5	& 5	& 5	& 0
\end{pmatrix}.
\end{equation*}
Consider the positive entries in each of these matrices:
\begin{itemize}
\item The column frequency of $5$ in $R$ is $2$, and the column frequencies of $1$, $4$, $9$, $16$, $26$, $36$, $45$, $52$, $65$, $81$, $85$, $94$ in $R$ all equal $1$.

\item The column frequency of $5$ in $R^{'}$ is $3$, the column frequency of $36$ in $R^{'}$ is $2$, and the column frequencies of $1$, $6$, $9$, $45$, $52$, $65$, $81$, $84$, $94$ in $R^{'}$ all equal 1.
\end{itemize}
If we set the score parameter $z$ to be $10$, then
\begin{eqnarray*}
f_{\textrm{score}}(R, z) & = &
\underbrace{z^{2}}_{\substack{\text{contribution} \\ \text{from $5$}}} + 
\underbrace{12z}_{\substack{\text{contribution from $1$, $4$, $9$, $16$,} \\ \text{$26$, $36$, $45$, $52$, $65$, $81$, $85$, $94$}}} = 220, \\
f_{\textrm{score}}(R^{'}, z) & = &
\underbrace{z^{3}}_{\substack{\text{contribution} \\ \text{from $5$}}} + 
\underbrace{z^{2}}_{\substack{\text{contribution} \\ \text{from $36$}}} + 
\underbrace{9z}_{\substack{\text{contribution from $1$, $6$,} \\ \text{$9$, $45$, $52$, $65$, $81$, $84$, $94$}}} = 1190.
\end{eqnarray*}
We can see that the score for $R^{'}$ is greater than the score for $R$.

We are now ready to present the GER algorithm:

\begin{algorithm}
\caption{The Greedy Entry Removal Algorithm (GER)}
\label{alg:GER}
\begin{algorithmic}[1]
\Require The score parameter $z > 1$, and \\
$P = r_{0} Q$ for some $r_{0} > 0$ and $2^{n} \times 2^{n}$ TPM $Q$.

\Ensure Positive real numbers $x_{1}, x_{2}, \ldots, x_{K}$ and distinct $2^{n} \times 2^{n}$ BN matrices $A_{1}, A_{2}, \ldots, A_{K}$ such that $ P = \sum^{K}_{i = 1} x_{i} A_{i}$ and $\sum^{K}_{i = 1} x_{i} = r_{0}$. 

\Statex

\State $R \gets P$
\Comment{Initialize the residue matrix $R$.}
\State $K \gets 0$
\Comment{A counter for the number of BN matrices added so far.}
\State $\mathbf{x} \gets [\,]$
\Comment{A list for holding weights $x_{1}, x_{2}, \ldots, x_{K}$.}
\State $\mathbf{BN} \gets [\,]$
\Comment{A list for holding BN matrices $A_{1}, A_{2}, \ldots, A_{K}$.}

\While{ $R \ne \mathbf{O}_{2^{n} \times 2^{n}}$ }
\State $K \gets K+1$
\State $B \gets \min\limits_{1\le j\le 2^{n}}\ \max\limits_{1\le i\le 2^{n}}\ R(i,j)$
\State Determine the list $\mathbf v$ of positive entries of $R$ which are not greater than $B$ and attain the maximum column frequency in $R$. 

\State $x \gets 0$ and $\mathtt{score} \gets -\infty$
\Comment{Initialize the weight and its score.}

\ForAll{$v \in \mathbf v$}
\State $\mathtt{temp\_A} \gets \mathtt{GERESA}(R,v)$
\LComment{{\tt GERESA} is a subroutine that outputs a BN matrix based on the inputs $R$ and $v$.}
\State $\mathtt{temp\_score} \gets f_{\textrm{score}}(R - v * \mathtt{temp\_A}, z)$
\LComment{Using the score parameter $z$, compute a score $\mathtt{temp\_score}$ for the pair $(v, \mathtt{temp\_A})$.}

\If{($\mathtt{temp\_score}>\mathtt{score}$)\\ \hphantom{xx} or (($\mathtt{temp\_score}=\mathtt{score}$) and ($v>x$)) }
\State $\mathtt{score} \gets \mathtt{temp\_score}$
\State $x \gets v$
\State $A \gets \mathtt{temp\_A}$
\EndIf
\LComment{After the for loop is fully executed, the pair $(v,\mathtt{temp\_A})$ with the highest score will be chosen as the required component $(x, A)$.  When there are two or more highest-scoring pairs, the highest-scoring pair with the greatest $v$ will be chosen.}
\EndFor

\State $R \gets R - x A$
\State Append $x$ to $\mathbf{x}$ and append $A$ to $\mathbf{BN}$
\EndWhile

\State \Output $\mathbf{x}$ and ${\bf BN}$
\end{algorithmic}
\end{algorithm}

\begin{algorithm}
\caption{The GER Entry Selection Algorithm (GERESA)}
\label{alg:GERESA}
\begin{algorithmic}[1]
\Require 
$R = r Q$ for some $r > 0$ and $2^{n} \times 2^{n}$ TPM $Q$; \\
a positive entry $v$ of $R$ such that 
$v \leq \min\limits_{1\le j\le 2^{n}}\ \max\limits_{1\le i\le 2^{n}}\ R(i,j)$.

\Ensure A $2^{n} \times 2^{n}$ BN matrix $A$.

\Statex
\State Create the variables $p_{1}, p_{2}, p_{3}, \ldots, p_{2^{n}}$
\State $R_{\textrm{copy}} \gets R$

\LComment{Here, passing by value is performed. Hence, any modifications of $R_{\textrm{copy}}$ will not affect $R$, and vice versa. We remark that the variable $R$ will not be modified throughout GERESA, and that $R_{\textrm{copy}}$ will be modified iteratively in lines 6-11 and lines 12-17.}

\State $\mathtt{selected\_columns} \gets [\,]$
\LComment{This list contains the indices of those $p_{j}$'s to which we have assigned values.}

\ForAll{$j \in \textrm{Col\_indices}(v, R)$}
\Comment{$\textrm{Col\_indices}(v, R) \neq \varnothing$ since $v$ is a positive entry of $R$.}
\State Set $p_{j}$ to be an arbitrary element of $\textrm{Occurrences}(v, R(:, j))$
\LComment{Note that $j \in \textrm{Col\_indices}(v, R)$ and hence $R(:, j)$ contains $v$ as a positive entry. Therefore, $\textrm{Occurrences}(v, R(:, j)) \neq \varnothing$.}

\State $R_{\textrm{copy}}(p_{j}, j) \gets 0$
\LComment{Update $R_{\textrm{copy}}(:, j)$ after the determination of the value of $p_{j}$. Note that $R(p_{j}, j) - v = 0$.}

\State Append $j$ to $\mathtt{selected\_columns}$
\EndFor

\ForAll{$j \in [2^{n}] \setminus \textrm{Col\_indices}(v, R)$}
\State Set $p_{j}$ to be an element $i$ of $\textrm{Larger}(v, R(:, j))$ which maximizes the column frequency of $R(i, j) - v$ in $R_{\textrm{copy}}(:, \mathtt{selected\_columns})$
\LComment{By Lemma \ref{lemma:larger}, $\textrm{Larger}(v, R(:, j)) \neq \varnothing$.}

\State $R_{\textrm{copy}}(p_{j}, j) \gets R(p_{j}, j) - v$
\LComment{Update $R_{\textrm{copy}}(:, j)$ after the determination of the value of $p_{j}$. Moreover, note that $R(p_{j}, j) - v > 0$ because $p_{j} \in \textrm{Larger}(v, R(:, j))$.}

\State Append $j$ to $\mathtt{selected\_columns}$
\EndFor

\LComment{After the execution of lines 6-17, $p_{1}, p_{2}, p_{3}, \ldots, p_{2^{n}}$ have all been assigned values. Moreover,
$R_{\textrm{copy}}$ equals $R - v\langle p_{1}, p_{2}, p_{3}, \ldots, p_{2^{n}} \rangle$.}

\State \Output $\langle p_{1}, p_{2}, p_{3}, \ldots, p_{2^{n}} \rangle$
\end{algorithmic}
\end{algorithm}

\newpage 

\subsection{Numerical Demonstration of the Greedy Entry Removal Algorithm}\label{subsection:GER_demo}

In this section, we illustrate GER in a holistic manner by presenting its execution on an example input matrix $Q$ which equals $r_{0} Q^{'}$ for some $r_{0} > 0$ and $8 \times 8$ TPM $Q^{'}$. The detailed step-by-step execution of GERESA (line 11 of GER) will not be presented here and will be deferred to Section \ref{subsection:GERESA_demo}. At the end of this section, we present how the output of GER can be used to construct a PBN with PBN matrix equal to $\frac{1}{r_{0}}Q = Q^{'}$.

Let $Q$ be the following matrix:
\begin{equation}\label{eq:example_integral_PBN_matrix}
Q \coloneqq \left(
\begin{array}{cccccccc}
32	& 0	& 2	& 0	& 0	& 4	& 0	& 0 \\
0  	& 0  	& 0 	& 36	& 4 	& 13  	& 0  	& 0 \\
0  	& 0  	& 9  	& 0  	& 0  	& 0 	& 61  	& 0 \\
0 	& 15  	& 0  	& 0  	& 0  	& 0  	& 0  	& 0 \\
0 	& 15 	& 30  	& 0  	& 0 	& 13  	& 0  	& 0 \\
29 	& 31 	& 20  	& 0 	& 25 	& 29  	& 0  	& 6 \\
0  	& 0  	& 0  	& 0 	& 32  	& 0  	& 0  	& 0 \\
0  	& 0  	& 0 	& 25  	& 0  	& 2  	& 0 	& 55 \\
\end{array}
\right).
\end{equation}
We can see that $r_{0} = 61$. We are going to apply GER to $Q$ when the GER score parameter $z$ is set to $10$.

In the initialization phase of GER (lines 1-4), we set $R$ to be $Q$ and $K$ to be $0$.
We also set $\mathbf{x}$ and $\mathbf{BN}$ to be two empty lists.

Similar to SER 1 and SER 2, until $R$ becomes the zero matrix, GER iteratively finds a positive real number $x$ and a BN matrix $A$ (lines 6-19), subtract $xA$ from $R$ (line 20), and append $x$ to $\mathbf{x}$ and $A$ to $\mathbf{BN}$ (line 21). Suppose that the loop in lines 5-21 has been iterated for $K^{'}$ times before $R$ becomes the zero matrix. Write $\mathbf{x}$ as $[x_{1}, x_{2}, \ldots, x_{K^{'}}]$ and $\mathbf{BN}$ as $[A_{1}, A_{2}, \ldots, A_{K^{'}}]$. GER eventually outputs $\mathbf{x}$ and $\mathbf{BN}$. Note that $Q = \sum^{K^{'}}_{i = 1} x_{i} A_{i}$ and hence $\sum^{K^{'}}_{i = 1} x_{i} = r_{0} = 61$ by Proposition \ref{prop:sum_of_xi_is_1_automatic}. 
For each $k \in [K^{'} + 1]$, let $R_{k} \coloneqq Q - \sum^{k - 1}_{i = 1} x_{i} A_{i}$.
We remark that the sequence of matrices $R_{1}, R_{2}, \ldots, R_{K^{'}+1}$ satisfies $\mathcal{N}^{+}(Q) = \mathcal{N}^{+}(R_{1}) > \mathcal{N}^{+}(R_{2}) > \ldots > \mathcal{N}^{+}(R_{K^{'}+1}) = \mathcal{N}^{+}(\mathbf{O}_{8 \times 8}) = 0$. 
Because $\mathcal{N}^{+}(Q) = \sum^{K^{'}}_{k = 1} \mathcal{N}^{+}(R_{k}) - \mathcal{N}^{+}(R_{k+1})$ is a positive constant, if $\mathcal{N}^{+}(R_{k}) - \mathcal{N}^{+}(R_{k+1})$ is larger for each $k$, $K^{'}$ would be smaller, which means a sparser decomposition as we desire. Therefore, GER tries to maximize $\mathcal{N}^{+}(R_{k}) - \mathcal{N}^{+}(R_{k+1})$ for each $k$, and it does so in a greedy manner. We will now proceed to explain in more detail how GER tries to achieve such maximization when the loop in lines 5-21 of GER is iterated repeatedly.

In the 1\textsuperscript{st} iteration of the while loop in lines 5-21, we try to choose a positive real number $x$ and an $8 \times 8$ BN matrix $A$, and subtract $xA$ from $R$. 
Note that $R_{1} = R = Q$ and the largest entries in the 1\textsuperscript{st} to 8\textsuperscript{th} columns of $R$ are $32$, $31$, $30$, $36$, $32$, $29$, $61$ and $55$ respectively. Therefore, $B$ is set to be $\min(32, 31, 30, 36, 32, 29, 61, 55) = 29$ (line 7). 
We remark that $x$ cannot be greater than $B = 29$; otherwise, $R - x A$ must contain at least one negative entry, which is undesirable. We also remark that $x$ must equal some positive entry value in $R$; otherwise, $R - x A$ will have the same number of positive entries as $R$ (i.e., $\mathcal{N}^{+}(R_{2}) = \mathcal{N}^{+}(R_{1})$), but we want $\mathcal{N}^{+}(R_{2}) < \mathcal{N}^{+}(R_{1})$.

We would like to choose $x$ and $A$ in such a way that:
\begin{description}
\item[(Target 1)] $\mathcal{N}^{+}(R_{1}) - \mathcal{N}^{+}(R_{2})$ is maximized, and that

\item[(Target 2)] $R_{2} = R_{1} - x A$ contains positive entries with high column frequencies. This is indicated by a high score $f_{\textrm{score}}(R_{2}, z)$.
\end{description}
We have already explained why we want to achieve Target 1, and we will explain why we want to achieve Target 2 when we discuss the 2\textsuperscript{nd} iteration of the while loop in lines 5-21 of GER.

To achieve Target 1, we consider the positive entries in $R$ (Eq.\ (\ref{eq:example_integral_PBN_matrix})) that are not greater than $B = 29$ and have the highest column frequencies in $R$. They are:
\begin{enumerate}
\item the value $v_{1} \coloneqq 2 = R(1, 3) = R(8, 6)$,

\item the value $v_{2} \coloneqq 4 = R(2, 5) = R(1, 6)$,

\item the value $v_{3} \coloneqq 25 = R(8, 4) = R(6, 5)$, and

\item the value $v_{4} \coloneqq 29 = R(6, 1) = R(6, 6)$.
\end{enumerate}
Then, we can see that if we choose $x$ and $A$ in one of the following four ways, then $\mathcal{N}^{+}(R_{1}) - \mathcal{N}^{+}(R_{2})$ is maximized and equals $2$ (thus achieving Target 1):
\begin{description}
\item[(Option 1.1)]
Set $x$ to be $v_{1} = 2$, and set $A$ to be some BN matrix $\langle p^{(1)}_{1}, p^{(1)}_{2}, p^{(1)}_{3}, \ldots, p^{(1)}_{8} \rangle$ such that the constraints $p^{(1)}_{3} = 1$ and $p^{(1)}_{6} = 8$ hold (call this collection of constraints $\mathcal{C}^{(1)}_{1}$);

\item[(Option 1.2)]
Set $x$ to be $v_{2} = 4$, and set $A$ to be some BN matrix $\langle p^{(2)}_{1}, p^{(2)}_{2}, p^{(2)}_{3}, \ldots, p^{(2)}_{8} \rangle$ such that the constraints $p^{(2)}_{5} = 2$ and $p^{(2)}_{6} = 1$ hold (call this collection of constraints $\mathcal{C}^{(2)}_{1}$);

\item[(Option 1.3)]
Set $x$ to be $v_{3} = 25$, and set $A$ to be some BN matrix $\langle p^{(3)}_{1}, p^{(3)}_{2}, p^{(3)}_{3}, \ldots, p^{(3)}_{8} \rangle$ such that the constraints $p^{(3)}_{4} = 8$ and $p^{(3)}_{5} = 6$ hold (call this collection of constraints $\mathcal{C}^{(3)}_{1}$);

\item[(Option 1.4)]
Set $x$ to be $v_{4} = 29$, and set $A$ to be some BN matrix $\langle p^{(4)}_{1}, p^{(4)}_{2}, p^{(4)}_{3}, \ldots, p^{(4)}_{8} \rangle$ such that the constraints $p^{(4)}_{1} = 6$ and $p^{(4)}_{6} = 6$ hold (call this collection of constraints $\mathcal{C}^{(4)}_{1}$).
\end{description}
GER ``explores'' each of these four options, and then determines which of them is the best one (in terms of Target 2) and should be chosen. To be more specific, for $j = 1, 2, 3, 4$, by executing GERESA on $R$ and $v_{j}$, GER forms a BN matrix $\langle p^{(j)}_{1}, p^{(j)}_{2}, p^{(j)}_{3}, \ldots, p^{(j)}_{8} \rangle$ such that the collection of constraints $\mathcal{C}^{(j)}_{1}$ is satisfied (related to achieving Target 1) and $R^{(j)}_{2} \coloneqq R - v_{j} \langle p^{(j)}_{1}, p^{(j)}_{2}, p^{(j)}_{3}, \ldots, p^{(j)}_{8} \rangle$ contains positive entries with high column frequencies (related to achieving Target 2). Then, it compares the scores $f_{\textrm{score}}(R^{(1)}_{2}, z)$, $f_{\textrm{score}}(R^{(2)}_{2}, z)$, $f_{\textrm{score}}(R^{(3)}_{2}, z)$, $f_{\textrm{score}}(R^{(4)}_{2}, z)$ and sets $x$ and $A$ to be the pair $v_{j}$ and $\langle p^{(j)}_{1}, p^{(j)}_{2}, p^{(j)}_{3}, \ldots, p^{(j)}_{8} \rangle$ that gives the highest score.

We first consider Option 1.1. GERESA outputs the BN matrix $\langle p^{(1)}_{1}, p^{(1)}_{2}, p^{(1)}_{3}, \ldots, p^{(1)}_{8} \rangle = \langle 1, 6, 1, 2, 7, 8, 3, 6 \rangle$ when $R$ and $v_{1}$ are input. Hence,
\begin{eqnarray}\label{eq:iteration_1_choice_v1}
R^{(1)}_{2}
& = & R - 2 \langle 1, 6, 1, 2, 7, 8, 3, 6 \rangle \nonumber \\
& = & \left(
\begin{array}{cccccccc}
30	& 0	& 0	& 0	& 0	& 4	& 0	& 0 \\
0  	& 0  	& 0 	& 34	& 4 	& 13  	& 0  	& 0 \\
0  	& 0  	& 9  	& 0  	& 0  	& 0 	& 59  	& 0 \\
0 	& 15  	& 0  	& 0  	& 0  	& 0  	& 0  	& 0 \\
0 	& 15 	& 30  	& 0  	& 0 	& 13  	& 0  	& 0 \\
29 	& 29 	& 20  	& 0 	& 25 	& 29  	& 0  	& 4 \\
0  	& 0  	& 0  	& 0 	& 30  	& 0  	& 0  	& 0 \\
0  	& 0  	& 0 	& 25  	& 0  	& 0  	& 0 	& 55 \\
\end{array}
\right).
\end{eqnarray}
GER then computes $f_{\textrm{score}}(R^{(1)}_{2}, z)$.
From Eq.\ (\ref{eq:iteration_1_choice_v1}), we can see that:
\begin{enumerate}
\item the column frequencies of $4$, $29$, $30$ in $R^{(1)}_{2}$ equal 3;

\item the column frequency of $25$ in $R^{(1)}_{2}$ equals 2;

\item the column frequencies of $9$, $13$, $15$, $20$, $34$, $55$, $59$ in $R^{(1)}_{2}$ equal 1.
\end{enumerate}
Therefore,
\begin{equation}\label{eq:iteration_1_score_u1}
f_{\textrm{score}}(R^{(1)}_{2}, z)
= 3 \times z^{3} + 1 \times z^{2} + 7 \times z^{1}
= 3 \times 10^{3} + 1 \times 10^{2} + 7 \times 10^{1}
= 3170.
\end{equation}

Next, we consider Option 1.2. GERESA outputs the BN matrix $\langle p^{(2)}_{1}, p^{(2)}_{2}, p^{(2)}_{3}, \ldots, p^{(2)}_{8} \rangle = \langle 6, 4, 3, 2, 2, 1, 3, 6 \rangle$ when $R$ and $v_{2}$ are input. Hence,
\begin{eqnarray}\label{eq:iteration_1_choice_v2}
R^{(2)}_{2}
& = & R - 4 \langle 6, 4, 3, 2, 2, 1, 3, 6 \rangle \nonumber \\
& = & \left(
\begin{array}{cccccccc}
32	& 0	& 2	& 0	& 0	& 0	& 0	& 0 \\
0  	& 0  	& 0 	& 32	& 0 	& 13  	& 0  	& 0 \\
0  	& 0  	& 5  	& 0  	& 0  	& 0 	& 57  	& 0 \\
0 	& 11  	& 0  	& 0  	& 0  	& 0  	& 0  	& 0 \\
0 	& 15 	& 30  	& 0  	& 0 	& 13  	& 0  	& 0 \\
25 	& 31 	& 20  	& 0 	& 25 	& 29  	& 0  	& 2 \\
0  	& 0  	& 0  	& 0 	& 32  	& 0  	& 0  	& 0 \\
0  	& 0  	& 0 	& 25  	& 0  	& 2  	& 0 	& 55 \\
\end{array}
\right).
\end{eqnarray}
GER then computes $f_{\textrm{score}}(R^{(2)}_{2}, z)$. From Eq.\ (\ref{eq:iteration_1_choice_v2}), we can see that:
\begin{enumerate}
\item the column frequencies of $2$, $25$, $32$ in $R^{(2)}_{2}$ equal 3;

\item the column frequencies of $5$, $11$, $13$, $15$, $20$, $29$, $30$, $31$, $55$, $57$ in $R^{(2)}_{2}$ equal 1.
\end{enumerate}
Therefore,
\begin{equation}\label{eq:iteration_1_score_u2}
f_{\textrm{score}}(R^{(2)}_{2}, z)
= 3 \times z^{3} + 10 \times z^{1}
= 3 \times 10^{3} + 10 \times 10^{1}
= 3100.
\end{equation}

Thirdly, we consider Option 1.3. GERESA outputs the BN matrix $\langle p^{(3)}_{1}, p^{(3)}_{2}, p^{(3)}_{3}, \ldots, p^{(3)}_{8} \rangle = \langle 6, 6, 5, 8, 6, 6, 3, 8 \rangle$ when $R$ and $v_{3}$ are input. Hence,
\begin{eqnarray}\label{eq:iteration_1_choice_v3}
R^{(3)}_{2}
& = & R - 25 \langle 6, 6, 5, 8, 6, 6, 3, 8 \rangle \nonumber \\
& = & \left(
\begin{array}{cccccccc}
32	& 0	& 2	& 0	& 0	& 4	& 0	& 0 \\
0  	& 0  	& 0 	& 36	& 4 	& 13  	& 0  	& 0 \\
0  	& 0  	& 9  	& 0  	& 0  	& 0 	& 36  	& 0 \\
0 	& 15  	& 0  	& 0  	& 0  	& 0  	& 0  	& 0 \\
0 	& 15 	& 5  	& 0  	& 0 	& 13  	& 0  	& 0 \\
4 	& 6 	& 20  	& 0 	& 0 	& 4  	& 0  	& 6 \\
0  	& 0  	& 0  	& 0 	& 32  	& 0  	& 0  	& 0 \\
0  	& 0  	& 0 	& 0  	& 0  	& 2  	& 0 	& 30 \\
\end{array}
\right).
\end{eqnarray}
GER then computes $f_{\textrm{score}}(R^{(3)}_{2}, z)$. From Eq.\ (\ref{eq:iteration_1_choice_v3}), we can see that:
\begin{enumerate}
\item the column frequency of $4$ in $R^{(3)}_{2}$ equals 3;

\item the column frequencies of $2$, $6$, $32$, $36$ in $R^{(3)}_{2}$ equal 2;

\item the column frequencies of $5$, $9$, $13$, $15$, $20$, $30$ in $R^{(3)}_{2}$ equal 1.
\end{enumerate}
Therefore,
\begin{equation}\label{eq:iteration_1_score_u3}
f_{\textrm{score}}(R^{(3)}_{2}, z)
= 1 \times z^{3} + 4 \times z^{2} + 6 \times z^{1}
= 1 \times 10^{3} + 4 \times 10^{2} + 6 \times 10^{1}
= 1460.
\end{equation}

Lastly, we consider Option 1.4. GERESA outputs the BN matrix $\langle p^{(4)}_{1}, p^{(4)}_{2}, p^{(4)}_{3}, \ldots, p^{(4)}_{8} \rangle = \langle 6, 6, 5, 2, 7, 6, 3, 8 \rangle$ when $R$ and $v_{4}$ are input. Hence,
\begin{eqnarray}\label{eq:iteration_1_choice_v4}
R^{(4)}_{2}
& = & R - 29 \langle 6, 6, 5, 2, 7, 6, 3, 8 \rangle \nonumber \\
& = & \left(
\begin{array}{cccccccc}
32	& 0	& 2	& 0	& 0	& 4	& 0	& 0 \\
0  	& 0  	& 0 	& 7	& 4 	& 13  	& 0  	& 0 \\
0  	& 0  	& 9  	& 0  	& 0  	& 0 	& 32  	& 0 \\
0 	& 15  	& 0  	& 0  	& 0  	& 0  	& 0  	& 0 \\
0 	& 15 	& 1  	& 0  	& 0 	& 13  	& 0  	& 0 \\
0 	& 2 	& 20  	& 0 	& 25 	& 0  	& 0  	& 6 \\
0  	& 0  	& 0  	& 0 	& 3  	& 0  	& 0  	& 0 \\
0  	& 0  	& 0 	& 25  	& 0  	& 2  	& 0 	& 26 \\
\end{array}
\right).
\end{eqnarray}
GER then computes $f_{\textrm{score}}(R^{(4)}_{2}, z)$. From Eq.\ (\ref{eq:iteration_1_choice_v4}), we can see that:
\begin{enumerate}
\item the column frequency of $2$ in $R^{(4)}_{2}$ equals 3;

\item the column frequencies of $4$, $25$, $32$ in $R^{(4)}_{2}$ equal 2;

\item the column frequencies of $1$, $3$, $6$, $7$, $9$, $13$, $15$, $20$, $26$ in $R^{(4)}_{2}$ equal 1.
\end{enumerate}
Therefore,
\begin{equation}\label{eq:iteration_1_score_u4}
f_{\textrm{score}}(R^{(4)}_{2}, z)
= 1 \times z^{3} + 3 \times z^{2} + 9 \times z^{1}
= 1 \times 10^{3} + 3 \times 10^{2} + 9 \times 10^{1}
= 1390.
\end{equation}

Now, we are able to determine which of the four options (Options 1.1, 1.2, 1.3 and 1.4) is the best one and should be chosen. In terms of achieving Target 1, all four options are equally good because for $j = 1, 2, 3, 4$, $\mathcal{N}^{+}(R) - \mathcal{N}^{+}(R^{(j)}_{2}) = 22 - 20 = 2$ (see Eq.\ (\ref{eq:example_integral_PBN_matrix}), (\ref{eq:iteration_1_choice_v1}), (\ref{eq:iteration_1_choice_v2}), (\ref{eq:iteration_1_choice_v3}), (\ref{eq:iteration_1_choice_v4})). Therefore, we just need to evaluate the four options only in terms of the extent to which Target 2 is achieved by each option. We does so by comparing the scores $f_{\textrm{score}}(R^{(1)}_{2}, z)$, $f_{\textrm{score}}(R^{(2)}_{2}, z)$, $f_{\textrm{score}}(R^{(3)}_{2}, z)$ and $f_{\textrm{score}}(R^{(4)}_{2}, z)$. A higher score indicates that the corresponding option achieves Target 2 to a greater extent. By Eq.\ (\ref{eq:iteration_1_score_u1}), (\ref{eq:iteration_1_score_u2}), (\ref{eq:iteration_1_score_u3}) and (\ref{eq:iteration_1_score_u4}), $f_{\textrm{score}}(R^{(1)}_{2}, z) > f_{\textrm{score}}(R^{(2)}_{2}, z) > f_{\textrm{score}}(R^{(3)}_{2}, z) > f_{\textrm{score}}(R^{(4)}_{2}, z)$. Therefore, Option 1.1 is the best option among the four and should be chosen. 
Hence, after the for loop in lines 10-19 of GER is fully executed, $x$ is set to be $v_{1} = 2$ and $A$ is set to be $\langle p^{(1)}_{1}, p^{(1)}_{2}, p^{(1)}_{3}, \ldots, p^{(1)}_{8} \rangle = \langle 1, 6, 1, 2, 7, 8, 3, 6 \rangle$. Then, we update $R$ in the following way (line 20 of GER):
\begin{eqnarray}\label{eq:GER_demo_R2}
R \gets R - x A 
& = & R - v_{1} \langle p^{(1)}_{1}, p^{(1)}_{2}, p^{(1)}_{3}, \ldots, p^{(1)}_{8} \rangle \nonumber \\
& = & R^{(1)}_{2} \nonumber \\
& = & \left(
\begin{array}{cccccccc}
30	& 0	& 0	& 0	& 0	& 4	& 0	& 0 \\
0  	& 0  	& 0 	& 34	& 4 	& 13  	& 0  	& 0 \\
0  	& 0  	& 9  	& 0  	& 0  	& 0 	& 59  	& 0 \\
0 	& 15  	& 0  	& 0  	& 0  	& 0  	& 0  	& 0 \\
0 	& 15 	& 30  	& 0  	& 0 	& 13  	& 0  	& 0 \\
29 	& 29 	& 20  	& 0 	& 25 	& 29  	& 0  	& 4 \\
0  	& 0  	& 0  	& 0 	& 30  	& 0  	& 0  	& 0 \\
0  	& 0  	& 0 	& 25  	& 0  	& 0  	& 0 	& 55 \\
\end{array}
\right).
\end{eqnarray}
Then, we append $x$ to $\mathbf{x}$ and append $A$ to $\mathbf{BN}$ (line 21 of GER) so that $\mathbf{x}$ becomes $[2]$ and $\mathbf{BN}$ becomes $[\langle 1, 6, 1, 2, 7, 8, 3, 6 \rangle]$. Because $R$ is not equal to $\mathbf{O}_{8 \times 8}$, we enter the 2\textsuperscript{nd} iteration of the while loop in lines 5-21.

We again try to choose a positive real number $x$ and an $8 \times 8$ BN matrix $A$, and subtract $xA$ from $R$. 
From Eq.\@ (\ref{eq:GER_demo_R2}), we can see that the largest entries in the 1\textsuperscript{st} to 8\textsuperscript{th} columns of $R$ are $30$, $29$, $30$, $34$, $30$, $29$, $59$ and $55$ respectively.
Therefore, $B$ is set to be $\min(30, 29, 30, 34, 30, 29, 59, 55) = 29$ (line 7).
Similar to the 1\textsuperscript{st} iteration of the while loop, the real number $x$ to be chosen in the 2\textsuperscript{nd} iteration must equal some positive entry value of $R$ not greater than $B$.

We would like to choose $x$ and $A$ in such a way that:
\begin{description}
\item[(Target 1')] $\mathcal{N}^{+}(R_{2}) - \mathcal{N}^{+}(R_{3})$ is maximized, and that

\item[(Target 2')] $R_{3} = R_{2} - x A$ contains positive entries with high column frequencies. This is indicated by a high score $f_{\textrm{score}}(R_{3}, z)$.
\end{description}

To achieve Target 1', we consider the positive entries in $R$ (Eq.\ (\ref{eq:GER_demo_R2})) that are not greater than $B = 29$ and have the highest column frequencies in $R$. They are:
\begin{enumerate}
\item the value $v_{1} \coloneqq 4 = R_{2}(2, 5) = R_{2}(1, 6) = R_{2}(6, 8)$, and

\item the value $v_{2} \coloneqq 29 = R_{2}(6, 1) = R_{2}(6, 2) = R_{2}(6, 6)$.
\end{enumerate}
Then, we can see that if we choose $x$ and $A$ in one of the following two ways, then $\mathcal{N}^{+}(R_{2}) - \mathcal{N}^{+}(R_{3})$ is maximized and equals $3$ (thus achieving Target 1'):
\begin{description}
\item[(Option 2.1)]
Set $x$ to be $v_{1} = 4$, and set $A$ to be some BN matrix $\langle p^{(1)}_{1}, p^{(1)}_{2}, p^{(1)}_{3}, \ldots, p^{(1)}_{8} \rangle$ such that the constraints $p^{(1)}_{5} = 2$, $p^{(1)}_{6} = 1$ and $p^{(1)}_{8} = 6$ hold (call this collection of constraints $\mathcal{C}^{(1)}_{2}$);

\item[(Option 2.2)]
Set $x$ to be $v_{2} = 29$, and set $A$ to be some BN matrix $\langle p^{(2)}_{1}, p^{(2)}_{2}, p^{(2)}_{3}, \ldots, p^{(2)}_{8} \rangle$ such that the constraints $p^{(2)}_{1} = 6$, $p^{(2)}_{2} = 6$ and $p^{(2)}_{6} = 6$ hold (call this collection of constraints $\mathcal{C}^{(2)}_{2}$).
\end{description}

We can now see why in the 1\textsuperscript{st} iteration of lines 6-21 of GER, the algorithm tries to choose $x$ and $A$ in such a way that Target 2 is achieved to a greater extent. Thanks to this, $R_{2} = R_{1} - x A$ contains $4$ and $29$ as two positive entries not greater than $B = 29$ whose column frequencies in $R_{2}$ equal 3. Therefore, two options (Option 2.1 and Option 2.2) each of which can maximize $\mathcal{N}^{+}(R_{2}) - \mathcal{N}^{+}(R_{3})$ to $3$ are available for GER to choose to achieve Target 1'. To put it differently, achieving Target 2 in the 1\textsuperscript{st} iteration of lines 6-21 of GER allows the algorithm to ``remove'' more positive entries from $R_{2}$ to form $R_{3}$ in the 2\textsuperscript{nd} iteration, which heuristically may lead to the output ($x_{1}, x_{2}, \ldots, x_{K^{'}} > 0$ and BN matrices $A_{1}, A_{2}, \ldots, A_{K^{'}}$) of GER involving fewer distinct BN matrices (smaller $K^{'}$). Similarly, achieving Target 2' in the 2\textsuperscript{nd} iteration of lines 6-21 of GER allows the algorithm to ``remove'' more positive entries from $R_{3}$ to form $R_{4}$ in the 3\textsuperscript{rd} iteration.
To understand further why achieving Target 2 to a greater extent in the 1\textsuperscript{st} iteration is useful, we can compare Option 1.1 and Option 1.4. Note that the score for Option 1.4 (Eq.\@ (\ref{eq:iteration_1_score_u4})) is lower than the score for Option 1.1 (Eq.\@ (\ref{eq:iteration_1_score_u1})). If we choose Option 1.4 in the 1\textsuperscript{st} iteration instead of Option 1.1, then at the beginning of the 2\textsuperscript{nd} iteration, $R$ equals $R^{(4)}_{2}$ instead of $R^{(1)}_{2}$. To achieve Target 1', there is only one option, which is to set $x$ and $A$ to be $2$ and some BN matrix $\langle p_{1}, p_{2}, p_{3}, \ldots, p_{8} \rangle$ such that $p_{2} = 6$, $p_{3} = 1$ and $p_{6} = 8$ (see Eq.\@ (\ref{eq:iteration_1_choice_v4})). In this case, we can still maximize $\mathcal{N}^{+}(R_{2}) - \mathcal{N}^{+}(R_{3})$ to $3$ (Target 1'). However, we prefer the existence of more options for setting $x$ and $A$ that can maximize $\mathcal{N}^{+}(R_{2}) - \mathcal{N}^{+}(R_{3})$, because heuristically, having more options gives us more choices to choose from so that Target 2' may be achieved to a greater extent.

GER ``explores'' each of Options 2.1 and 2.2, and then determines which of them is better (in terms of Target 2') and should be chosen.
Again, GERESA is executed on $R$ and $v_{j}$ to produce BN matrix $\langle p^{(j)}_{1}, p^{(j)}_{2}, p^{(j)}_{3}, \ldots, p^{(j)}_{8} \rangle$, $j=1,2$. GER then computes and compares the scores $f_{\textrm{score}}(R^{(1)}_{3}, z)$ and $f_{\textrm{score}}(R^{(2)}_{3}, z)$ and sets $x$ and $A$ to be the pair $v_{j}$ and $\langle p^{(j)}_{1}, p^{(j)}_{2}, p^{(j)}_{3}, \ldots, p^{(j)}_{8} \rangle$ that gives the higher score.

We first consider Option 2.1. GERESA outputs the BN matrix $\langle p^{(1)}_{1}, p^{(1)}_{2}, p^{(1)}_{3}, \ldots, p^{(1)}_{8} \rangle = \langle 6, 6, 3, 2, 2, 1, 3, 6 \rangle$ when $R$ and $v_{1}$ are input. Hence,
\begin{eqnarray}\label{eq:iteration_2_choice_v1}
R^{(1)}_{3}
& = & R - 4 \langle 6, 6, 3, 2, 2, 1, 3, 6 \rangle \nonumber \\
& = & \left(
\begin{array}{cccccccc}
30	& 0	& 0	& 0	& 0	& 0	& 0	& 0 \\
0  	& 0  	& 0 	& 30	& 0 	& 13  	& 0  	& 0 \\
0  	& 0  	& 5  	& 0  	& 0  	& 0 	& 55  	& 0 \\
0 	& 15  	& 0  	& 0  	& 0  	& 0  	& 0  	& 0 \\
0 	& 15 	& 30  	& 0  	& 0 	& 13  	& 0  	& 0 \\
25 	& 25 	& 20  	& 0 	& 25 	& 29  	& 0  	& 0 \\
0  	& 0  	& 0  	& 0 	& 30  	& 0  	& 0  	& 0 \\
0  	& 0  	& 0 	& 25  	& 0  	& 0  	& 0 	& 55 \\
\end{array}
\right).
\end{eqnarray}
GER then computes $f_{\textrm{score}}(R^{(1)}_{3}, z)$ to be $20150$.


%

Next, we consider Option 2.2. GERESA outputs the BN matrix $\langle p^{(2)}_{1}, p^{(2)}_{2}, p^{(2)}_{3}, \ldots, p^{(2)}_{8} \rangle = \langle 6, 6, 5, 2, 7, 6, 3, 8 \rangle$ when $R$ and $v_{2}$ are input. Hence,
\begin{eqnarray}\label{eq:iteration_2_choice_v2}
R^{(2)}_{3}
& = & R - 29 \langle 6, 6, 5, 2, 7, 6, 3, 8 \rangle \nonumber \\
& = & \left(
\begin{array}{cccccccc}
30	& 0	& 0	& 0	& 0	& 4	& 0	& 0 \\
0  	& 0  	& 0 	& 5	& 4 	& 13  	& 0  	& 0 \\
0  	& 0  	& 9  	& 0  	& 0  	& 0 	& 30  	& 0 \\
0 	& 15  	& 0  	& 0  	& 0  	& 0  	& 0  	& 0 \\
0 	& 15 	& 1  	& 0  	& 0 	& 13  	& 0  	& 0 \\
0 	& 0 	& 20  	& 0 	& 25 	& 0  	& 0  	& 4 \\
0  	& 0  	& 0  	& 0 	& 1  	& 0  	& 0  	& 0 \\
0  	& 0  	& 0 	& 25  	& 0  	& 0  	& 0 	& 26 \\
\end{array}
\right).
\end{eqnarray}
GER then computes $f_{\textrm{score}}(R^{(2)}_{3}, z)$ to be $1360$.


%

Now, we are able to determine which of the two options (Options 2.1 and 2.2) is better and should be chosen. In terms of Target 1', both options are equally good because for $j = 1, 2$, $\mathcal{N}^{+}(R) - \mathcal{N}^{+}(R^{(j)}_{3}) = 20 - 17 = 3$ (see Eq.\ (\ref{eq:GER_demo_R2}), (\ref{eq:iteration_2_choice_v1}), (\ref{eq:iteration_2_choice_v2})). Therefore, we just need to evaluate the two options only in terms of Target 2'. We does so by comparing the scores $f_{\textrm{score}}(R^{(1)}_{3}, z)$ and $f_{\textrm{score}}(R^{(2)}_{3}, z)$. A higher score indicates that the corresponding option can lead to a greater achievement of Target 2'. Because $f_{\textrm{score}}(R^{(1)}_{3}, z) > f_{\textrm{score}}(R^{(2)}_{3}, z)$, Option 2.1 is the better option among the two and should be chosen. Hence, after the for loop in line 10-19 of GER is fully executed,
$x$ is set to be $v_{1} = 4$, and $A$ is set to be  $\langle p^{(1)}_{1}, p^{(1)}_{2}, p^{(1)}_{3}, \ldots, p^{(1)}_{8} \rangle = \langle 6, 6, 3, 2, 2, 1, 3, 6 \rangle$. Then, we update $R$ in the following way (line 20 of GER):
\begin{eqnarray}\label{eq:GER_demo_R3}
R \gets R - xA 
& = & R - v_{1} \langle p^{(1)}_{1}, p^{(1)}_{2}, p^{(1)}_{3}, \ldots, p^{(1)}_{8} \rangle \nonumber \\
& = & R^{(1)}_{3} \nonumber \\
& = & \left(
\begin{array}{cccccccc}
30	& 0	& 0	& 0	& 0	& 0	& 0	& 0 \\
0  	& 0  	& 0 	& 30	& 0 	& 13  	& 0  	& 0 \\
0  	& 0  	& 5  	& 0  	& 0  	& 0 	& 55  	& 0 \\
0 	& 15  	& 0  	& 0  	& 0  	& 0  	& 0  	& 0 \\
0 	& 15 	& 30  	& 0  	& 0 	& 13  	& 0  	& 0 \\
25 	& 25 	& 20  	& 0 	& 25 	& 29  	& 0  	& 0 \\
0  	& 0  	& 0  	& 0 	& 30  	& 0  	& 0  	& 0 \\
0  	& 0  	& 0 	& 25  	& 0  	& 0  	& 0 	& 55 \\
\end{array}
\right).
\end{eqnarray}
Then, we append $x$ to $\mathbf{x}$ and append $A$ to $\mathbf{BN}$ (line 21 of GER) so that $\mathbf{x}$ becomes $[2, 4]$ and $\mathbf{BN}$ becomes $[\langle 1, 6, 1, 2, 7, 8, 3, 6 \rangle, \langle 6, 6, 3, 2, 2, 1, 3, 6 \rangle]$. Because $R$ is not equal to $\mathbf{O}_{8 \times 8}$, we enter the 3\textsuperscript{rd} iteration of the while loop in lines 5-21.

After executing lines 6-19 of GER, $x$ and $A$ are determined to be $25$ and $\langle 6, 6, 5, 8, 6, 6, 3, 8 \rangle$ respectively. Then, we update $R$ in the following way (line 20 of GER):
\begin{equation}\label{eq:GER_demo_R4}
R \gets R - xA = 
\left(
\begin{array}{cccccccc}
30	& 0	& 0	& 0	& 0	& 0	& 0	& 0 \\
0  	& 0  	& 0 	& 30	& 0 	& 13  	& 0  	& 0 \\
0  	& 0  	& 5  	& 0  	& 0  	& 0 	& 30  	& 0 \\
0 	& 15  	& 0  	& 0  	& 0  	& 0  	& 0  	& 0 \\
0 	& 15 	& 5  	& 0  	& 0 	& 13  	& 0  	& 0 \\
0 	& 0 	& 20  	& 0 	& 0 	& 4  	& 0  	& 0 \\
0  	& 0  	& 0  	& 0 	& 30  	& 0  	& 0  	& 0 \\
0  	& 0  	& 0 	& 0  	& 0  	& 0  	& 0 	& 30 \\
\end{array}
\right).
\end{equation}
Then, we append $x$ to $\mathbf{x}$ and append $A$ to $\mathbf{BN}$ (line 21 of GER) so that $\mathbf{x}$ becomes $[2, 4, 25]$ and $\mathbf{BN}$ becomes $[\langle 1, 6, 1, 2, 7, 8, 3, 6 \rangle, \langle 6, 6, 3, 2, 2, 1, 3, 6 \rangle, \langle 6, 6, 5, 8, 6, 6, 3, 8 \rangle]$. Because $R$ is not equal to $\mathbf{O}_{8 \times 8}$, we enter the 4\textsuperscript{th} iteration of the while loop in lines 5-21.

After executing lines 6-19 of GER, $x$ and $A$ are determined to be $5$ and $\langle 1, 4, 3, 2, 7, 2, 3, 8 \rangle$ respectively. Then, we update $R$ in the following way (line 20 of GER):
\begin{equation}\label{eq:GER_demo_R5}
R \gets R - xA = 
\left(
\begin{array}{cccccccc}
25	& 0	& 0	& 0	& 0	& 0	& 0	& 0 \\
0  	& 0  	& 0 	& 25	& 0 	& 8  	& 0  	& 0 \\
0  	& 0  	& 0  	& 0  	& 0  	& 0 	& 25  	& 0 \\
0 	& 10  	& 0  	& 0  	& 0  	& 0  	& 0  	& 0 \\
0 	& 15 	& 5  	& 0  	& 0 	& 13  	& 0  	& 0 \\
0 	& 0 	& 20  	& 0 	& 0 	& 4  	& 0  	& 0 \\
0  	& 0  	& 0  	& 0 	& 25  	& 0  	& 0  	& 0 \\
0  	& 0  	& 0 	& 0  	& 0  	& 0  	& 0 	& 25 \\
\end{array}
\right).
\end{equation}
Then, we append $x$ to $\mathbf{x}$ and append $A$ to $\mathbf{BN}$ (line 21 of GER) so that $\mathbf{x}$ becomes $[2, 4, 25, 5]$ and $\mathbf{BN}$ becomes $[\langle 1, 6, 1, 2, 7, 8, 3, 6 \rangle, \langle 6, 6, 3, 2, 2, 1, 3, 6 \rangle, \langle 6, 6, 5, 8, 6, 6, 3, 8 \rangle, \langle 1, 4, 3, 2, 7, 2, 3, 8 \rangle]$. Because $R$ is not equal to $\mathbf{O}_{8 \times 8}$, we enter the 5\textsuperscript{th} iteration of the while loop in lines 5-21.

After executing lines 6-19 of GER, $x$ and $A$ are determined to be $10$ and $\langle 1, 4, 6, 2, 7, 5, 3, 8 \rangle$ respectively. Then, we update $R$ in the following way (line 20 of GER):
\begin{equation}\label{eq:GER_demo_R6}
R \gets R - xA =
\left(
\begin{array}{cccccccc}
15	& 0	& 0	& 0	& 0	& 0	& 0	& 0 \\
0  	& 0  	& 0 	& 15	& 0 	& 8  	& 0  	& 0 \\
0  	& 0  	& 0  	& 0  	& 0  	& 0 	& 15  	& 0 \\
0 	& 0  	& 0  	& 0  	& 0  	& 0  	& 0  	& 0 \\
0 	& 15 	& 5  	& 0  	& 0 	& 3  	& 0  	& 0 \\
0 	& 0 	& 10  	& 0 	& 0 	& 4  	& 0  	& 0 \\
0  	& 0  	& 0  	& 0 	& 15  	& 0  	& 0  	& 0 \\
0  	& 0  	& 0 	& 0  	& 0  	& 0  	& 0 	& 15 \\
\end{array}
\right).
\end{equation}
Then, we append $x$ to $\mathbf{x}$ and append $A$ to $\mathbf{BN}$ (line 21 of GER) so that $\mathbf{x}$ becomes $[2, 4, 25, 5, 10]$ and $\mathbf{BN}$ becomes $[\langle 1, 6, 1, 2, 7, 8, 3, 6 \rangle, \langle 6, 6, 3, 2, 2, 1, 3, 6 \rangle, \langle 6, 6, 5, 8, 6, 6, 3, 8 \rangle, \langle 1, 4, 3, 2, 7, 2, 3, 8 \rangle, \\ \langle 1, 4, 6, 2, 7, 5, 3, 8 \rangle]$. Because $R$ is not equal to $\mathbf{O}_{8 \times 8}$, we enter the 6\textsuperscript{th} iteration of the while loop in lines 5-21.

After executing lines 6-19 of GER, $x$ and $A$ are determined to be $5$ and $\langle 1, 5, 5, 2, 7, 2, 3, 8 \rangle$ respectively. Then, we update $R$ in the following way (line 20 of GER):
\begin{equation}\label{eq:GER_demo_R7}
R \gets R - xA =
\left(
\begin{array}{cccccccc}
10	& 0	& 0	& 0	& 0	& 0	& 0	& 0 \\
0  	& 0  	& 0 	& 10	& 0 	& 3  	& 0  	& 0 \\
0  	& 0  	& 0  	& 0  	& 0  	& 0 	& 10  	& 0 \\
0 	& 0  	& 0  	& 0  	& 0  	& 0  	& 0  	& 0 \\
0 	& 10 	& 0  	& 0  	& 0 	& 3  	& 0  	& 0 \\
0 	& 0 	& 10  	& 0 	& 0 	& 4  	& 0  	& 0 \\
0  	& 0  	& 0  	& 0 	& 10  	& 0  	& 0  	& 0 \\
0  	& 0  	& 0 	& 0  	& 0  	& 0  	& 0 	& 10 \\
\end{array}
\right).
\end{equation}
Then, we append $x$ to $\mathbf{x}$ and append $A$ to $\mathbf{BN}$ (line 21 of GER) so that $\mathbf{x}$ becomes $[2, 4, 25, 5, 10, 5]$ and $\mathbf{BN}$ becomes $[\langle 1, 6, 1, 2, 7, 8, 3, 6 \rangle, \langle 6, 6, 3, 2, 2, 1, 3, 6 \rangle, \langle 6, 6, 5, 8, 6, 6, 3, 8 \rangle, \langle 1, 4, 3, 2, 7, 2, 3, 8 \rangle, \\ \langle 1, 4, 6, 2, 7, 5, 3, 8 \rangle, \langle 1, 5, 5, 2, 7, 2, 3, 8 \rangle]$. Because $R$ is not equal to $\mathbf{O}_{8 \times 8}$, we enter the 7\textsuperscript{th} iteration of the while loop in lines 5-21.

After executing lines 6-19 of GER, $x$ and $A$ are determined to be $3$ and $\langle 1, 5, 6, 2, 7, 2, 3, 8 \rangle$ respectively. Then, we update $R$ in the following way (line 20 of GER):
\begin{equation}\label{eq:GER_demo_R8}
R \gets R - xA = 
\left(
\begin{array}{cccccccc}
7	& 0	& 0	& 0	& 0	& 0	& 0	& 0 \\
0  	& 0  	& 0 	& 7	& 0 	& 0  	& 0  	& 0 \\
0  	& 0  	& 0  	& 0  	& 0  	& 0 	& 7  	& 0 \\
0 	& 0  	& 0  	& 0  	& 0  	& 0  	& 0  	& 0 \\
0 	& 7 	& 0  	& 0  	& 0 	& 3  	& 0  	& 0 \\
0 	& 0 	& 7  	& 0 	& 0 	& 4  	& 0  	& 0 \\
0  	& 0  	& 0  	& 0 	& 7  	& 0  	& 0  	& 0 \\
0  	& 0  	& 0 	& 0  	& 0  	& 0  	& 0 	& 7 \\
\end{array}
\right).
\end{equation}
Then, we append $x$ to $\mathbf{x}$ and append $A$ to $\mathbf{BN}$ (line 21 of GER) so that $\mathbf{x}$ becomes $[2, 4, 25, 5, 10, \\ 5, 3]$ and $\mathbf{BN}$ becomes $[\langle 1, 6, 1, 2, 7, 8, 3, 6 \rangle, \langle 6, 6, 3, 2, 2, 1, 3, 6 \rangle, \langle 6, 6, 5, 8, 6, 6, 3, 8 \rangle, \\ \langle 1, 4, 3, 2, 7, 2, 3, 8 \rangle, \langle 1, 4, 6, 2, 7, 5, 3, 8 \rangle, \langle 1, 5, 5, 2, 7, 2, 3, 8 \rangle, \langle 1, 5, 6, 2, 7, 2, 3, 8 \rangle]$. Because $R$ is not equal to $\mathbf{O}_{8 \times 8}$, we enter the 8\textsuperscript{th} iteration of the while loop in lines 5-21.

After executing lines 6-19 of GER, $x$ and $A$ are determined to be $4$ and $\langle 1, 5, 6, 2, 7, 6, 3, 8 \rangle$ respectively. Then, we update $R$ in the following way (line 20 of GER):
\begin{eqnarray}\label{eq:GER_demo_R9}
R \gets R - xA =
\left(
\begin{array}{cccccccc}
3	& 0	& 0	& 0	& 0	& 0	& 0	& 0 \\
0  	& 0  	& 0 	& 3	& 0 	& 0  	& 0  	& 0 \\
0  	& 0  	& 0  	& 0  	& 0  	& 0 	& 3  	& 0 \\
0 	& 0  	& 0  	& 0  	& 0  	& 0  	& 0  	& 0 \\
0 	& 3 	& 0  	& 0  	& 0 	& 3  	& 0  	& 0 \\
0 	& 0 	& 3  	& 0 	& 0 	& 0  	& 0  	& 0 \\
0  	& 0  	& 0  	& 0 	& 3  	& 0  	& 0  	& 0 \\
0  	& 0  	& 0 	& 0  	& 0  	& 0  	& 0 	& 3 \\
\end{array}
\right).
\end{eqnarray}
Then, we append $x$ to $\mathbf{x}$ and append $A$ to $\mathbf{BN}$ (line 21 of GER) so that $\mathbf{x}$ becomes $[2, 4, 25, 5, 10, \\ 5, 3, 4]$ and $\mathbf{BN}$ becomes $[\langle 1, 6, 1, 2, 7, 8, 3, 6 \rangle, \langle 6, 6, 3, 2, 2, 1, 3, 6 \rangle, \langle 6, 6, 5, 8, 6, 6, 3, 8 \rangle, \\ \langle 1, 4, 3, 2, 7, 2, 3, 8 \rangle, \langle 1, 4, 6, 2, 7, 5, 3, 8 \rangle, \langle 1, 5, 5, 2, 7, 2, 3, 8 \rangle, \langle 1, 5, 6, 2, 7, 2, 3, 8 \rangle, \langle 1, 5, 6, 2, 7, 6, 3, 8 \rangle]$. Because $R$ is not equal to $\mathbf{O}_{8 \times 8}$, we enter the 9\textsuperscript{th} iteration of the while loop in lines 5-21.

After executing lines 6-19 of GER, $x$ and $A$ are determined to be $3$ and $\langle 1, 5, 6, 2, 7, 5, 3, 8 \rangle$ respectively. Then, we update $R$ in the following way (line 20 of GER):
\begin{equation}\label{eq:GER_demo_R10}
R \gets R - xA = \mathbf{O}_{8 \times 8}.
\end{equation}
Then, we append $x$ to $\mathbf{x}$ and append $A$ to $\mathbf{BN}$ (line 21 of GER) so that $\mathbf{x}$ becomes $[2, 4, 25, 5, 10, \\ 5, 3, 4, 3]$ and $\mathbf{BN}$ becomes $[\langle 1, 6, 1, 2, 7, 8, 3, 6 \rangle, \langle 6, 6, 3, 2, 2, 1, 3, 6 \rangle, \langle 6, 6, 5, 8, 6, 6, 3, 8 \rangle, \\ \langle 1, 4, 3, 2, 7, 2, 3, 8 \rangle, \langle 1, 4, 6, 2, 7, 5, 3, 8 \rangle, \langle 1, 5, 5, 2, 7, 2, 3, 8 \rangle, \langle 1, 5, 6, 2, 7, 2, 3, 8 \rangle, \langle 1, 5, 6, 2, 7, 6, 3, 8 \rangle, \\ \langle 1, 5, 6, 2, 7, 5, 3, 8 \rangle]$. Because $R$ is equal to $\mathbf{O}_{8 \times 8}$, we exit the while loop in lines 5-21. Finally, we output $\mathbf{x}$ and $\mathbf{BN}$ (line 22 of GER).

Write $\mathbf{x} \eqqcolon [x_{1}, x_{2}, \ldots, x_{9}]$ and $\mathbf{BN} \eqqcolon [A_{1}, A_{2}, \ldots, A_{9}]$. It can be easily checked that each $x_{i}$ is positive, $\sum^{9}_{i = 1} x_{i} = 61 = r_{0}$ and $\sum^{9}_{i = 1} x_{i} A_{i} = Q$.

From the outputs $\mathbf{x}$ and $\mathbf{BN}$, we can construct a $3$-node PBN $\mathcal{P} = (V, F_{1}, F_{2}, F_{3}, \mathcal{D})$ whose PBN matrix equals the TPM $\frac{1}{r_{0}}Q = Q^{'}$. Let $V=\{v_{1}, v_{2}, v_{3}\}$. The BNs represented by the BN matrices $A_{1}, A_{2}, \ldots, A_{9}$ are given in Tables~\ref{table:BN_represented_by_A1_GER}--\ref{table:BN_represented_by_A9_GER}:

\begin{table}[h]
\begin{center}
\begin{tabular}{|c|ccc|c|c|c|}
\hline 
BN Global States & $v_{1}(t)$ 	& $v_{2}(t)$ 	& $v_{3}(t)$ 	& $f^{(1)}_{1}$ 	& $f^{(2)}_{1}$ 	& $f^{(3)}_{1}$ 	\\
\hline
1 & 0 & 0 & 0 & 0 & 0 & 0 \\
2 & 0 & 0 & 1 & 1 & 0 & 1 \\
3 & 0 & 1 & 0 & 0 & 0 & 0 \\
4 & 0 & 1 & 1 & 0 & 0 & 1 \\
5 & 1 & 0 & 0 & 1 & 1 & 0 \\
6 & 1 & 0 & 1 & 1 & 1 & 1 \\
7 & 1 & 1 & 0 & 0 & 1 & 0 \\
8 & 1 & 1 & 1 & 1 & 0 & 1 \\
\hline
\end{tabular}
\caption{The truth table for the BN represented by the BN matrix $A_{1}$.}
\label{table:BN_represented_by_A1_GER}
\end{center}
\end{table}

\FloatBarrier

\begin{table}[h]
\begin{center}
\begin{tabular}{|c|ccc|c|c|c|}
\hline 
BN Global States & $v_{1}(t)$ 	& $v_{2}(t)$ 	& $v_{3}(t)$ 	& $f^{(1)}_{2}$ 	& $f^{(2)}_{2}$ 	& $f^{(3)}_{2}$ 	\\
\hline
1 & 0 & 0 & 0 & 1 & 0 & 1 \\
2 & 0 & 0 & 1 & 1 & 0 & 1 \\
3 & 0 & 1 & 0 & 0 & 1 & 0 \\
4 & 0 & 1 & 1 & 0 & 0 & 1 \\
5 & 1 & 0 & 0 & 0 & 0 & 1 \\
6 & 1 & 0 & 1 & 0 & 0 & 0 \\
7 & 1 & 1 & 0 & 0 & 1 & 0 \\
8 & 1 & 1 & 1 & 1 & 0 & 1 \\
\hline
\end{tabular}
\caption{The truth table for the BN represented by the BN matrix $A_{2}$.}
\label{table:BN_represented_by_A2_GER}
\end{center}
\end{table}

\FloatBarrier

\begin{table}[h]
\begin{center}
\begin{tabular}{|c|ccc|c|c|c|}
\hline 
BN Global States & $v_{1}(t)$ 	& $v_{2}(t)$ 	& $v_{3}(t)$ 	& $f^{(1)}_{3}$ 	& $f^{(2)}_{3}$ 	& $f^{(3)}_{3}$ 	\\
\hline
1 & 0 & 0 & 0 & 1 & 0 & 1 \\
2 & 0 & 0 & 1 & 1 & 0 & 1 \\
3 & 0 & 1 & 0 & 1 & 0 & 0 \\
4 & 0 & 1 & 1 & 1 & 1 & 1 \\
5 & 1 & 0 & 0 & 1 & 0 & 1 \\
6 & 1 & 0 & 1 & 1 & 0 & 1 \\
7 & 1 & 1 & 0 & 0 & 1 & 0 \\
8 & 1 & 1 & 1 & 1 & 1 & 1 \\
\hline
\end{tabular}
\caption{The truth table for the BN represented by the BN matrix $A_{3}$.}
\label{table:BN_represented_by_A3_GER}
\end{center}
\end{table}

\FloatBarrier

\begin{table}[h]
\begin{center}
\begin{tabular}{|c|ccc|c|c|c|}
\hline 
BN Global States & $v_{1}(t)$ 	& $v_{2}(t)$ 	& $v_{3}(t)$ 	& $f^{(1)}_{4}$ 	& $f^{(2)}_{4}$ 	& $f^{(3)}_{1}$ 	\\
\hline
1 & 0 & 0 & 0 & 0 & 0 & 0 \\
2 & 0 & 0 & 1 & 0 & 1 & 1 \\
3 & 0 & 1 & 0 & 0 & 1 & 0 \\
4 & 0 & 1 & 1 & 0 & 0 & 1 \\
5 & 1 & 0 & 0 & 1 & 1 & 0 \\
6 & 1 & 0 & 1 & 0 & 0 & 1 \\
7 & 1 & 1 & 0 & 0 & 1 & 0 \\
8 & 1 & 1 & 1 & 1 & 1 & 1 \\
\hline
\end{tabular}
\caption{The truth table for the BN represented by the BN matrix $A_{4}$.}
\label{table:BN_represented_by_A4_GER}
\end{center}
\end{table}

\FloatBarrier

\begin{table}[h]
\begin{center}
\begin{tabular}{|c|ccc|c|c|c|}
\hline 
BN Global States & $v_{1}(t)$ 	& $v_{2}(t)$ 	& $v_{3}(t)$ 	& $f^{(1)}_{5}$ 	& $f^{(2)}_{5}$ 	& $f^{(3)}_{4}$ 	\\
\hline
1 & 0 & 0 & 0 & 0 & 0 & 0 \\
2 & 0 & 0 & 1 & 0 & 1 & 1 \\
3 & 0 & 1 & 0 & 1 & 0 & 1 \\
4 & 0 & 1 & 1 & 0 & 0 & 1 \\
5 & 1 & 0 & 0 & 1 & 1 & 0 \\
6 & 1 & 0 & 1 & 1 & 0 & 0 \\
7 & 1 & 1 & 0 & 0 & 1 & 0 \\
8 & 1 & 1 & 1 & 1 & 1 & 1 \\
\hline
\end{tabular}
\caption{The truth table for the BN represented by the BN matrix $A_{5}$.}
\label{table:BN_represented_by_A5_GER}
\end{center}
\end{table}

\FloatBarrier

\begin{table}[h]
\begin{center}
\begin{tabular}{|c|ccc|c|c|c|}
\hline 
BN Global States & $v_{1}(t)$ 	& $v_{2}(t)$ 	& $v_{3}(t)$ 	& $f^{(1)}_{6}$ 	& $f^{(2)}_{6}$ 	& $f^{(3)}_{5}$ 	\\
\hline
1 & 0 & 0 & 0 & 0 & 0 & 0 \\
2 & 0 & 0 & 1 & 1 & 0 & 0 \\
3 & 0 & 1 & 0 & 1 & 0 & 0 \\
4 & 0 & 1 & 1 & 0 & 0 & 1 \\
5 & 1 & 0 & 0 & 1 & 1 & 0 \\
6 & 1 & 0 & 1 & 0 & 0 & 1 \\
7 & 1 & 1 & 0 & 0 & 1 & 0 \\
8 & 1 & 1 & 1 & 1 & 1 & 1 \\
\hline
\end{tabular}
\caption{The truth table for the BN represented by the BN matrix $A_{6}$.}
\label{table:BN_represented_by_A6_GER}
\end{center}
\end{table}

\FloatBarrier

\begin{table}[h]
\begin{center}
\begin{tabular}{|c|ccc|c|c|c|}
\hline 
BN Global States & $v_{1}(t)$ 	& $v_{2}(t)$ 	& $v_{3}(t)$ 	& $f^{(1)}_{6}$ 	& $f^{(2)}_{6}$ 	& $f^{(3)}_{6}$ 	\\
\hline
1 & 0 & 0 & 0 & 0 & 0 & 0 \\
2 & 0 & 0 & 1 & 1 & 0 & 0 \\
3 & 0 & 1 & 0 & 1 & 0 & 1 \\
4 & 0 & 1 & 1 & 0 & 0 & 1 \\
5 & 1 & 0 & 0 & 1 & 1 & 0 \\
6 & 1 & 0 & 1 & 0 & 0 & 1 \\
7 & 1 & 1 & 0 & 0 & 1 & 0 \\
8 & 1 & 1 & 1 & 1 & 1 & 1 \\
\hline
\end{tabular}
\caption{The truth table for the BN represented by the BN matrix $A_{7}$.}
\label{table:BN_represented_by_A7_GER}
\end{center}
\end{table}

\FloatBarrier

\begin{table}[h]
\begin{center}
\begin{tabular}{|c|ccc|c|c|c|}
\hline 
BN Global States & $v_{1}(t)$ 	& $v_{2}(t)$ 	& $v_{3}(t)$ 	& $f^{(1)}_{7}$ 	& $f^{(2)}_{6}$ 	& $f^{(3)}_{6}$ 	\\
\hline
1 & 0 & 0 & 0 & 0 & 0 & 0 \\
2 & 0 & 0 & 1 & 1 & 0 & 0 \\
3 & 0 & 1 & 0 & 1 & 0 & 1 \\
4 & 0 & 1 & 1 & 0 & 0 & 1 \\
5 & 1 & 0 & 0 & 1 & 1 & 0 \\
6 & 1 & 0 & 1 & 1 & 0 & 1 \\
7 & 1 & 1 & 0 & 0 & 1 & 0 \\
8 & 1 & 1 & 1 & 1 & 1 & 1 \\
\hline
\end{tabular}
\caption{The truth table for the BN represented by the BN matrix $A_{8}$.}
\label{table:BN_represented_by_A8_GER}
\end{center}
\end{table}

\FloatBarrier

\begin{table}[h]
\begin{center}
\begin{tabular}{|c|ccc|c|c|c|}
\hline 
BN Global States & $v_{1}(t)$ 	& $v_{2}(t)$ 	& $v_{3}(t)$ 	& $f^{(1)}_{7}$ 	& $f^{(2)}_{6}$ 	& $f^{(3)}_{7}$ 	\\
\hline
1 & 0 & 0 & 0 & 0 & 0 & 0 \\
2 & 0 & 0 & 1 & 1 & 0 & 0 \\
3 & 0 & 1 & 0 & 1 & 0 & 1 \\
4 & 0 & 1 & 1 & 0 & 0 & 1 \\
5 & 1 & 0 & 0 & 1 & 1 & 0 \\
6 & 1 & 0 & 1 & 1 & 0 & 0 \\
7 & 1 & 1 & 0 & 0 & 1 & 0 \\
8 & 1 & 1 & 1 & 1 & 1 & 1 \\
\hline
\end{tabular}
\caption{The truth table for the BN represented by the BN matrix $A_{9}$.}
\label{table:BN_represented_by_A9_GER}
\end{center}
\end{table}

\FloatBarrier

The lists of Boolean functions associated with the nodes $v_{1}, v_{2}, v_{3}$ are respectively
$F_{1} \coloneqq \left( f^{(1)}_{1}, f^{(1)}_{2}, f^{(1)}_{3}, f^{(1)}_{4}, f^{(1)}_{5}, f^{(1)}_{6}, f^{(1)}_{7} \right)$, 
$F_{2} \coloneqq \left( f^{(2)}_{1}, f^{(2)}_{2}, f^{(2)}_{3}, f^{(2)}_{4}, f^{(2)}_{5}, f^{(2)}_{6} \right)$
and $F_{3} \coloneqq \left( f^{(3)}_{1}, f^{(3)}_{2}, \right. \\ \left. f^{(3)}_{3}, f^{(3)}_{4}, f^{(3)}_{5}, f^{(3)}_{6}, f^{(3)}_{7} \right)$. 
In addition, the probability distribution $\mathcal{D}$ on $[7] \times [6] \times [7]$ is given by the probability mass function
\begin{eqnarray*}
& & \frac{x_1}{r_0} \mathbbm{1}_{\{(1, 1, 1)\}} 
+ \frac{x_2}{r_0} \mathbbm{1}_{\{(2, 2, 2)\}} 
+ \frac{x_3}{r_0} \mathbbm{1}_{\{(3, 3, 3)\}} 
+ \frac{x_4}{r_0} \mathbbm{1}_{\{(4, 4, 1)\}} 
+ \frac{x_5}{r_0} \mathbbm{1}_{\{(5, 5, 4)\}} \\
& + & \frac{x_6}{r_0} \mathbbm{1}_{\{(6, 6, 5)\}} 
+ \frac{x_7}{r_0} \mathbbm{1}_{\{(6, 6, 6)\}} 
+ \frac{x_8}{r_0} \mathbbm{1}_{\{(7, 6, 6)\}} 
+ \frac{x_9}{r_0} \mathbbm{1}_{\{(7, 6, 7)\}} \\
& = & \frac{2}{61} \mathbbm{1}_{\{(1, 1, 1)\}} 
+ \frac{3}{61} \mathbbm{1}_{\{(6, 6, 6), (7, 6, 7)\}} 
+ \frac{4}{61} \mathbbm{1}_{\{(2, 2, 2), (7, 6, 6)\}} \\
& + & \frac{5}{61} \mathbbm{1}_{\{(4, 4, 1), (6, 6, 5)\}}
+ \frac{10}{61} \mathbbm{1}_{\{(5, 5, 4)\}}
+ \frac{25}{61} \mathbbm{1}_{\{(3, 3, 3)\}}.
\end{eqnarray*}
Consequently, the TPM $\frac{1}{r_{0}} Q = Q^{'}$ is the PBN matrix of the PBN
$\mathcal{P} = (V, F_{1}, F_{2}, F_{3}, \mathcal{D})$.

\subsection{Numerical Demonstration of the GER Entry Selection Algorithm}\label{subsection:GERESA_demo}

In Section \ref{subsection:GER_demo}, we explain the execution of GER holistically without going into the details of the executions of GERESA. In this section, we are going to demonstrate the execution of GERESA on the same matrix $Q$ (Eq.\@ (\ref{eq:example_integral_PBN_matrix})) with $v \coloneqq 2$ in detail. Concerning the numerical demonstration of GER in Section \ref{subsection:GER_demo}, we remark that GER has called GERESA with $Q$ and $2$ as inputs when it ``explores'' Option 1.1 in the 1\textsuperscript{st} iteration of the while statement.

Before we go into the step-by-step demonstration of GERESA, we further remark that lines 6-11 of GERESA determine the values of those $p_{j}$'s when the $j$-th column of $Q$ contains $v$ while lines 12-17 of GERESA determine the values of those $p_{j}$'s when the $j$-th column of $Q$ does not contain $v$. The latter is the key for creating positive entries in $Q - vA$ with high column frequencies (Target-2-related).

We now present the step-by-step execution of GERESA on $Q$ and $v = 2$.

In the initialization phase of GERESA (lines 1-5), we create the variables $p_{1}$, $p_{2}$, $p_{3}$, \ldots, $p_{8}$, initialize $R_{\textrm{copy}}$ to $Q$, and initialize $\mathtt{selected\_columns}$ to an empty list.

Then, we enter the phase of GERESA in which the values of $p_{1}, p_{2}, p_{3}, \ldots, p_{8}$ are set (lines 6-17). Note that $\textrm{Col\_indices}(v, Q) = \{ 3, 6 \}$, because the 3\textsuperscript{rd} and the $6\textsuperscript{th}$ columns of $Q$ are the only columns that contain $v = 2$ as a positive entry. Therefore, the values of $p_{3}$ and $p_{6}$ are set first (lines 6-11) before the values of the other $p_{j}$'s are set (lines 12-17).

Firstly, we consider the 3\textsuperscript{rd} column of $Q$ so that $p_{3}$ will be determined. Note that $\textrm{Occurrences}(v, Q(:, 3)) = \{ 1 \}$. Therefore, $p_{3}$ is set to be $1$ (line 7). Then, we set $R_{\textrm{copy}}(p_{3}, 3)$ to be $0$ (line 9). Hence, 
\begin{equation}\label{eq:R_copy_after_3}
R_{\textrm{copy}} = 
\begin{pmatrix}
32	& 0	& [0]	& 0	& 0	& 4	& 0	& 0 \\
0  	& 0  	& 0 	& 36	& 4 	& 13  	& 0  	& 0 \\
0  	& 0  	& 9  	& 0  	& 0  	& 0 	& 61  	& 0 \\
0 	& 15  	& 0  	& 0  	& 0  	& 0  	& 0  	& 0 \\
0 	& 15 	& 30  	& 0  	& 0 	& 13  	& 0  	& 0 \\
29 	& 31 	& 20  	& 0 	& 25 	& 29  	& 0  	& 6 \\
0  	& 0  	& 0  	& 0 	& 32  	& 0  	& 0  	& 0 \\
0  	& 0  	& 0 	& 25  	& 0  	& 2  	& 0 	& 55
\end{pmatrix},
\end{equation}
where the bracketed entry is the one that has just been updated. Afterwards, we append $3$ to the list $\mathtt{selected\_columns}$ (line 11) so that $\mathtt{selected\_columns}$ becomes $[3]$.

Secondly, we consider the 6\textsuperscript{th} column of $Q$ so that $p_{6}$ will be determined. Note that $\textrm{Occurrences}(v, Q(:, 6)) = \{ 8 \}$. Therefore, $p_{6}$ is set to be $8$ (line 7). Then, we set $R_{\textrm{copy}}(p_{6}, 6)$ to be $0$ (line 9). Hence, 
\begin{equation}\label{eq:R_copy_after_3_6}
R_{\textrm{copy}} = 
\begin{pmatrix}
32	& 0	& [0]	& 0	& 0	& 4	& 0	& 0 \\
0  	& 0  	& 0 	& 36	& 4 	& 13  	& 0  	& 0 \\
0  	& 0  	& 9  	& 0  	& 0  	& 0 	& 61  	& 0 \\
0 	& 15  	& 0  	& 0  	& 0  	& 0  	& 0  	& 0 \\
0 	& 15 	& 30  	& 0  	& 0 	& 13  	& 0  	& 0 \\
29 	& 31 	& 20  	& 0 	& 25 	& 29  	& 0  	& 6 \\
0  	& 0  	& 0  	& 0 	& 32  	& 0  	& 0  	& 0 \\
0  	& 0  	& 0 	& 25  	& 0  	& [0]	& 0 	& 55
\end{pmatrix},
\end{equation}
where the bracketed entries are those that have been updated so far. Afterwards, we append $6$ to the list $\mathtt{selected\_columns}$ (line 11) so that $\mathtt{selected\_columns}$ becomes $[3, 6]$.

We have now finished executing lines 6-11 of GERESA. Therefore, we proceed to executing lines 12-17 of GERESA, in which we will set the values of $p_{1}$, $p_{2}$, $p_{4}$, $p_{5}$, $p_{7}$, $p_{8}$.

Firstly, we consider the 1\textsuperscript{st} column of $Q$ so that $p_{1}$ will be determined. Note that $\textrm{Larger}(v, Q(:, 1)) = \{ 1, 6 \}$, because $Q(1, 1) = 32 > 2$ and $Q(6, 1) = 29 > 2$. Then, we consider the column frequencies of $Q(1, 1) - 2 = 30$ and $Q(6, 1) - 2 = 27$ in the matrix $R_{\textrm{copy}}(:, \mathtt{selected\_columns}) = R_{\textrm{copy}}(:, [3, 6])$. By referring to the columns with bracketed entries in Eq.\@ (\ref{eq:R_copy_after_3_6}), we can see that the column frequencies of $30$ and $27$ in $R_{\textrm{copy}}(:, [3, 6])$ are $1$ and $0$ respectively. Therefore, we set $p_{1}$ to be 1 instead of 6. The rationale behind this choice is as follows:

Note that the 3\textsuperscript{rd} and the 6\textsuperscript{th} columns of $R_{\textrm{copy}}$ are the same as the corresponding columns of $Q - vA$. Because the column frequencies of $30$ and $27$ in $R_{\textrm{copy}}(:, [3, 6])$ are $1$ and $0$ respectively, the column frequencies of $30$ and $27$ in $Q - vA$ will be at least 1 and 0 respectively. Therefore,
\begin{enumerate}
\item If we set $p_{1}$ to be $1$, then the column frequency of the entry value $30$ in $Q - vA$ will be at least $2$, because $Q(1, 1) - v  = 30$;

\item If we set $p_{1}$ to be $6$, then the column frequency of the entry value $27$ in $Q - vA$ will be at least $1$, because $Q(6, 1) - v =  27$.
\end{enumerate}
Therefore, setting $p_{1}$ to be 1 is a better action than setting $p_{1}$ to be $6$. This is because we want $Q - vA$ to have positive entries with high column frequencies (see Target 2 in Section \ref{subsection:GER_demo}).

After setting the value of $p_{1}$, we set $R_{\textrm{copy}}(p_{1}, 1)$ to be $Q(p_{1}, 1) - v = 30$ (line 15). Hence,
\begin{equation}\label{eq:R_copy_after_3_6_1}
R_{\textrm{copy}} = 
\begin{pmatrix}
[30]	& 0	& [0]	& 0	& 0	& 4	& 0	& 0 \\
0  	& 0  	& 0 	& 36	& 4 	& 13  	& 0  	& 0 \\
0  	& 0  	& 9  	& 0  	& 0  	& 0 	& 61  	& 0 \\
0 	& 15  	& 0  	& 0  	& 0  	& 0  	& 0  	& 0 \\
0 	& 15 	& 30  	& 0  	& 0 	& 13  	& 0  	& 0 \\
29 	& 31 	& 20  	& 0 	& 25 	& 29  	& 0  	& 6 \\
0  	& 0  	& 0  	& 0 	& 32  	& 0  	& 0  	& 0 \\
0  	& 0  	& 0 	& 25  	& 0  	& [0]	& 0 	& 55
\end{pmatrix}.
\end{equation}
Afterwards, we append $1$ to the list $\mathtt{selected\_columns}$ (line 17) so that $\mathtt{selected\_columns}$ becomes $[3, 6, 1]$.

Secondly, we consider the 2\textsuperscript{nd} column of $Q$ so that $p_{2}$ will be determined. Note that $\textrm{Larger}(v, Q(:, 2)) = \{ 4, 5, 6 \}$, because $Q(4, 2) = 15 > 2$, $Q(5, 2) = 15 > 2$ and $Q(6, 2) = 31 > 2$. Then, we consider the column frequencies of $Q(4, 2) - 2 = 13$, $Q(5, 2) - 2 = 13$ and $Q(6, 2) - 2 = 29$ in the matrix $R_{\textrm{copy}}(:, \mathtt{selected\_columns}) = R_{\textrm{copy}}(:, [3, 6, 1])$. By referring to the columns with bracketed entries in Eq.\@ (\ref{eq:R_copy_after_3_6_1}), we can see that the column frequencies of $13$ and $29$ in $R_{\textrm{copy}}(:, [3, 6, 1])$ are 1 and 2 respectively. Therefore, we set $p_{2}$ to be $6$ instead of $4$ and $5$ because $Q(6, 2) - 2 = 29$ has a higher associated column frequency than $Q(4, 2) - 2 = 13$ and $Q(5, 2) - 2 = 13$. The rationale behind this choice is that we want $Q - vA$ to have positive entries with high column frequencies (see Target 2 in Section \ref{subsection:GER_demo}).

After setting the value of $p_{2}$, we set $R_{\textrm{copy}}(p_{2}, 2)$ to be $Q(p_{2}, 2) - v = 29$ (line 15). Hence,
\begin{equation}\label{eq:R_copy_after_3_6_1_2}
R_{\textrm{copy}} = 
\begin{pmatrix}
[30]	& 0	& [0]	& 0	& 0	& 4	& 0	& 0 \\
0  	& 0  	& 0 	& 36	& 4 	& 13  	& 0  	& 0 \\
0  	& 0  	& 9  	& 0  	& 0  	& 0 	& 61  	& 0 \\
0 	& 15  	& 0  	& 0  	& 0  	& 0  	& 0  	& 0 \\
0 	& 15 	& 30  	& 0  	& 0 	& 13  	& 0  	& 0 \\
29 	& [29]	& 20  	& 0 	& 25 	& 29  	& 0  	& 6 \\
0  	& 0  	& 0  	& 0 	& 32  	& 0  	& 0  	& 0 \\
0  	& 0  	& 0 	& 25  	& 0  	& [0]	& 0 	& 55
\end{pmatrix}.
\end{equation}
Afterwards, we append $2$ to the list $\mathtt{selected\_columns}$ (line 17) so that $\mathtt{selected\_columns}$ becomes $[3, 6, 1, 2]$.

Thirdly, we consider the 4\textsuperscript{th} column of $Q$ so that $p_{4}$ will be determined. Note that $\textrm{Larger}(v, Q(:, 4)) = \{ 2, 8 \}$. Then, we consider the column frequencies of $Q(2, 4) - 2 = 34$ and $Q(8, 4) - 2 = 23$ in the matrix $R_{\textrm{copy}}(:, \mathtt{selected\_columns}) = R_{\textrm{copy}}(:, [3, 6, 1, 2])$. By referring to the columns with bracketed entries in Eq.\@ (\ref{eq:R_copy_after_3_6_1_2}), we can see that the column frequencies of $34$ and $23$ in $R_{\textrm{copy}}(:, [3, 6, 1, 2])$ both equal $0$. We break the tie by arbitrarily choosing one of 2 and 8. Let’s set $p_{4}$ to be 2 instead of 8 for instance. Then, we set $R_{\textrm{copy}}(p_{4}, 4)$ to be $Q(p_{4}, 4) - v = 34$ (line 15). Hence,
\begin{equation}\label{eq:R_copy_after_3_6_1_2_4}
R_{\textrm{copy}} = 
\begin{pmatrix}
[30]	& 0	& [0]	& 0	& 0	& 4	& 0	& 0 \\
0  	& 0  	& 0 	& [34]	& 4 	& 13  	& 0  	& 0 \\
0  	& 0  	& 9  	& 0  	& 0  	& 0 	& 61  	& 0 \\
0 	& 15  	& 0  	& 0  	& 0  	& 0  	& 0  	& 0 \\
0 	& 15 	& 30  	& 0  	& 0 	& 13  	& 0  	& 0 \\
29 	& [29]	& 20  	& 0 	& 25 	& 29  	& 0  	& 6 \\
0  	& 0  	& 0  	& 0 	& 32  	& 0  	& 0  	& 0 \\
0  	& 0  	& 0 	& 25  	& 0  	& [0]	& 0 	& 55
\end{pmatrix}.
\end{equation}
Afterwards, we append $4$ to the list $\mathtt{selected\_columns}$ (line 17) so that $\mathtt{selected\_columns}$ becomes $[3, 6, 1, 2, 4]$.

We continue to iterate the for statement in lines 12-17 three more times. During this process, $p_{5}$, $p_{7}$, $p_{8}$ are set to be $7$, $3$, $6$ respectively. After we fully execute lines 6-17, 
\begin{equation}\label{eq:R_copy_after_setting_all_pi's}
R_{\textrm{copy}} = 
\begin{pmatrix}
[30]	& 0	& [0]	& 0	& 0	& 4	& 0	& 0 \\
0  	& 0  	& 0 	& [34]	& 4 	& 13  	& 0  	& 0 \\
0  	& 0  	& 9  	& 0  	& 0  	& 0 	& [59]	& 0 \\
0 	& 15  	& 0  	& 0  	& 0  	& 0  	& 0  	& 0 \\
0 	& 15 	& 30  	& 0  	& 0 	& 13  	& 0  	& 0 \\
29 	& [29]	& 20  	& 0 	& 25 	& 29  	& 0  	& [4] \\
0  	& 0  	& 0  	& 0 	& [30]	& 0  	& 0  	& 0 \\
0  	& 0  	& 0 	& 25  	& 0  	& [0]	& 0 	& 55
\end{pmatrix}
= Q - v\langle p_{1}, p_{2}, p_{3}, \ldots, p_{8} \rangle.
\end{equation}
Finally, the BN matrix $A = \langle p_{1}, p_{2}, p_{3}, \ldots, p_{8} \rangle = \langle 1, 6, 1, 2, 7, 8, 3, 6 \rangle$ is output by GERESA.

\subsection{Termination of GER}\label{subsubsection:main_tools_main_thms_for_GER}

In this section, we justify that the GER algorithm terminates for any input matrix $P=r_{0} Q$ where $r_{0} > 0$ and $Q$ is a $2^{n} \times 2^{n}$ TPM.

\begin{lemma}\label{lemma:GERESA}
Let $R = r Q$ where $r > 0$ and $Q$ is a $2^{n} \times 2^{n}$ TPM.
Let $v$ be any positive entry of $R$ such that $v \leq \min\limits_{1\le j\le 2^{n}}\ \max\limits_{1\le i\le 2^{n}}\ R(i,j)$.
When GERESA is executed with $R$ and $v$ as inputs, the algorithm will output a $2^{n} \times 2^{n}$ BN matrix $A = \langle p_{1}, p_{2}, p_{3}, \ldots, p_{2^{n}} \rangle$ such that for all $j \in [2^{n}]$, $R(p_{j}, j) \geq v$, and that there exists $j^{*} \in [2^{n}]$ satisfying $R(p_{j^{*}}, j^{*}) = v$.
\end{lemma}

\begin{proof}
From the pseudocode of GERESA and the comments therein, it is evident that when executed on $R$ and $v$, GERESA will terminate by successfully outputting a $2^{n} \times 2^{n}$ BN matrix $A = \langle p_{1}, p_{2}, p_{3}, \ldots, p_{2^{n}} \rangle$.
Moreover, for each element $j$ of the non-empty set $\textrm{Col\_indices}(v, R)$, $p_{j} \in \textrm{Occurrences}(v, R(:, j))$ and hence $R(p_{j}, j) = v$ (see line 7 of GERESA).
On the other hand, for each $j^{'} \in [2^{n}] \setminus \textrm{Col\_indices}(v, R)$, $p_{j^{'}} \in \textrm{Larger}(v, R(:, j^{'}))$ (see line 13 of GERESA) and hence $R(p_{j^{'}}, j^{'}) > v$.
Therefore, this lemma is proved.
\end{proof}

We now examine lines 5-21 of the GER algorithm. We can regard lines 6-21 as a procedure (called $\lambda_{\textrm{GER}}$) which takes as input a $2^{n} \times 2^{n}$ matrix $R$ and then performs the following steps: 
\begin{enumerate}
\item compute a positive real number $x$ and a $2^{n} \times 2^{n}$ BN matrix $A$, 
\item update $R$ by subtracting $xA$ from it, and 
\item append $x$ and $A$ to the lists $\mathbf{x}$ and $\mathbf{BN}$ respectively.
\end{enumerate}
If $R$ is non-negative and $\vec{1}^{\top}_{2^{n}} R = r \vec{1}^{\top}_{2^{n}}$ for some $r \in (0, r_{0}]$ where $r_{0} > 0$, then one of the following situations will occur:

Case 1: there exists a BN matrix $\tilde{A} = \langle p_{1}, p_{2}, p_{3}, \ldots, p_{2^{n}} \rangle$ such that $R = r \tilde{A}$ (i.e., $R$ is a positive multiple of some BN matrix). Then, $B$ will be set to $r$ (line 7) and $\mathbf{v}$ will be set to the list $[r]$ (line 8). Afterwards, $x$ and $\mathtt{score}$ will be initialized to $0$ and $-\infty$ respectively (line 9). Then, we execute GERESA with $R$ and $r$ as input to produce a BN matrix $\mathtt{temp\_A}$ (line 11), and a real number score $\mathtt{temp\_score}$ is computed (line 13). By Lemma \ref{lemma:GERESA}, $\mathtt{temp\_A} = \tilde{A}$. After lines 15-18 are fully executed, the variables $x$ and $A$ will be set to $r$ and $\mathtt{temp\_A} = \tilde{A}$ respectively. Then, we update $R$ by subtracting $xA = r\tilde{A}$ from it (line 20) and hence $R$ becomes $\mathbf{O}_{2^{n} \times 2^{n}}$. Afterwards, $x$ and $A$ are appended to the lists $\mathbf{x}$ and $\mathbf{BN}$ respectively. We remark that when $R$ becomes $\mathbf{O}_{2^{n} \times 2^{n}}$, the while loop in lines 5-21 will end and then $\mathbf{x}$ and $\mathbf{BN}$ will be output.

Case 2: $R$ is not a positive multiple of some BN matrix. In lines 7-8, the value $B$ and the list $\mathbf{v}$ are determined. In the for loop in lines 10-19, for each element $v$ of the list $\mathbf{v}$, we apply GERESA to $R$ and $v$ to obtain a BN matrix $\mathtt{temp\_A}$, and compute a score $\mathtt{temp\_score}$ for the pair $(v, \mathtt{temp\_A})$. After the for statement in lines 10-19 is fully executed, the variable $x$ will be set to some element $v^{*}$ of $\mathbf{v}$ and the variable $A$ will be set to the BN matrix $\mathtt{GERESA}(R, v^{*}) \eqqcolon \langle p_{1}, p_{2}, p_{3}, \ldots, p_{2^{n}} \rangle$ (refer to the comment in line 19 for explanation). By Lemma $\ref{lemma:GERESA}$, for all $j \in [2^{n}]$, $R(p_{j}, j) \geq v^{*}$, and there exists $j^{*} \in [2^{n}]$ satisfying $R(p_{j^{*}}, j^{*}) = v^{*}$. Hence, all entries of $R - xA = R - v^{*}\langle p_{1}, p_{2}, p_{3}, \ldots, p_{2^{n}} \rangle$ are non-negative and $\mathcal{N}^{+}(R - xA) < \mathcal{N}^{+}(R)$. Moreover, $\vec{1}^{\top}_{2^{n}}(R - xA) = (r - x) \vec{1}^{\top}_{2^{n}}$. Because $R \neq xA$, $R - xA$ contains positive entries and hence $0 < r - x < r_{0}$. In line 20, $R$ is updated to become $R - xA$. Finally, in line 21, we append $x$ and $A$ to the lists $\mathbf{x}$ and $\mathbf{BN}$ respectively.

With the above considerations, we can prove the following important theorem about GER.

\begin{theorem}\label{thm:GER_must_terminate_and_output}
Let $P = r_{0} Q$ where $r_{0} > 0$ and $Q$ is a $2^{n} \times 2^{n}$ TPM.
When GER is executed on $P$ with the score parameter $z$ being any positive real number greater than $1$, the algorithm will eventually terminate at line 22 and output a list of positive real numbers $\mathbf{x} = [x_{1}, x_{2}, \ldots, x_{K}]$ and a list of distinct BN matrices $\mathbf{BN} = [A_{1}, A_{2}, \ldots, A_{K}]$ such that $P = \sum^{K}_{i = 1} x_{i} A_{i}$ and $\sum^{K}_{i = 1} x_{i} = r_{0}$.
In other words, 
$Q$ is the PBN matrix of some PBN which has the decomposition
\[ Q = \sum^{K}_{i=1}\ \frac{x_i}{r_0}\ A_i. \]
\end{theorem}

\begin{proof}
The proof for the fact that the algorithm will eventually terminate at line 22 and the outputs $\mathbf{x} = [x_{1}, x_{2}, \ldots, x_{K}]$, $\mathbf{BN} = [A_{1}, A_{2}, \ldots, A_{K}]$ satisfy $P = \sum^{K}_{i = 1} x_{i} A_{i}$ and $\sum^{K}_{i = 1} x_{i} = r_{0}$, as well as the claims on $Q$, follows the same line of arguments as the proof of the corresponding parts of Theorem \ref{thm:SER*_must_terminate_and_output} so that they are omitted.

It remains to prove that $A_{1}, A_{2}, \ldots, A_{K}$ are distinct. Fix arbitrary $s, t \in [K]$ such that $s < t$. Write $A_{s}$ as $\langle k^{(s)}_{1}, k^{(s)}_{2}, k^{(s)}_{3}, \ldots, k^{(s)}_{2^{n}} \rangle$ and $A_{t}$ as $\langle k^{(t)}_{1}, k^{(t)}_{2}, k^{(t)}_{3}, \ldots, k^{(t)}_{2^{n}} \rangle$.
For all $k \in [K+1]$, define $R_{k} \coloneqq P - \sum^{k-1}_{i = 1} x_{i}A_{i}$.
Note that for all $i, j \in [2^{n}]$, $R_{1}(i, j) \geq R_{2}(i, j) \geq \ldots \geq R_{K+1}(i, j) = 0$, and that for all $k \in [K]$, $A_{k} = \mathtt{GERESA}(R_{k}, x_{k})$.
Because $R_{s+1} = R_{s} - x_{s} A_{s}$ and by Lemma \ref{lemma:GERESA}, $R_{s}(k^{(s)}_{j^{*}}, j^{*}) = x_{s}$ for some $j^{*} \in [2^{n}]$, $R_{s+1}(k^{(s)}_{j^{*}}, j^{*}) = 0$, which implies that $R_{t}(k^{(s)}_{j^{*}}, j^{*}) = 0$.
Because $R_{t}(k^{(t)}_{j^{*}}, j^{*}) \geq x_{t} > 0$ (by Lemma \ref{lemma:GERESA}), $k^{(s)}_{j^{*}} \neq k^{(t)}_{j^{*}}$. Hence, $A_{s} \neq A_{t}$. Since $s$ and $t$ are arbitrary, $A_{1}, A_{2}, \ldots, A_{K}$ are distinct.
\end{proof}

\subsection{Two Upper Bound Theorems for GER}\label{subsection:GER_upper_bounds}

Theorem \ref{thm:GER_must_terminate_and_output} says that for any $2^{n} \times 2^{n}$ TPM $P$, we can obtain a decomposition of $P$ by executing GER with $r_{0} P$ as input, where $r_{0}$ is any positive real number. In this section, we are going to provide two theorems related to the sparsity of the decompositions found using GER.

\begin{theorem}\label{thm:entry_removal_upper_bound_GER}
Let $P$ be a $2^{n} \times 2^{n}$ TPM and let $r_{0}$ be any positive real number.
If $x_{1}, x_{2}, \ldots, x_{K} > 0$ and $A_{1}, A_{2}, \ldots, A_{K}$ (distinct BN matrices) are the outputs when GER is applied to the matrix $r_{0}P$ with the score parameter $z$ being any real number greater than $1$, then $K \leq \mathcal{N}^{+}(P) - 2^{n} + 1$.
\end{theorem}

\begin{proof}
The proof of this theorem follows the same line of arguments as the proof of Theorem \ref{thm:entry_removal_upper_bound_SER*}.
\end{proof}

\begin{theorem}\label{thm:rational_upper_bound_GER}
Let $P$ be a $2^{n} \times 2^{n}$ rational TPM.
Let $p$ be the smallest positive integer such that all entries of $pP$ are integers. 
Let $a^{*}$ be the smallest positive entry of $P$. 
Let $r_{0}$ be any positive real number.
If $x_{1}, x_{2}, \ldots, x_{K} > 0$ and $A_{1}, A_{2}, \ldots, A_{K}$ (distinct BN matrices) are the outputs when GER is applied to the matrix $r_{0}P$ with the score parameter $z$ being any real number greater than $1$, then
$K \leq (1 - a^{*})p + 1$.
\end{theorem}

\begin{proof}
Consider the matrices $R_{1}, R_{2}, \ldots, R_{K+1}$ defined in the proof of Theorem \ref{thm:GER_must_terminate_and_output}. Note that the variable $R$ in GER would be set to $R_{1}, R_{2}, \ldots, R_{K+1}$ in chronological order when the algorithm is executed on $r_{0} P$ with the score parameter $z$.
Note that each entry of $R_{1} = r_{0} P$ is of the form $\frac{r_{0}s}{p}$ for some non-negative integer $s$.
Write $A_{1}$ as $\langle p^{(1)}_{1}, p^{(1)}_{2}, p^{(1)}_{3}, \ldots, p^{(1)}_{2^{n}} \rangle$. 
Note that $x_{1}$ is some positive entry of $R_{1}$ and $A_{1} = \mathtt{GERESA}(R_{1}, x_{1})$. 
Hence, there exists a positive integer $s_{1}$ such that $x_{1} = \frac{r_{0} s_{1}}{p} \geq r_{0} a^{*}$.
Moreover, by Lemma \ref{lemma:GERESA}, for all $j \in [2^{n}]$, $R_{1}(p^{(1)}_{j}, j) \geq x_{1}$.  
Therefore, each entry of $R_{2} = R_{1} - x_{1} A_{1} = R_{1} - x_{1} \langle p^{(1)}_{1}, p^{(1)}_{2}, p^{(1)}_{3}, \ldots, p^{(1)}_{2^{n}} \rangle$ is also of the form $\frac{r_{0} s}{p}$ for some non-negative integer $s$.
Using the same argument, we deduce that for each $k \in [K]$, $x_{k} = \frac{r_{0}s_{k}}{p}$ for some positive integer $s_{k}$.
By Theorem \ref{thm:GER_must_terminate_and_output},
$\sum^{K}_{k = 1} x_{k} = r_{0}$. This implies
\begin{equation*}
1 
= \sum^{K}_{k = 1} \frac{s_{k}}{p}
\geq a^{*} + \sum^{K}_{k = 2} \frac{s_{k}}{p}
\geq a^{*} + \sum^{K}_{k = 2} \frac{1}{p}.
\end{equation*}
Rearranging, we get $K \leq (1 - a^{*}) p + 1$.
\end{proof}

\newpage

\section{Theoretical Results Related to the Lower Bound Problem}\label{section:lower_bound}

Let $P$ be any $2^{n} \times 2^{n}$ TPM. The construction problem of sparse PBNs is to find a decomposition $\sum^{K}_{i = 1} x_{i} A_{i}$ of $P$ such that the length $K$ is as small as possible. This problem can be solved by brute-force: test if $P$ is a convex combination of $k$ elements in $B_n(P)$ iteratively for $k=1,2,3,\ldots$ until a convex combination is found. Obviously, this brute-force checking is associated with high time complexity so that it's impracticable. 

It may be then natural to ask whether the minimum value of $K$ can be determined from the matrix $P$ in advance so that a researcher can safely conclude that a derived decomposition has the optimal length. Unfortunately, no such result is known to our best knowledge. Therefore, we turn to a more tractable problem: finding lower bounds for the length ($K$) of any decomposition of a given TPM $P$, and we call it \textit{the lower bound problem of sparse PBN construction}.
Preferably, such lower bounds should be as sharp as possible. In this section, we propose a series of theoretical results related to the lower bound problem of sparse PBN construction. Among these results, Theorem \ref{thm:non-trivial lower bound} is the strongest and most non-trivial one. These results are applied in Appendix \ref{appendix:GER_gives_17_sparsest_decomposition} to justify that most of the decompositions generated by GER in our numerical experiments (Section~\ref{section:numerical_experiments}) have the optimal lengths.

Let's start with a couple of easy lower bounds.

\begin{proposition}\label{prop:max_P(:,j)_lower_bound}
Let $P$ be a $2^{n} \times 2^{n}$ TPM.
Then, the length of any decomposition of $P$ is at least $\max\limits_{1\leq j \leq 2^{n}} \mathcal{N}^{+} (P(:, j))$.
\end{proposition}

\begin{proof}
Let $\sum^{K}_{r = 1} x_{r} A_{r}$ be any decomposition of $P$.
It suffices to show that for each $j \in [2^{n}]$, $K \geq \mathcal{N}^{+} (P(:, j))$. Because each $A_{r}$ is a BN matrix, the $j$-th column of $A_{r}$ has exactly one non-zero entry and this non-zero entry equals $1$. Therefore, the $j$-th column of $P$ has at most $K$ positive entries, i.e.\@ $K \geq \mathcal{N}^{+} (P(:, j))$.
\end{proof}

\begin{proposition}\label{prop:copying_PBN}
Let $P$ be a $2^{n} \times 2^{n}$ TPM. Consider
$
Q \coloneqq
\begin{pmatrix}
P						& \mathbf{O}_{2^{n} \times 2^{n}} \\
\mathbf{O}_{2^{n} \times 2^{n}}	& P
\end{pmatrix}
$ which is a $2^{n+1} \times 2^{n+1}$ TPM.
Then, for all positive integers $d$, any decomposition $\sum^{K}_{r = 1} x_{r} A_{r}$ of $P$ satisfies $K \geq d$ if and only if any decomposition $\sum^{L}_{s = 1} x^{'}_{s} A^{'}_{s}$ of $Q$ satisfies $L \geq d$.
\end{proposition}

\begin{proof}
Fix arbitrary positive integer $d$. 

First, we prove the ``only if'' part of the statement.
Let $\sum^{L}_{s = 1} x^{'}_{s} A^{'}_{s}$ be an arbitrary decomposition of $Q$.
For each $s \in [L]$, write $A^{'}_{s}$ as 
$\langle p^{(s)}_{1}, p^{(s)}_{2}, p^{(s)}_{3}, \ldots, p^{(s)}_{2^{n+1}} \rangle$.
Define $A_{s} \coloneqq \langle p^{(s)}_{1}, p^{(s)}_{2}, p^{(s)}_{3}, \ldots, p^{(s)}_{2^{n}} \rangle$.
We can see that $\sum^{L}_{s = 1} x^{'}_{s} A_{s}$ is a decomposition of $P$.
By assumption, $L \geq d$.

Next, we prove the ``if'' part of the statement.
Let $\sum^{K}_{r = 1} x_{r} A_{r}$ be an arbitrary decomposition of $P$.
For each $r \in [K]$, write $A_{r}$ as $\langle p^{(r)}_{1}, p^{(r)}_{2}, p^{(r)}_{3}, \ldots, p^{(r)}_{2^{n}} \rangle$.
Define $A^{'}_{r} \coloneqq \langle p^{(r)}_{1}, p^{(r)}_{2}, p^{(r)}_{3}, \ldots, p^{(r)}_{2^{n}}, \\ 2^{n} + p^{(r)}_{1}, 2^{n} + p^{(r)}_{2}, 2^{n} + p^{(r)}_{3}, \ldots, 2^{n} + p^{(r)}_{2^{n}} \rangle$.
We can see that $\sum^{K}_{r = 1} x_{r} A^{'}_{r}$ is a decomposition of $Q$.
By assumption, $K \geq d$.
\end{proof}

We then proceed to establish some non-trivial lower bounds. Their proofs rely on suitable partitions of the index set $[K]$ where $K$ is the length of a decomposition of a $2^{n} \times 2^{n}$ TPM.

\begin{lemma}\label{lemma:convenient_notations}
Let $P$ be a $2^{n} \times 2^{n}$ TPM.
Let $x_{1}, x_{2}, \ldots, x_{K}$ (positive real numbers) and $A_{1}, A_{2}, \ldots, A_{K}$ (distinct BN matrices) define a decomposition of $P$.
Fix arbitrary $i \in [2^{n}]$. Let $P(i_{1}, i)$, $P(i_{2}, i)$, $\ldots$, $P(i_{d}, i)$ be the positive entries of $P(:, i)$.
For all $l \in [d]$, define $X^{i}_{l} \coloneqq \{ r \in [K] : A_{r}(i_{l}, i) = 1 \}$.
Then, $\left\{ X^{i}_{1}, X^{i}_{2}, \ldots, X^{i}_{d} \right\}$ is a partition of $[K]$.
Moreover, $P(i_{l}, i) = \sum_{r \in X^{i}_{l}} x_{r}$.
\end{lemma}

\begin{proof}
Note that for all $l \in [d]$,
\begin{equation}
P(i_{l}, i)
= \sum^{K}_{r = 1} x_{r} A_{r}(i_{l}, i)
= \sum_{r \in X^{i}_{l}} x_{r} A_{r}(i_{l}, i) 
= \sum_{r \in X^{i}_{l}} x_{r}.
\end{equation}

To prove that $\left\{ X^{i}_{1}, \ldots, X^{i}_{d} \right\}$ is a partition of $[K]$, first we note that for all $r \in [K]$, $A_{r}(:, i)$ has exactly one non-zero entry and this non-zero entry is equal to $1$. This implies that $X^{i}_{1}, \ldots, X^{i}_{d}$ are pairwise disjoint. 
Moreover, for all $l \in [d]$, $0 < P(i_{l}, i) = \sum_{r \in X^{i}_{l}} x_{r}$ and hence $X^{i}_{l}$ is non-empty.
Clearly, $X^{i}_{1} \cup \ldots \cup X^{i}_{d} \subseteq [K]$. To prove $[K] \subseteq X^{i}_{1} \cup \ldots \cup X^{i}_{d}$, assume for a contradiction that there exists $r_{0} \in [K]$ such that $r_{0} \notin X^{i}_{1} \cup \ldots \cup X^{i}_{d}$. Then, 
$A_{r_{0}}(i_{1}, i) = \ldots = A_{r_{0}}(i_{d}, i) = 0$, so $A_{r_{0}}(s, i) = 1$ for some $s \in [2^{n}] \setminus \{ i_{1}, \ldots, i_{d} \}$.
Hence,
$P(s, i) = \sum^{K}_{r = 1} x_{r} A_{r}(s, i) \geq x_{r_{0}} A_{r_{0}}(s, i) = x_{r_{0}} > 0$,
which contradicts with the fact that $P(i_{1}, i), \ldots, P(i_{d}, i)$ are all the positive entries in $P(:, i)$. Therefore, $X^{i}_{1} \cup \ldots \cup X^{i}_{d} = [K]$.
\end{proof}

The following technical proposition gives a sufficient condition for the statement that a cell of the above partition is not a singleton. It is useful to establish that certain decompositions of some TPMs have the optimal lengths (Section \ref{section:numerical_experiments} and Appendix \ref{appendix:GER_gives_17_sparsest_decomposition}).

\begin{proposition}\label{prop:al_is_larger_than_max}
Let $P$ be a $2^{n} \times 2^{n}$ TPM.
Let $x_{1}, x_{2}, \ldots, x_{K}$ (positive real numbers) and $A_{1}, A_{2}, \ldots, A_{K}$ (distinct BN matrices) define a decomposition of $P$.
Let $a_{1} \coloneqq P(i_{1}, i)$, $a_{2} \coloneqq P(i_{2}, i)$, \ldots, $a_{d_{1}} \coloneqq P(i_{d_{1}}, i)$ be the positive entries of $P(:, i)$, and
$b_{1} \coloneqq P(j_{1}, j)$, $b_{2} \coloneqq P(j_{2}, j)$, \ldots, $b_{d_{2}} \coloneqq P(j_{d_{2}}, j)$ be the positive entries of $P(:, j)$, where $i \neq j$.
For all $l \in [d_{1}]$, define $X^{i}_{l} \coloneqq \{ r \in [K] : A_{r}(i_{l}, i) = 1 \}$.
For all $l \in [d_{2}]$, define $X^{j}_{l} \coloneqq \{ r \in [K] : A_{r}(j_{l}, j) = 1 \}$.
If $a_{l^{*}} > \max (b_{1}, b_{2}, \ldots, b_{d_{2}})$ for some $l^{*} \in [d_{1}]$,
then $|X^{i}_{l^{*}}| \geq 2$.
\end{proposition}

\begin{proof}
By Lemma \ref{lemma:convenient_notations}, $\left\{ X^{i}_{1}, X^{i}_{2}, \ldots, X^{i}_{d_{1}} \right\}$ and $\left\{ X^{j}_{1}, X^{j}_{2}, \ldots, X^{j}_{d_{2}} \right\}$ are partitions of $[K]$.
Moreover, for all $l \in [d_{1}]$, $a_{l} = \sum_{r \in X^{i}_{l}} x_{r}$.
Similarly, for all $l \in [d_{2}]$, $b_{l} = \sum_{r \in X^{j}_{l}} x_{r}$.
It suffices to show that $|X^{i}_{l^{*}}| \neq 1$.
Suppose for a contradiction that $X^{i}_{l^{*}} = \{ r^{*} \}$ for some $r^{*} \in [K]$.
Then, $a_{l^{*}} = \sum_{r \in X^{i}_{l^{*}}} x_{r} = x_{r^{*}}$.
Let $l^{'} \in [d_{2}]$ such that $r^{*} \in X^{j}_{l^{'}}$.
Then,
\begin{equation*}
\max (b_{1}, b_{2}, \ldots, b_{d_{2}}) \geq b_{l^{'}} = \sum_{r \in X^{j}_{l^{'}}} x_{r}
\geq x_{r^{*}} = a_{l^{*}}.
\end{equation*}
A contradiction is reached.
Therefore, $|X^{i}_{l^{*}}| \geq 2$.
\end{proof}

Below we present the first non-trivial lower bound on $K$.

\begin{proposition}\label{prop:trivial_lower_bound}
Let $P$ be a $2^{n} \times 2^{n}$ TPM.
Let $P(i_{1}, i) \eqqcolon a_{1}$, $P(i_{2}, i) \eqqcolon a_{2}$, \ldots, $P(i_{d_{1}}, i) \eqqcolon a_{d_{1}}$ be the positive entries of $P(:, i)$, and let $P(j_{1}, j) \eqqcolon b_{1}$, $P(j_{2}, j) \eqqcolon b_{2}$, \ldots, $P(j_{d_{2}}, j) \eqqcolon b_{d_{2}}$ be the positive entries of $P(:, j)$, where $i\ne j$. Without loss of generality, we may assume that $d_{1} \geq d_{2}$.
Suppose that there does not exist a partition $\left\{ Y_{1}, Y_{2}, \ldots, Y_{d_{2}} \right\}$ of $[d_{1}]$ such that for all $l \in [d_{2}]$, $b_{l} = \sum_{s \in Y_{l}} a_{s}$.
Then, the length of any decomposition of $P$ is greater than $d_{1}$.
\end{proposition}

\begin{proof}
Suppose that $x_{1}, x_{2}, \ldots, x_{K} > 0$ and distinct BN matrices $A_{1}, A_{2}, \ldots, A_{K}$ define a decomposition of $P$, i.e., $\sum^K_{r=1}\ x_r=1$ and $P=\sum^{K}_{r = 1} x_{r} A_{r}$. For any $l \in [d_{1}]$, set $X^{i}_{l} \coloneqq \{ r \in [K] : A_{r}(i_{l}, i) = 1 \}$, and similarly for any $l \in [d_{2}]$, set $X^{j}_{l} \coloneqq \{ r \in [K] : A_{r}(j_{l}, j) = 1 \}$. By Lemma \ref{lemma:convenient_notations}, $\left\{X^{i}_{1}, \ldots, X^{i}_{d_{1}}\right\}$ and $\left\{X^{j}_{1}, \ldots, X^{j}_{d_{2}}\right\}$ are partitions of $[K]$. Moreover,
\[ a_{l} = P(i_{l}, i) = \sum_{r \in X^{i}_{l}} x_{r},\quad l \in [d_1] \]
and
\[ b_{l} = P(j_{l}, j) = \sum_{r \in X^{j}_{l}} x_{r},\quad l \in [d_2]. \]
By Proposition \ref{prop:max_P(:,j)_lower_bound},
\[ d_2 \le d_1 \le K. \]
Assume for a contradiction that $d_{1} = K$ so that $X^{i}_{1}, \ldots, X^{i}_{d_{1}}$ are all singleton sets. By re-labeling the indices, we may assume that $X^{i}_{1} = \{ 1 \}$, $X^{i}_{2} = \{ 2 \}$, $\ldots$, $X^{i}_{d_{1}} = \{ d_{1} \}$ whereby $a_{l}=x_{l}$ for all $l\in[d_1]=[K]$. But then we have
\[ b_{l} = \sum_{r \in X^{j}_{l}} a_{r}, \quad l\in[d_2], \]
which contradicts our assumption. Hence, $d_{1} < K$.
\end{proof}

\begin{corollary}\label{coroll:permutation}
Let $P$ be a $2^{n} \times 2^{n}$ TPM.
Suppose that two distinct columns $i$ and $j$ of $P$ have the same number of positive entries, namely, $P(i_{1}, i) \eqqcolon a_{1}$, $P(i_{2}, i) \eqqcolon a_{2}$, \ldots, $P(i_{d}, i) \eqqcolon a_{d}$ in $P(:, i)$ and $P(j_{1}, j) \eqqcolon b_{1}$, $P(j_{2}, j) \eqqcolon b_{2}$, \ldots, $P(j_{d}, j) \eqqcolon b_{d}$ in $P(:, j)$. 
Suppose further that $b_{1}, b_{2}, \ldots, b_{d}$ cannot be obtained by permuting $a_{1}, a_{2}, \ldots, a_{d}$ (i.e., there does not exist a bijective function $f : [d] \to [d]$ such that for all $k \in [d]$, $b_{k} = a_{f(k)}$), or equivalently, $P(:, j)$ is not a permutation of $P(:, i)$.
Then, the length of any decomposition of $P$ is greater than $d$.
\end{corollary}

\begin{proof}
Suppose that there were a partition $\left\{ Y_{1}, Y_{2}, \ldots, Y_{d} \right\}$ of $[d]$ such that for all $l \in [d]$, $b_{l} = \sum_{s \in Y_{l}} a_{s}$. Then $Y_{1}, Y_{2}, \ldots, Y_{d}$ must be singleton sets and let's write $Y_{1} = \{ f(1) \}$, $Y_{2} = \{ f(2) \}$, \ldots, $Y_{d} = \{ f(d) \}$. Clearly, $f: [d] \to [d]$ is a bijection and for all $l \in [d]$, $b_{l} = a_{f(l)}$, contradicting the assumption that $b_{1}, b_{2}, \ldots, b_{d}$ cannot be obtained by permuting $a_{1}, a_{2}, \ldots, a_{d}$. 
By Proposition \ref{prop:trivial_lower_bound}, the length of any decomposition of $P$ is greater than $d$.
\end{proof}

Finally, we will state and prove the strongest and most non-trivial lower bound theorem (Theorem \ref{thm:non-trivial lower bound}) of this section. To prove Theorem \ref{thm:non-trivial lower bound}, we need the following intermediate result:

\begin{lemma}\label{lemma:common_singleton}
Let $d \geq 2$ and $K$ be positive integers such that $d \leq K < \frac{4d}{3}$.
Suppose that $\left\{ X_{1}, X_{2}, \ldots, X_{d} \right\}$ and $\left\{ Y_{1}, Y_{2}, \ldots, Y_{d} \right\}$ are two partitions of $[K]$.
Then, there exist $s \in [K]$, $i, j \in [d]$ such that $X_{i} = Y_{j} = \{ s \}$.
\end{lemma}

\begin{proof}
When $d = K$, the lemma trivially holds because $X_{1}, X_{2}, \ldots, X_{d}$ and $Y_{1}, Y_{2}, \ldots, Y_{d}$ are all singleton subsets of $[K]$. Therefore, we assume that $d < K$.
Without loss of generality, we can assume that $|X_{1}| \leq |X_{2}| \leq \ldots \leq |X_{d}|$ and $|Y_{1}| \leq |Y_{2}| \leq \ldots \leq |Y_{d}|$.
Note that $|X_{1}| = 1$; otherwise, each $|X_{i}| \geq 2$ and hence 
$K = \sum^{d}_{i = 1} |X_{i}| \geq 2d > \frac{4d}{3} > K$. Similarly, $|Y_{1}| = 1$.
Moreover, there exists $i_{1} \in [d] \setminus \{ 1 \}$ such that $|X_{i_{1}}| \geq 2$; otherwise, $K = \sum^{d}_{i = 1} |X_{i}| = d < K$. Similarly, there exists $i_{2} \in [d] \setminus \{ 1 \}$ such that $|Y_{i_{2}}| \geq 2$.
Hence, there exist $j_{1}, j_{2} \in [d - 1]$ such that $|X_{j_{1}}| = 1$, $|X_{j_{1}+1}| \geq 2$, $|Y_{j_{2}}| = 1$ and $|Y_{j_{2}+1}| \geq 2$.
Then, $\frac{4d}{3} > K = \sum^{j_{1}}_{i = 1} |X_{i}| + \sum^{d}_{i = j_{1} + 1} |X_{i}| \geq j_{1} + 2(d - j_{1}) = 2d - j_{1}$. So $j_{1} > \frac{2d}{3}$. Similarly, $j_{2} > \frac{2d}{3}$.
Assume for a contradiction that $\bigcup^{j_{1}}_{j = 1} X_{j}$ and $\bigcup^{j_{2}}_{j = 1} Y_{j}$ are disjoint. Then, 
$
\left| \left( \bigcup^{j_{1}}_{j = 1} X_{j} \right) \cup \left( \bigcup^{j_{2}}_{j = 1} Y_{j} \right) \right| 
= \sum^{j_{1}}_{j = 1} |X_{j}| + \sum^{j_{2}}_{j = 1} |Y_{j}| 
= j_{1} + j_{2} > \frac{4d}{3} > K
$.
But because each $X_{j}, Y_{j} \subseteq [K]$, 
$\left( \bigcup^{j_{1}}_{j = 1} X_{j} \right) \cup \left( \bigcup^{j_{2}}_{j = 1} Y_{j} \right) \subseteq [K]$ and hence 
$\left| \left( \bigcup^{j_{1}}_{j = 1} X_{j} \right) \cup \left( \bigcup^{j_{2}}_{j = 1} Y_{j} \right) \right| \leq K$.
A contradiction is reached.
Hence, 
$\bigcup^{j_{1}}_{j = 1} X_{j}$ and $\bigcup^{j_{2}}_{j = 1} Y_{j}$ are not disjoint.
Therefore, there exist $s \in [K]$, $i \in [j_{1}]$ and $j \in [j_{2}]$ such that $X_{i} = Y_{j} = \{ s \}$.
\end{proof}

We are now ready to prove Theorem \ref{thm:non-trivial lower bound}.

\begin{theorem}
\label{thm:non-trivial lower bound}
Let $P$ be a $2^{n} \times 2^{n}$ TPM.
Suppose that there exist integers $1 \leq i < j \leq 2^{n}$ satisfying that $P(:, i)$ and $P(:, j)$ each have exactly $d$ positive entries. 
Let $C_{i}$ and $C_{j}$ be the sets of positive entries of $P(:, i)$ and $P(:, j)$ respectively. 
Suppose further that $C_{i} \cap C_{j} = \varnothing$.
Then, the length of any decomposition of $P$ is at least $\frac{4d}{3}$.
\end{theorem}

\begin{proof}
Note that $d \neq 1$; otherwise, $C_{i} = C_{j} = \{ 1 \}$. Hence, $d \geq 2$.
Let $\sum^{K}_{r = 1} x_{r} A_{r}$ be any decomposition of $P$.
By Proposition \ref{prop:max_P(:,j)_lower_bound}, $K \geq d$.
Assume for a contradiction that $K < \frac{4d}{3}$.
Consider the $d$ positive entries $P(i_{1}, i)$, $P(i_{2}, i)$, \ldots, $P(i_{d}, i)$ of $P(:, i)$, and the $d$ positive entries $P(j_{1}, j)$, $P(j_{2}, j)$, \ldots, $P(j_{d}, j)$ of $P(:, j)$. For any $l \in [d]$, set
$X^{i}_{l} \coloneqq \{ r \in [K] : A_{r}(i_{l}, i) = 1 \}$
and
$X^{j}_{l} \coloneqq \{ r \in [K] : A_{r}(j_{l}, j) = 1 \}$.
By Lemma \ref{lemma:convenient_notations}, $\left\{ X^{i}_{1}, \ldots, X^{i}_{d} \right\}$ is a partition of $[K]$, and so is $\left\{ X^{j}_{1}, \ldots, X^{j}_{d} \right\}$.
By Lemma \ref{lemma:common_singleton}, there exist $s \in [K]$, $l_{1}, l_{2} \in [d]$ such that $X^{i}_{l_{1}} = X^{j}_{l_{2}} = \{ s \}$. Hence, 
$P(i_{l_{1}}, i) = \sum_{r \in X^{i}_{l_{1}}} x_{r} = x_{s}$ and 
$P(j_{l_{2}}, j) = \sum_{r \in X^{j}_{l_{2}}} x_{r} = x_{s}$.
This implies that $x_{s} \in C_{i} \cap C_{j}$.
But $C_{i}$ and $C_{j}$ are disjoint. A contradiction is reached. Therefore, $K \geq \frac{4d}{3}$.
\end{proof}

\section{Numerical Experiments}\label{section:numerical_experiments}

In this section, we present the numerical performance of GER for solving the sparse PBN construction problem. To illustrate the effectiveness of our method, we compare GER with SER 1, SER 2 and two versions of the MOMP algorithm.
The two versions are called MOMP-interior point convex (MOMP-IPC) and MOMP-active set (MOMP-AS). In the MOMP-IPC algorithm, the quadratic programming step in MOMP (step 2.3) is performed using the interior point convex algorithm (IPC) in MATLAB \cite{Altman_interior_point, Nocedal_numerical_optimization, Vanderbei_interior_point}, whereas in the MOMP-AS algorithm, step 2.3 is performed using the active set algorithm (AS) in MATLAB \cite{Nocedal_numerical_optimization, anti_cycling}.
The numerical experiments for GER, SER 1 and SER 2 were carried out using JupyterLab, whereas the numerical experiments for MOMP-IPC and MOMP-AS were carried out using MATLAB R2022b. All numerical experiments were carried out on a personal laptop with 16 GB of RAM and an AMD Ryzen 7 5800U processor with Radeon Graphics at 1.90 GHz.
The codes for our numerical experiments are available at https://github.com/christopherfok2015/PBN-construction-project-two-GER-papers.git.

In Section \ref{section:18_TPMs_tested}, we describe the 18 TPMs on which we executed GER and the other four sparse PBN construction algorithms. In Section \ref{section:parameters_setting}, we detail how the parameters of GER, MOMP-IPC and MOMP-AS were set in our numerical experiments. In Section \ref{section:expt_results}, we present and discuss our experimental results.

\subsection{TPMs Tested in our Numerical Experiments}\label{section:18_TPMs_tested}

We define five TPMs $P_{1}, P_{2}, P_{3}, P_{4}, P_{5}$ as follows:

\begin{equation*}
P_{1} \coloneqq 
\begin{pmatrix}
0.1 	& 0.5	& 0.6	& 0 \\ 
0.4 	& 0	& 0.2	& 0 \\ 
0.5 	& 0.2	& 0	& 1 \\ 
0	& 0.3	& 0.2	& 0 
\end{pmatrix},
\end{equation*}

\begin{equation*}
P_{2} \coloneqq \frac{1}{110}
\begin{pmatrix}
12	& 30	& 22	& 10	& 10	& 15	& 54	& 34	\\
10	& 24	& 19	& 54	& 30	&  0	&  0	&  0	\\
54	& 15	&  0	& 12	& 12	&  0	&  0	& 30	\\
0	&  0	& 24	& 15	& 24	& 19	& 10	&  0	\\
0	&  0	&  0	&  0	& 34	& 10	& 12	&  0	\\
19	&  0	& 15	&  0	&  0	& 12	&  0	& 22	\\
15	& 19	&  0	& 19	&  0	& 54	&  0	& 24	\\
0	& 22	& 30	&  0	&  0	&  0	& 34	&  0
\end{pmatrix},
\end{equation*}

\begin{equation*}
P_{3} \coloneqq \frac{1}{110}
\begin{pmatrix}
0	&  0	&  0	& 49	&  0	& 43	&  0	& 49	\\
0	& 30	& 12	&  0	& 30	&  0	& 25	&  0	\\
25	&  0	&  0	& 15	&  0	& 22	& 12	& 15	\\
0	&  0	& 10	& 24	& 19	&  0	& 30	&  0	\\
30	& 43	& 15	&  0	& 24	& 15	&  0	&  0	\\
43	&  0	& 49	&  0	&  0	&  0	& 19	& 24	\\
0	& 22	&  0	& 22	& 12	& 30	&  0	& 12	\\
12	& 15	& 24	&  0	& 25	&  0	& 24	& 10
\end{pmatrix},
\end{equation*}

\begin{equation*}
P_{4} \coloneqq \frac{1}{265}
\begin{psmallmatrix}
26	&   0	&   0	&   0	&  59	&   0	&  49	&   0	& 29	&   0	&   0	&   0	&   0	&   9	&   0	&  46	\\
0	&   0	&  39	&  17	&   0	&   0	&   0	&  49	& 49	&   0	&   0	&   0	&   0	&  98	&  54	&  17	\\
0	&   0	&  26	&  49	&   0	&   0	&   0	&   9	& 0	&   0	&   0	&  59	&   9	&   0	&   0	&   0	\\
0	&   0	&   0	&   0	&   0	&   9	&  54	&   0	& 0	&   0	&   0	&   0	&  39	&   0	&  49	&  26	\\
49	&   0	&   9	&   0	&   0	&  29	&   0	&   0	& 0	&   0	& 108	&  63	&   0	&  17	&  59	&   0	\\
0	&   0	&   0	&  26	&   0	&   0	&   0	&  17	& 39	&  29	&   0	&   0	&   0	&   0	&   0	&  88	\\
17	&  63	&  88	&   0	&   0	&   0	&   0	&  98	& 0	&   9	&  37	&   0	&  88	&   0	&   0	&  49	\\
0	&   0	&   0	&   0	&  49	&   0	&   0	&   0	& 17	&  37	&   9	&  29	&  63	&  49	&  39	&   0	\\
98	&   0	&   0	&   0	&  46	& 108	&  26	&   0	& 0	&   0	&   0	&   9	&   0	&  63	&   9	&   0	\\
0	&  29	&   0	&   0	&  17	&   0	&  29	&   0	& 46	&   0	&  26	&   0	&   0	&   0	&  26	&   0	\\
29	& 108	&   0	&  88	&   0	&  39	&   0	&   0	& 0	&   0	&  39	&   0	&   0	&  29	&   0	&   0	\\
0	&  39	&   0	&  39	&   0	&   0	&   0	&   0	& 26	&  26	&  29	&  39	&   0	&   0	&   0	&   0	\\
0	&   0	&   0	&   0	&  29	&   0	&  98	&  29	& 0	&  39	&   0	&   0	&   0	&   0	&   0	&   0	\\
46	&   0	&  49	&  37	&  39	&  26	&   9	&   0	& 0	& 108	&  17	&   0	&  49	&   0	&  29	&   0	\\
0	&   9	&  54	&   9	&   0	&  54	&   0	&  37	& 0	&  17	&   0	&  49	&  17	&   0	&   0	&  39	\\
0	&  17	&   0	&   0	&  26	&   0	&   0	&  26	& 59	&   0	&   0	&  17	&   0	&   0	&   0	&   0
\end{psmallmatrix},
\end{equation*}

\begin{equation*}
P_{5} \coloneqq \frac{1}{10}
\begin{psmallmatrix}
0	& 6	& 8	& 1	& 7	& 5	& 8	& 5	& 6	& 5	& 3	& 4	& 10	& 2	& 2	& 0	\\
8	& 1	& 1	& 0	& 0	& 0	& 0	& 0	& 1	& 0	& 0	& 0	&  0	& 0	& 0	& 3	\\
2	& 1	& 0	& 0	& 0	& 0	& 0	& 0	& 1	& 0	& 0	& 0	&  0	& 0	& 0	& 0	\\
0	& 1	& 0	& 0	& 0	& 0	& 0	& 0	& 0	& 0	& 0	& 0	&  0	& 0	& 0	& 0	\\
0	& 1	& 1	& 8	& 2	& 4	& 2	& 5	& 0	& 3	& 2	& 3	&  0	& 3	& 4	& 1	\\
0	& 0	& 0	& 1	& 0	& 1	& 0	& 0	& 0	& 0	& 0	& 0	&  0	& 0	& 0	& 0	\\
0	& 0	& 0	& 0	& 0	& 0	& 0	& 0	& 0	& 0	& 0	& 0	&  0	& 0	& 0	& 0	\\
0	& 0	& 0	& 0	& 0	& 0	& 0	& 0	& 0	& 0	& 0	& 0	&  0	& 0	& 0	& 0	\\
0	& 0	& 0	& 0	& 1	& 0	& 0	& 0	& 1	& 0	& 0	& 3	&  0	& 0	& 2	& 1	\\
0	& 0	& 0	& 0	& 0	& 0	& 0	& 0	& 1	& 2	& 2	& 0	&  0	& 0	& 0	& 0	\\
0	& 0	& 0	& 0	& 0	& 0	& 0	& 0	& 0	& 0	& 0	& 0	&  0	& 0	& 0	& 0	\\
0	& 0	& 0	& 0	& 0	& 0	& 0	& 0	& 0	& 0	& 0	& 0	&  0	& 0	& 0	& 0	\\
0	& 0	& 0	& 0	& 0	& 0	& 0	& 0	& 0	& 0	& 3	& 0	&  0	& 4	& 2	& 5	\\
0	& 0	& 0	& 0	& 0	& 0	& 0	& 0	& 0	& 0	& 0	& 0	&  0	& 1	& 0	& 0	\\
0	& 0	& 0	& 0	& 0	& 0	& 0	& 0	& 0	& 0	& 0	& 0	&  0	& 0	& 0	& 0	\\
0	& 0	& 0	& 0	& 0	& 0	& 0	& 0	& 0	& 0	& 0	& 0	&  0	& 0	& 0	& 0	
\end{psmallmatrix}.
\end{equation*}

We remark that $P_5$ is obtained from a preliminary study of a nursing home in Hong Kong \cite{Leung1,Leung2}. 
In the nursing home, members of staff are concerned about their overtime work and consecutive shifts, both of which are important factors for the service quality and their turnover rate \cite{Leung2}.

Besides the synthetic TPMs $P_{1}, P_{2}, P_{3}, P_{4}, P_{5}$, we also consider TPMs originating in existing literature. We will introduce these matrices $P^{\textrm{A}}_{1}$, $P^{\textrm{A}}_{2}$, $P^{\textrm{A}}_{3}$, $P^{\textrm{B}}_{1}$, $P^{\textrm{B}}_{3}$, $P^{\textrm{B}}_{4}(d)$, $P^{\textrm{B}}_{6}(d)$ below.

$P^{\textrm{A}}_{1}$ is a $4 \times 4$ TPM proposed in \cite{chen_jiang_ching_PBN_construction} and
\begin{equation*}
P^{\textrm{A}}_{1} =
\begin{pmatrix}
0.1	& 0.3	& 0.2	& 0.1 \\
0.2	& 0.3	& 0.2	& 0	\\
0	& 0	& 0.6	& 0.4	\\
0.7	& 0.4	& 0	& 0.5
\end{pmatrix}.
\end{equation*}
$P^{\textrm{A}}_{2}$ was proposed in \cite{sparse_solution_of_nonnegative} and 
\[
P^{\textrm{A}}_{2} =
\begin{pmatrix}
P^{\textrm{A}}_{1} & \mathbf{O}_{4 \times 4} \\
\mathbf{O}_{4 \times 4} & P^{\textrm{A}}_{1}
\end{pmatrix}.
\]
$P^{\textrm{A}}_{3}$ is an $8 \times 8$ TPM originating in \cite{credit_defaults}, and was computed from credit default data. It is given by
\begin{equation*}
P^{\textrm{A}}_{3} = 
\begin{pmatrix}
0.57	& 0		& 0.1	& 0		& 0		& 0.04	& 0		& 0		\\
0.14	& 0.31	& 0	& 0.5		& 0.12	& 0.13	& 0.33	& 0.06	\\
0	& 0.08	& 0.4	& 0.25	& 0.25	& 0		& 0.67	& 0		\\
0	& 0.15	& 0	& 0		& 0		& 0.08	& 0		& 0		\\
0	& 0.15	& 0.3	& 0		& 0		& 0.13	& 0		& 0		\\
0.29	& 0.31	& 0.2	& 0		& 0.25	& 0.29	& 0		& 0.39	\\
0	& 0		& 0	& 0		& 0.38	& 0		& 0		& 0		\\
0	& 0		& 0	& 0.25	& 0		& 0.33	& 0		& 0.55
\end{pmatrix}.
\end{equation*}
$P^{\textrm{B}}_{1}$ is a $4 \times 4$ TPM proposed in \cite{max_entropy_rate} and
\begin{equation*}
P^{\textrm{B}}_{1} = 
\begin{pmatrix}
0.1 	& 0.3	& 0.5	& 0.6	\\
0	& 0.7	& 0	& 0	\\ 
0	& 0	& 0.5	& 0	\\ 
0.9	& 0	& 0	& 0.4
\end{pmatrix}.
\end{equation*}
$P^{\textrm{B}}_{3}$ is an $8 \times 8$ TPM given in \cite{shmulevich_PBN} and
\begin{equation*}
P^{\textrm{B}}_{3} = 
\begin{pmatrix}
1	& 0	& 0	& 0.2	& 0	& 0	& 0	& 0	\\
0	& 0	& 0	& 0.2	& 0	& 0	& 0	& 0 	\\
0	& 0	& 0	& 0	& 1	& 0	& 0	& 0	\\
0	& 0	& 0	& 0	& 0	& 0	& 0	& 0	\\
0	& 0	& 0	& 0.3	& 0	& 0	& 0.5	& 0	\\
0	& 0	& 0	& 0.3	& 0	& 0	& 0.5	& 0	\\
0	& 1	& 1	& 0	& 0	& 0.5	& 0	& 0	\\
0	& 0	& 0	& 0	& 0	& 0.5	& 0	& 1
\end{pmatrix}.
\end{equation*}
$P^{\textrm{B}}_{4}(d)$ is a TPM originating in \cite{sparse_solution_of_nonnegative}. It is formed by adding some perturbation to $P^{\textrm{B}}_{3}$, where $d$ represents the magnitude of the perturbation. $P^{\textrm{B}}_{4}(d)$ is given by
\begin{equation*}
P^{\textrm{B}}_{4}(d) = 
\begin{pmatrix}
1-d	& 0	& 0	& 0.2	& 0	& 0	& 0	& 0	\\
0	& 0	& 0	& 0.2	& 0	& 0	& 0	& 0 	\\
0	& 0	& 0	& 0	& 1-d	& 0	& 0	& 0	\\
d	& d	& d	& 0	& d	& 0	& 0	& d	\\
0	& 0	& 0	& 0.3	& 0	& 0	& 0.5	& 0	\\
0	& 0	& 0	& 0.3	& 0	& 0	& 0.5	& 0	\\
0	& 1-d	& 1-d	& 0	& 0	& 0.5	& 0	& 0	\\
0	& 0	& 0	& 0	& 0	& 0.5	& 0	& 1-d
\end{pmatrix}.
\end{equation*}
We executed the five sparse PBN construction algorithms on $P^{\textrm{B}}_{4}(0.01), P^{\textrm{B}}_{4}(0.02), P^{\textrm{B}}_{4}(0.03), P^{\textrm{B}}_{4}(0.04)$. Finally, $P^{\textrm{B}}_{6}(d)$ is also a TPM originating in \cite{sparse_solution_of_nonnegative} and
\[
P^{\textrm{B}}_{6}(d) = 
\begin{pmatrix}
P^{\textrm{B}}_{4}(d) & \mathbf{O}_{8 \times 8} \\
\mathbf{O}_{8 \times 8} & P^{\textrm{B}}_{4}(d)
\end{pmatrix}.
\]
We executed the five sparse PBN construction algorithms on $P^{\textrm{B}}_{6}(0.01), P^{\textrm{B}}_{6}(0.02), P^{\textrm{B}}_{6}(0.03), P^{\textrm{B}}_{6}(0.04)$.

\subsection{Parameters in our Numerical Experiments}\label{section:parameters_setting}

Concerning GER, we set the score parameter $z$ to be $10$ for all matrices except $P^{\textrm{A}}_{3}$, for which the score parameter was set to $2$ instead. It was because GER produced a sparser decomposition for $P^{\textrm{A}}_{3}$ when $z=2$.

Concerning MOMP-IPC, for all the 18 PBN matrices, the initial guess $\vec{x}^{0}$ was set to be $\frac{1}{N} \vec{1}_{N}$, where $N \coloneqq |B_{n}(P)| = \prod^{2^{n}}_{j = 1} |D_{j}(P)|$. Also, $\varepsilon_{\textrm{tolerance}}$ was set to be $10^{-7}$ for all matrices except $P_{5}$ because the execution of MOMP-IPC did not terminate after more than 11 hours when the algorithm was applied to $P_{5}$ with $\varepsilon_{\textrm{tolerance}}=10^{-7}$. We eventually set  $\varepsilon_{\textrm{tolerance}}$ to be $10^{-4}$ and the execution terminated after almost four hours.

Concerning MOMP-AS, for all the 18 PBN matrices, 
$\vec{x}^{0}$ was set to be $\frac{1}{N} \vec{1}_{N}$,
$\varepsilon_{\textrm{tolerance}}$ was set to be $10^{-7}$, and
the starting point of the AS algorithm was set to be $\frac{1}{|S^{k + 1}|} \sum_{i \in S^{k + 1}} \vec{e}_{i}$, where $\vec{e}_{i} \in \mathbb{R}^{N}$ and $S^{k+1}$ is a variable in the MOMP algorithm which is always a subset of $[N]$.

\subsection{Experimental Results}\label{section:expt_results}

Table \ref{table:numerical_results} shows the lengths of the decompositions obtained in our numerical experiments.

\begin{table}[h]
\centering
\begin{tabular}{|l|c|c|c|c|c|}
\hline
\multicolumn{1}{|c|}{TPM} & GER & SER 1 & SER 2 & MOMP-IPC & MOMP-AS \\
\hline
$P_{1}$ & 4 & 7 & 4 & 5 & 5 \\
\hline
$P_{2}$ & 6 & 24 & 10 & 28 & 17 \\
\hline
$P_{3}$ & 6 & 22 & 9 & 24 & 16 \\
\hline
$P_{4}$ & 8 & 46 & 13 & ITR & ITR \\
\hline
$P_{5}$ & 7 & 10 & 7 & 31 & 29 \\
\hline
$P^{\textrm{A}}_{1}$ & 5 & 7 & 5 & 5 & 5 \\
\hline
$P^{\textrm{A}}_{2}$ & 5 & 9 & 5 & 5 & 5 \\
\hline
$P^{\textrm{A}}_{3}$ & 11 & 20 & 11 & 26 & 22 \\
\hline
$P^{\textrm{B}}_{1}$ & 4 & 5 & 4 & 6 & 5 \\
\hline
$P^{\textrm{B}}_{3}$ & 4 & 4 & 4 & 4 & 4 \\
\hline
$P^{\textrm{B}}_{4}(0.01)$ & 5 & 11 & 5 & 5 & 6 \\
\hline
$P^{\textrm{B}}_{4}(0.02)$ & 5 & 11 & 5 & 5 & 5 \\
\hline
$P^{\textrm{B}}_{4}(0.03)$ & 5 & 11 & 5 & 5 & 7 \\
\hline
$P^{\textrm{B}}_{4}(0.04)$ & 5 & 11 & 5 & 5 & 6 \\
\hline
$P^{\textrm{B}}_{6}(0.01)$ & 5 & 16 & 5 & 7 & 6 \\
\hline
$P^{\textrm{B}}_{6}(0.02)$ & 5 & 14 & 5 & 5 & 6 \\
\hline
$P^{\textrm{B}}_{6}(0.03)$ & 5 & 16 & 5 & 5 & 6 \\
\hline
$P^{\textrm{B}}_{6}(0.04)$ & 5 & 16 & 5 & 5 & 6 \\
\hline
\end{tabular}
\caption{Results of executing GER, SER 1, SER 2, MOMP-IPC and MOMP-AS with different TPMs as inputs. Each number in the table is the length $K$ of the decomposition found by a particular sparse PBN construction algorithm.}
\label{table:numerical_results}
\end{table}

\FloatBarrier

We remark that in Table \ref{table:numerical_results}, ``ITR'' stands for ``infeasible to run''. It is not feasible to execute MOMP-IPC and MOMP-AS on $P_{4}$. The reason behind is explained in Appendix \ref{appendix:MOMP_infeasibility}.

As shown in Table \ref{table:numerical_results}, for each of the 18 TPMs, GER outputs a decomposition with the smallest length. Therefore, GER gives the best performance among the five sparse PBN construction algorithms.

In fact, we can prove that for each TPM $P$ in Table \ref{table:numerical_results} except for $P^{\textrm{A}}_{3}$, GER outputs a decomposition with the smallest possible length. In other words, we can prove that for each $P$ (where $P$ can be any one of the 17 TPMs excluding $P^{\textrm{A}}_{3}$),
there does not exist a decomposition of $P$ whose length is strictly smaller than that of the decomposition of $P$ output by GER.
The proofs are given in Appendix \ref{appendix:GER_gives_17_sparsest_decomposition}.

Moreover, we remark that the true PBN giving rise to $P^{\textrm{B}}_{3}$ (described in \cite{shmulevich_PBN}) is the same as the PBN constructed from the decomposition of $P^{\textrm{B}}_{3}$ output by GER. In other words, GER is able to recover the true PBN from the TPM $P^{\textrm{B}}_{3}$.

Next, we compare the execution time (in seconds) of the five sparse PBN construction algorithms on the 18 TPMs. This information is shown in Table \ref{table:execution_time}.

\begin{table}[h]
\centering
\begin{tabular}{|l|c|c|c|c|c|}
\hline
\multicolumn{1}{|c|}{TPM} & GER & SER 1 & SER 2 & MOMP-IPC & MOMP-AS \\
\hline
$P_{1}$ & $0.002063$ & $< 0.000001$ & $0.001002$ & $0.455606$ & $0.090787$ \\
\hline
$P_{2}$ & $0.003855$ & $0.001559$ & $0.001000$ & $85.226294$ & $53.393364$ \\
\hline
$P_{3}$ & $0.002614$ & $0.001797$ & $< 0.000001$ & $47.199041$ & $31.745643$ \\
\hline
$P_{4}$ & $0.006098$ & $0.008434$ & $0.001504$ & $> 40380$ & $> 40380$ \\
\hline
$P_{5}$ & $0.004407$ & $0.001004$ & $< 0.000001$ & $13373.813786$ & $12874.390265$ \\
\hline
$P^{\textrm{A}}_{1}$ & $0.002083$ & $0.001366$ & $0.001002$ & $0.140595$ & $0.045735$ \\
\hline
$P^{\textrm{A}}_{2}$ & $0.001926$ & $0.001332$ & $0.000515$ & $0.583858$ & $0.740798$ \\
\hline
$P^{\textrm{A}}_{3}$ & $0.005193$ & $0.001721$ & $0.001001$ & $6.079369$ & $5.126854$ \\
\hline
$P^{\textrm{B}}_{1}$ & $0.001882$ & $0.001051$ & $< 0.000001$ & $0.031529$ & $0.011561$ \\
\hline
$P^{\textrm{B}}_{3}$ & $0.002502$ & $0.001011$ & $< 0.000001$ & $0.017397$ & $0.011455$ \\
\hline
$P^{\textrm{B}}_{4}(0.01)$ & $0.001507$ & $0.001000$ & $0.001357$ & $0.036819$ & $0.054702$ \\
\hline
$P^{\textrm{B}}_{4}(0.02)$ & $0.001246$ & $0.001009$ & $< 0.000001$ & $0.066220$ & $0.066379$ \\
\hline
$P^{\textrm{B}}_{4}(0.03)$ & $0.002253$ & $0.000817$ & $< 0.000001$ & $0.039834$ & $0.049690$ \\
\hline
$P^{\textrm{B}}_{4}(0.04)$ & $0.001841$ & $< 0.000001$ & $< 0.000001$ & $0.045648$ & $0.041034$ \\
\hline
$P^{\textrm{B}}_{6}(0.01)$ & $0.004512$ & $0.002017$ & $< 0.000001$ & $16.403175$ & $13.907554$ \\
\hline
$P^{\textrm{B}}_{6}(0.02)$ & $0.004557$ & $0.001341$ & $< 0.000001$ & $11.584851$ & $13.972031$ \\
\hline
$P^{\textrm{B}}_{6}(0.03)$ & $0.004834$ & $0.002561$ & $0.001006$ & $11.476147$ & $13.904559$ \\
\hline
$P^{\textrm{B}}_{6}(0.04)$ & $0.003473$ & $0.002868$ & $< 0.000001$ & $11.587956$ & $14.103169$ \\
\hline
\end{tabular}
\caption{The execution time (in seconds) of GER, SER 1, SER 2, MOMP-IPC and MOMP-AS on different TPMs.}
\label{table:execution_time}
\end{table}

\FloatBarrier

From Table \ref{table:execution_time}, we can see that SER 1, SER 2 and GER were extremely fast: they took milliseconds or less to execute on each of the 18 TPMs.
Overall, SER 1 was slightly faster than GER and SER 2 was slightly faster than SER 1.
On several TPMs, MOMP-IPC and MOMP-AS took a few tens of seconds to execute.
On $P_{5}$, MOMP-IPC and MOMP-AS took more than 3 hours to run.
MOMP-IPC and MOMP-AS did not terminate after being executed on $P_{4}$ for 40380 seconds (11 hours 13 minutes). The execution was therefore manually terminated. We have already remarked that it is not feasible to execute MOMP-IPC and MOMP-AS on $P_{4}$.
We conclude that GER is a fast PBN construction algorithm.

\section{Conclusion}\label{section:conclusion}

In this paper, we study the construction problem of sparse Probabilistic Boolean Networks (PBNs).
We propose a novel Greedy Entry Removal (GER) algorithm
for solving this problem.
We present new theoretical upper bounds for two existing sparse PBN construction algorithms and the GER algorithm.
Moreover, we are the first to study the lower bound problem related to the sparse PBN construction problem, and to present a series of theoretical results about this problem.
To evaluate the effectiveness and the efficiency of GER, we conducted numerical experiments based on both synthetic and practical data. Our experimental results show that GER is fast, and that GER gives the best performance among state-of-the-art sparse PBN construction algorithms and outputs sparsest decompositions on most of the transition probability matrices tested.

There are three possible extensions of this work. Firstly, new algorithms for constructing sparse PBNs can be formulated. Secondly, new theoretical upper bounds for existing sparse PBN construction algorithms other than Simple Entry Removal Algorithm 1 (SER 1), Simple Entry Removal Algorithm 2 (SER 2) and GER can be derived. Thirdly, the research community can continue to study the lower bound problem and derive new lower bounds.

\newpage

\appendix
\appendixpage

\section{Infeasibility of the Modified Orthogonal Matching Pursuit Algorithm}\label{appendix:MOMP_infeasibility}

In this section, we elaborate why it is not feasible to execute MOMP-IPC and MOMP-AS on some TPMs, in particular, $P_{4}$. Meanwhile, we introduce how to enhance the adaptiveness of this algorithm. 

Recall that for any TPM $P$, this algorithm forms a matrix $A = [\textrm{vec}(A_{1}), \textrm{vec}(A_{2}), \ldots, \\ \textrm{vec}(A_{N})]$ where $A_{1}, A_{2}, \ldots, A_{N}$ are all the distinct BN matrices in $B_{n}(P)$ and $N=|B_{n}(P)|$. 

In particular, $P=P_{4}$ is a $16\times 16$ matrix of which eight columns has six positive entries while the remaining eight has seven. It follows that $N=6^{8} \times 7^{8} \approx 2^{43}$ so that the dimension of $A$ is approximately $(2^{4}\times 2^{4})\times 2^{43} = 2^{51}$ which cannot be handled by MATLAB running on a personal computer (the upper bound is approximately $2^{48}$).  

In order to implement MOMP-IPC and MOMP-AS in our personal computer, we take measures to avoid the creation of such huge matrices in Step 2 of the algorithm. Firstly, since every $\vec{x}^{k}$ has (at most) $k$ non-zero (indeed positive) entries and $\vec{e_{j}}^\top A^\top=\left(A\vec{e_{j}}\right)^\top=\textrm{vec}(A_{j})^\top$, we need at most $k+1$ columns of $A$ in Steps 2.1 and 2.3. Secondly, concerning Step 2.1, we find a way to compute $\argmax\limits_{j \in [N]} \vec{e_{j}}^\top A^\top (\vec{b} - A\vec{x}^{k})$ with small time and space complexities. For full details of how we implemented Steps 2.1 and 2.3 of MOMP, please refer to our codes on GitHub. Now we turn to Step 2.4, where the scalar $\lambda_k:=(\vec{x}^{k+1})^{\top} A^{\top} \vec{r}^{k+1}$ and the error term $e^{k+1}$ should be computed with reasonable time and space complexities. For the detailed treatment of $\lambda_k$, please refer to our codes on GitHub. We shall introduce how to reduce the space complexity in computing $e^{k+1}$ below. 

Let $k$ be arbitrary fixed and let $|\textrm{supp}(\vec{x}^{k+1})|=L$. We may then write $A_{\textrm{supp}(\vec{x}^{k+1})}=[\vec{v}_{1}, \vec{v}_{2}, \ldots, \vec{v}_{L}]$ and $A_{[N] \setminus \textrm{supp}(\vec{x}^{k+1})}=[\vec{w}_{1}, \vec{w}_{2}, \ldots, \vec{w}_{M}]$ where $M=N-L$. Then,

\begin{eqnarray}
& &
\left\| 
A^{\top}_{\textrm{supp}(\vec{x}^{k+1})}\vec{r}^{k+1} - 
\lambda^{k+1} \vec{1}_{\left| \textrm{supp}(\vec{x}^{k+1}) \right|}
\right\|_{2} \nonumber \\
& = &
\left\| 
\begin{pmatrix}
\vec{v}^{\top}_{1} \vec{r}^{k+1} \\
\vec{v}^{\top}_{2} \vec{r}^{k+1} \\
\vdots \\
\vec{v}^{\top}_{L} \vec{r}^{k+1}
\end{pmatrix}
- \lambda^{k+1} \vec{1}_{L}
\right\|_{2} \nonumber \\
& = &
\sqrt{\sum^{L}_{l = 1} (\vec{v}^{\top}_{l} \vec{r}^{k+1} - \lambda^{k+1})^{2}}
\end{eqnarray}
and
\begin{eqnarray}
& &
\left\| \max
\left( 
A^{\top}_{[N] \setminus \textrm{supp}(\vec{x}^{k+1})}\vec{r}^{k+1} - 
\lambda^{k+1} \vec{1}_{\left| [N] \setminus \textrm{supp}(\vec{x}^{k+1}) \right|}, 
\vec{0}_{\left| [N] \setminus \textrm{supp}(\vec{x}^{k+1}) \right|} 
\right) \right\|_{2} \nonumber \\
& = &
\left\| \max \left(
\begin{pmatrix}
\vec{w}^{\top}_{1} \vec{r}^{k+1} - \lambda^{k+1} \\
\vec{w}^{\top}_{2} \vec{r}^{k+1} - \lambda^{k+1} \\
\vdots \\
\vec{w}^{\top}_{M} \vec{r}^{k+1} - \lambda^{k+1}
\end{pmatrix}
, \vec{0}_{M} \right) \right\|_{2} \nonumber \\
& = &
\left\|
\begin{pmatrix}
\max (\vec{w}^{\top}_{1} \vec{r}^{k+1} - \lambda^{k+1}, 0) \\
\max (\vec{w}^{\top}_{2} \vec{r}^{k+1} - \lambda^{k+1}, 0) \\
\vdots \\
\max (\vec{w}^{\top}_{M} \vec{r}^{k+1} - \lambda^{k+1}, 0)
\end{pmatrix}
\right\|_{2} \nonumber \\
& = &
\sqrt{
\sum^{M}_{l = 1} \max (\vec{w}^{\top}_{l} \vec{r}^{k+1} - \lambda^{k+1}, 0)^{2}
}.
\end{eqnarray}
Therefore,
\begin{equation}
e^{k+1} = 
\sqrt{\sum^{L}_{l = 1} (\vec{v}^{\top}_{l} \vec{r}^{k+1} - \lambda^{k+1})^{2}}
+
\sqrt{
\sum^{M}_{l = 1} \max (\vec{w}^{\top}_{l} \vec{r}^{k+1} - \lambda^{k+1}, 0)^{2}
}.
\end{equation}

This formula suggests us to compute $e^{k+1}$ iteratively so that a single column of $A$ is required in each iteration. 

\begin{algorithm}
\caption{A Space-Efficient Algorithm for Computing $e^{k+1}$}
\label{alg:computing_e_(k+1)}
\begin{algorithmic}[1]

\State Initialize $\mathtt{sum\_1} \gets 0$ and $\mathtt{sum\_2} \gets 0$.

\ForAll{$i \in [N]$}
	\State Compute the BN matrix $A_{i}$ and set $\vec{u} \gets \textrm{vec}(A_{i})$.

	\If{$\vec{u} \in \{ \vec{v}_{1}, \vec{v}_{2}, \ldots, \vec{v}_{L} \}$}
		\State $\mathtt{sum\_1} \gets \mathtt{sum\_1} + (\vec{u}^{\top} \vec{r}^{k+1} - \lambda^{k+1})^{2}$

	\Else
	\Comment{i.e., $\vec{u} \in \{ \vec{w}_{1}, \vec{w}_{2}, \ldots, \vec{w}_{M} \}$}
		\State $\mathtt{sum\_2} \gets \mathtt{sum\_2} + \max (\vec{u}^{\top} \vec{r}^{k+1} - \lambda^{k+1}, 0)^{2}$
	\EndIf
\EndFor

\State Set $e^{k+1} \gets \sqrt{\mathtt{sum\_1}} + \sqrt{\mathtt{sum\_2}}$.

\Output $e^{k+1}$
\end{algorithmic}
\end{algorithm}

As shown in Algorithm \ref{alg:computing_e_(k+1)}, the columns of $A_{\textrm{supp}(\vec{x}^{k+1})}$ and 
$A_{[N] \setminus \textrm{supp}(\vec{x}^{k+1})}$ are formed, processed and destroyed iteratively one at a time. For full details of Algorithm \ref{alg:computing_e_(k+1)}, please refer to the program file \texttt{compute\_current\_error\_} \texttt{according\_to\_page\_9.m} and the program files of the user-defined functions used therein. These files can be accessed on the GitHub page for our numerical experiments.

However, we remark that MOMP is still an impracticable algorithm to handle the matrix $P_{4}$ even after this enhancement due to the huge size of $B_4(P_4)$. To be precise, if our MOMP-IPC and MOMP-AS programs are executed on $P_{4}$, there would be $N \approx 2^{43}$ iterations in the for loop of Algorithm \ref{alg:computing_e_(k+1)}. The resulting program execution time is too long and hence it is not feasible to execute our MOMP-IPC and MOMP-AS programs on $P_{4}$ using a personal laptop.

To conclude, we cannot find a way to implement MOMP-IPC and MOMP-AS which avoids out-of-memory errors and ensures sufficiently short execution time on $P_{4}$.

\subsection{$|B_{n}(P)|$ and the Time Efficiency of MOMP}\label{subappendix:time_efficiency_of_MOMP}

As seen from the previous discussion, the size of $B_{n}(P)$ is a key factor for both the time complexity and the space complexity of MOMP methods. In this subsection, we further illustrate the role of $|B_{n}(P)|$ in determining the time efficiency of using MOMP methods to find a decomposition of a $2^{n} \times 2^{n}$ TPM $P$.

Table \ref{table:BnP} lists the sizes of $B_n(P)$ and the execution times of both MOMP algorithms for six TPMs $P$. These six matrices are exactly those $16 \times 16$ matrices we have investigated in our numerical experiments (Section \ref{section:numerical_experiments}) and they are actually the top six matrices with the largest $B_n(P)$ sizes in our experiments.

\begin{table}[h]
\centering
\begin{tabular}{|l|l|l|l|}
\hline
TPM & $|B_{n}(P)|$ & \makecell{Execution time \\ of MOMP-IPC \\ (in seconds)} & \makecell{Execution time \\ of MOMP-AS \\ (in seconds)} \\
\hline
$P_{4}$	& $9682651996416$	& $> 40380$	& $> 40380$ \\
\hline
$P_{5}$	& $37324800$		& $13373.813786$	& $12874.390265$ \\
\hline
$P^{\textrm{B}}_{6}(0.01)$	& $262144$	& $16.403175$	& $13.907554$ \\
\hline
$P^{\textrm{B}}_{6}(0.02)$	& $262144$	& $11.584851$	& $13.972031$ \\
\hline
$P^{\textrm{B}}_{6}(0.03)$	& $262144$	& $11.476147$	& $13.904559$ \\
\hline
$P^{\textrm{B}}_{6}(0.04)$	& $262144$	& $11.587956$	& $14.103169$ \\
\hline
\end{tabular}
\caption{The execution time of MOMP-IPC and MOMP-AS on the top six TPMs with the greatest values of $|B_{n}(P)|$ among the 18 TPMs tested in our numerical experiments. Note that the six TPMs shown are exactly those with size $16 \times 16$ among the 18 TPMs tested.}
\label{table:BnP}
\end{table}

From Table \ref{table:BnP}, we can infer that the larger the value of $|B_{n}(P)|$, the longer it takes MOMP-IPC and MOMP-AS to execute on the TPM $P$. Because the sizes of the six TPMs shown all equal $16 \times 16$, we can also infer that the size of a TPM $P$ does not seem to play a significant role in determining the execution time of MOMP-IPC and MOMP-AS on $P$.

When the value of $|B_{n}(P)|$ is very large, MOMP will take an infeasibly long time to execute on $P$ and produce a decomposition of $P$. Therefore, on such occasions, MOMP is possibly not a good choice for constructing sparse PBNs---especially when there are constraints on the execution time.

\newpage

\section{Mathematical Proofs About the Sparsity of the Decompositions Output by GER}
\label{appendix:GER_gives_17_sparsest_decomposition}

In Section \ref{section:expt_results}, we claim that for each matrix in Table \ref{table:numerical_results} except for $P^{\textrm{A}}_{3}$, GER outputs a decomposition with the smallest possible length. In this section, we are going to prove this claim.

\begin{theorem}\label{prop:P1}
The length of any decomposition of $P_{1}$ is at least $4$.
\end{theorem}

\begin{proof}
Consider the 1\textsuperscript{st} and the 3\textsuperscript{rd} columns of $P_{1}$, each of which has exactly 3 positive entries.
The set of positive entries of the 1\textsuperscript{st} column is $\{ 0.1, 0.4, 0.5 \} \eqqcolon C_{1}$, whereas the set of positive entries of the 3\textsuperscript{rd} column is $\{ 0.2, 0.6 \} \eqqcolon C_{3}$. Clearly, $C_{1} \cap C_{3} = \varnothing$. By Theorem \ref{thm:non-trivial lower bound}, any decomposition $\sum^{K}_{r = 1} x_{r} A_{r}$ of $P_{1}$ satisfies $K \geq \frac{(4)(3)}{3} = 4$.
\end{proof}

The decomposition of $P_{1}$ found using our Python implementation of GER (with the score parameter set to $10$) is
\begin{equation*}
0.2 \langle 	2, 	3, 	2, 	3 \rangle
+ 0.5 \langle	3,	1,	1,	3 \rangle
+ 0.2 \langle	2,	4,	4,	3 \rangle
+ 0.1 \langle	1,	4,	1,	3 \rangle,
\end{equation*}
which involves exactly 4 distinct BN matrices. This shows the sharpness of the bound in Theorem \ref{prop:P1} and the effectiveness of GER.

\begin{theorem}\label{prop:P2}
The length of any decomposition of $P_{2}$ is at least $6$.
\end{theorem}

\begin{proof}
Consider the 1\textsuperscript{st} and the 2\textsuperscript{nd} columns of $P_{2}$. Clearly, the 1\textsuperscript{st} column is not a permutation of the 2\textsuperscript{nd} column.
By Corollary \ref{coroll:permutation}, any decomposition $\sum^{K}_{k = 1} x_{k} A_{k}$ of $P_{2}$ satisfies $K > 5$.
\end{proof}

The decomposition of $P_{2}$ found using our Python implementation of GER (with the score parameter set to $10$) is
\begin{eqnarray*}
& & 
\frac{19}{110} \langle 6, 7, 2, 7, 5, 4, 8, 1 \rangle
+ \frac{3}{22} \langle 7, 3, 6, 4, 5, 1, 8, 1 \rangle
+ \frac{6}{55} \langle 1, 8, 1, 3, 3, 6, 5, 6 \rangle \\
& + &
\frac{1}{11} \langle 2, 8, 1, 1, 1, 5, 4, 6 \rangle
+ \frac{3}{11} \langle 3, 1, 8, 2, 2, 7, 1, 3 \rangle
+ \frac{12}{55} \langle 3, 2, 4, 2, 4, 7, 1, 7 \rangle,
\end{eqnarray*}
which involves exactly 6 distinct BN matrices. This shows the sharpness of the bound in Theorem \ref{prop:P2} and the effectiveness of GER.

\begin{theorem}\label{prop:P3}
The length of any decomposition of $P_{3}$ is at least $6$.
\end{theorem}

\begin{proof}
Consider the 3\textsuperscript{rd} and the 5\textsuperscript{th} columns of $P_{3}$. Clearly, the 3\textsuperscript{rd} column is not a permutation of the 5\textsuperscript{th} column. 
By Corollary \ref{coroll:permutation}, any decomposition $\sum^{K}_{k = 1} x_{k} A_{k}$ of $P_{3}$ satisfies $K > 5$.
\end{proof}

The decomposition of $P_{3}$ found using our Python implementation of GER (with the score parameter set to $10$) is
\begin{eqnarray*}
& & 
\frac{3}{11} \langle 5, 2, 6, 1, 2, 7, 4, 1 \rangle
+ \frac{12}{55} \langle 6, 5, 8, 4, 5, 1, 8, 6 \rangle
+ \frac{19}{110} \langle 6, 5, 6, 1, 4, 1, 6, 1 \rangle \\
& + &
\frac{3}{22} \langle 3, 8, 5, 3, 8, 5, 2, 3 \rangle
+ \frac{6}{55} \langle 8, 7, 2, 7, 7, 3, 3, 7 \rangle
+ \frac{1}{11} \langle 3, 7, 4, 7, 8, 3, 2, 8 \rangle,
\end{eqnarray*}
which involves exactly 6 distinct BN matrices. This shows the sharpness of the bound in Theorem \ref{prop:P3} and the effectiveness of GER.

\begin{theorem}\label{prop:P4}
The length of any decomposition of $P_{4}$ is at least $8$.
\end{theorem}

\begin{proof}
Consider the 4\textsuperscript{th} and the 5\textsuperscript{th} columns of $P_{4}$.
Clearly, the 4\textsuperscript{th} column is not a permutation of the 5\textsuperscript{th} column.
By Corollary \ref{coroll:permutation}, any decomposition $\sum^{K}_{k = 1} x_{k} A_{k}$ of $P_{4}$ satisfies $K > 7$.
\end{proof}

The decomposition of $P_{4}$ found using our Python implementation of GER (with the score parameter set to $10$) is
\begin{eqnarray*}
& &
\frac{49}{265} \langle 5,	11,	14,	3,	8,	9,	1,	2,	2,	14,	5,	15,	14,	8,	4,	7 \rangle \\
& + &
\frac{39}{265} \langle 9,	12,	2,	12,	14,	11,	13,	7,	6,	13,	11,	12,	4,	2,	8,	15 \rangle \\
& + &
\frac{59}{265} \langle 9,	11,	7,	11,	1,	9,	13,	7,	16,	14,	5,	3,	7,	2,	5,	6 \rangle \\
& + &
\frac{29}{265} \langle 11,	10,	7,	11,	13,	5,	10,	13,	1,	6,	12,	8,	7,	11,	14,	6 \rangle \\
& + &
\frac{26}{265} \langle 1,	7,	3,	6,	16,	14,	9,	16,	12,	12,	10,	5,	8,	9,	10,	4 \rangle \\
& + &
\frac{17}{265} \langle 7,	16,	15,	2,	10,	15,	4,	6,	8,	15,	14,	16,	15,	5,	2,	2 \rangle \\
& + &
\frac{37}{265} \langle 14,	7,	15,	14,	9,	15,	4,	15,	10,	8,	7,	5,	8,	9,	2,	1 \rangle \\
& + &
\frac{9}{265} \langle 14,	15,	5,	15,	9,	4,	14,	3,	10,	7,	8,	9,	3,	1,	9,	1 \rangle,
\end{eqnarray*}
which involves exactly 8 distinct BN matrices. This shows the sharpness of the bound in Theorem \ref{prop:P4} and the effectiveness of GER.

\begin{theorem}\label{prop:P5}
The length of any decomposition of $P_{5}$ is at least $7$.
\end{theorem}

\begin{proof}
Consider the 2\textsuperscript{nd}, the 11\textsuperscript{th} and the 15\textsuperscript{th} columns of $P_{5}$.
The positive entries of the 2\textsuperscript{nd} column are 
$a_{1} \coloneqq P_{5}(1, 2) = 0.6$, $a_{2} \coloneqq P_{5}(2, 2) = 0.1$,
$a_{3} \coloneqq P_{5}(3, 2) = 0.1$, $a_{4} \coloneqq P_{5}(4, 2) = 0.1$,
$a_{5} \coloneqq P_{5}(5, 2) = 0.1$.
The positive entries of the 11\textsuperscript{th} column are 
$b_{1} \coloneqq P_{5}(1, 11) = 0.3$, $b_{2} \coloneqq P_{5}(5, 11) = 0.2$,
$b_{3} \coloneqq P_{5}(10, 11) = 0.2$, $b_{4} \coloneqq P_{5}(13, 11) = 0.3$.
The positive entries of the 15\textsuperscript{th} column are 
$c_{1} \coloneqq P_{5}(1, 15) = 0.2$, $c_{2} \coloneqq P_{5}(5, 15) = 0.4$,
$c_{3} \coloneqq P_{5}(9, 15) = 0.2$, $c_{4} \coloneqq P_{5}(13, 15) = 0.2$.
Let $x_{1}, x_{2}, \ldots, x_{K}$ (positive real numbers) and $A_{1}, A_{2}, \ldots, A_{K}$ (distinct BN matrices) define a decomposition of $P_{5}$.

Note that both the 2\textsuperscript{nd} and the 9\textsuperscript{th} columns of $P_{5}$ are columns having the maximum number of positive entries. By Proposition \ref{prop:max_P(:,j)_lower_bound}, $K\ge 5$. Note that we can't apply Corollary~\ref{coroll:permutation} to conclude that $K>5$.  

Suppose that $K=5$. Apply Lemma~\ref{lemma:convenient_notations} to the 2\textsuperscript{nd} column of $P_{5}$, $x_r=a_1=0.6$ for some $r\in[5]$. WLOG, assume that $x_1=0.6$. But then $x_1A_1$ has a 0.6 in its 11\textsuperscript{th} column ($A_1$ is a BN matrix) whereby the maximum of the 11\textsuperscript{th} column of $P_{5}$ is at least $0.6$ which is absurd. Therefore $K>5$. Now, we are going to prove that $K \geq 7$. It suffices to show that $K \neq 6$.

Assume for a contradiction that $K = 6$. By Lemma \ref{lemma:convenient_notations}, there exist three partitions 
$\left\{ X^{2}_{1}, X^{2}_{2}, X^{2}_{3}, X^{2}_{4}, X^{2}_{5} \right\}$,
$\left\{ X^{11}_{1}, X^{11}_{2}, X^{11}_{3}, X^{11}_{4} \right\}$,
$\left\{ X^{15}_{1}, X^{15}_{2}, X^{15}_{3}, X^{15}_{4} \right\}$ of $[6]$ such that:
for all $l \in [5]$, $a_{l} = \sum_{r \in X^{2}_{l}} x_{r}$; 
for all $l \in [4]$, $b_{l} = \sum_{r \in X^{11}_{l}} x_{r}$; and
for all $l \in [4]$, $c_{l} = \sum_{r \in X^{15}_{l}} x_{r}$.

Consider the partition $\left\{ X^{2}_{1}, X^{2}_{2}, X^{2}_{3}, X^{2}_{4}, X^{2}_{5} \right\}$ of $[K] = [6]$. Clearly, four of these sets have cardinality $1$, and the remaining one has cardinality $2$. 
Consider the positive entries of the 2\textsuperscript{nd} and the 11\textsuperscript{th} columns of $P_{5}$. Note that $0.6 = a_{1} > \max (b_{1}, b_{2}, b_{3}, b_{4}) = 0.3$.
By Proposition \ref{prop:al_is_larger_than_max}, $\left|X^{2}_{1}\right| \geq 2$ and hence $\left|X^{2}_{1}\right| = 2$.
By re-labeling the indices of $x_{1}, x_{2}, \ldots, x_{6}$ and $A_{1}, A_{2}, \ldots, A_{6}$, we can assume that $X^{2}_{1}=\{1,2\}$ and hence $0.6=a_1=x_1+x_2$. Since $x_1,x_2\ge 0$, one of them is $\ge 0.3$, say it's $x_2$.
Let $p^{*} \in [4]$ such that $2 \in X^{11}_{p^{*}}$. Then, 
\begin{equation*}
0.3 \geq b_{p^{*}} = \sum_{r \in X^{11}_{p^{*}}} x_{r} \geq x_{2} \geq 0.3.
\end{equation*}
Hence, $x_{1} = x_{2} = 0.3$. Let $q_{1}, q_{2} \in [4]$ such that $1 \in X^{15}_{q_{1}}$  and $2 \in X^{15}_{q_{2}}$. Then, because
\begin{eqnarray*}
c_{q_{1}} & = & \sum_{r \in X^{15}_{q_{1}}} x_{r} \geq x_{1} = 0.3\quad \textrm{and} \\
c_{q_{2}} & = & \sum_{r \in X^{15}_{q_{2}}} x_{r} \geq x_{2} = 0.3,
\end{eqnarray*}
we deduce that $q_{1} = q_{2} = 2$. Hence, 
\begin{equation*}
0.4 = c_{2} = \sum_{r \in X^{15}_{2}} x_{r} \geq x_{1} + x_{2} = 0.6.
\end{equation*}
A contradiction is reached. Therefore, $K \neq 6$ and hence $K \geq 7$.
\end{proof}

The decomposition of $P_{5}$ found using our Python implementation of GER (with the score parameter set to $10$) is
\begin{eqnarray*}
& &
0.1 \langle 3,	2,	2,	1,	9,	6,	5,	1,	2,	10,	5,	1,	1,	14,	1,	5 \rangle \\
& + &
0.1 \langle 3,	3,	5,	6,	1,	5,	5,	1,	3,	10,	5,	1,	1,	13,	1,	9 \rangle \\
& + &
0.3 \langle 2,	1,	1,	5,	1,	5,	1,	1,	1,	5,	1,	5,	1,	5,	5,	2 \rangle \\
& + &
0.2 \langle 2,	1,	1,	5,	5,	1,	1,	5,	1,	1,	10,	1,	1,	1,	9,	13 \rangle \\
& + &
0.1 \langle 2,	1,	1,	5,	1,	1,	1,	5,	1,	1,	13,	9,	1,	13,	5,	13 \rangle \\
& + &
0.1 \langle 2,	4,	1,	5,	1,	1,	1,	5,	9,	1,	13,	9,	1,	13,	13,	13 \rangle \\
& + &
0.1 \langle 2,	5,	1,	5,	1,	1,	1,	5,	10,	1,	13,	9,	1,	13,	13,	13 \rangle,
\end{eqnarray*}
which involves exactly 7 distinct BN matrices. This shows the sharpness of the bound in Theorem \ref{prop:P5} and the effectiveness of GER.

\begin{theorem}\label{prop:PA_1}
The length of any decomposition of $P^{\textrm{A}}_{1}$ is at least $5$.
\end{theorem}

\begin{proof}
Let $x_{1}, x_{2}, \ldots, x_{K}$ (positive real numbers) and $A_{1}, A_{2}, \ldots, A_{K}$ (distinct BN matrices) define a decomposition of $P^{\textrm{A}}_{1}$.
Note that the 1\textsuperscript{st} and the 2\textsuperscript{nd} columns of $P^{\textrm{A}}_{1}$ each has $3$ positive entries. Moreover, the 1\textsuperscript{st} column is clearly not a permutation of the 2\textsuperscript{nd} column. By Corollary \ref{coroll:permutation}, $K > 3$. To show that $K \geq 5$, it suffices to show that $K \neq 4$.

Assume for a contradiction that $K = 4$. Consider the 1\textsuperscript{st}, the 2\textsuperscript{nd} and the 3\textsuperscript{rd} columns of $P^{\textrm{A}}_{1}$. 
The positive entries of the 1\textsuperscript{st} column are
$a_{1} \coloneqq P^{\textrm{A}}_{1}(1, 1) = 0.1$,
$a_{2} \coloneqq P^{\textrm{A}}_{1}(2, 1) = 0.2$,
$a_{3} \coloneqq P^{\textrm{A}}_{1}(4, 1) = 0.7$.
The positive entries of the 2\textsuperscript{nd} column are
$b_{1} \coloneqq P^{\textrm{A}}_{1}(1, 2) = 0.3$,
$b_{2} \coloneqq P^{\textrm{A}}_{1}(2, 2) = 0.3$,
$b_{3} \coloneqq P^{\textrm{A}}_{1}(4, 2) = 0.4$.
The positive entries of the 3\textsuperscript{rd} column are
$c_{1} \coloneqq P^{\textrm{A}}_{1}(1, 3) = 0.2$,
$c_{2} \coloneqq P^{\textrm{A}}_{1}(2, 3) = 0.2$,
$c_{3} \coloneqq P^{\textrm{A}}_{1}(3, 3) = 0.6$.
By Lemma \ref{lemma:convenient_notations}, there exist three partitions 
$\left\{ X^{1}_{1}, X^{1}_{2}, X^{1}_{3} \right\}$,
$\left\{ X^{2}_{1}, X^{2}_{2}, X^{2}_{3} \right\}$, 
$\left\{ X^{3}_{1}, X^{3}_{2}, X^{3}_{3} \right\}$
of $[K] = [4]$ such that for all $l \in [3]$,
$a_{l} = \sum_{r \in X^{1}_{l}} x_{r}$,
$b_{l} = \sum_{r \in X^{2}_{l}} x_{r}$ and
$c_{l} = \sum_{r \in X^{3}_{l}} x_{r}$.

Now, consider the 1\textsuperscript{st} column of $P^{\textrm{A}}_{1}$ and $\left\{ X^{1}_{1}, X^{1}_{2}, X^{1}_{3} \right\}$. Clearly, two of $X^{1}_{1}$, $X^{1}_{2}$, $X^{1}_{3}$ have cardinality $1$, and the remaining one has cardinality $2$. 
Note that $0.7 = a_{3} > \max (b_{1}, b_{2}, b_{3}) = 0.4$. By Proposition \ref{prop:al_is_larger_than_max}, $|X^{1}_{3}| \geq 2$.
Hence, $|X^{1}_{1}| = 1$, $|X^{1}_{2}| = 1$ and $|X^{1}_{3}| = 2$.
By re-labeling the indices of $x_{1}, x_{2}, x_{3}, x_{4}$ and $A_{1}, A_{2}, A_{3}, A_{4}$, we can assume that $X^{1}_{1} = \{ 1 \}$, $X^{1}_{2} = \{ 2 \}$ and $X^{1}_{3} = \{ 3, 4 \}$. Hence, $0.1 = a_{1} = x_{1}$, $0.2 = a_{2} = x_{2}$ and $0.7 = a_{3} = x_{3} + x_{4}$.

Next, we consider the partition $\left\{ X^{2}_{1}, X^{2}_{2}, X^{2}_{3} \right\}$ of $[4]$. Clearly, two of $X^{2}_{1}, X^{2}_{2}, X^{2}_{3}$ have cardinality $1$, and the remaining one has cardinality $2$.
Note that for all $l \in [3]$, $X^{2}_{l} \neq \{ 1 \}$; otherwise,
$0.3 \leq b_{l} = x_{1} = 0.1$ for some $l \in [3]$.
Using a similar argument, we can justify that for all $l \in [3]$, $X^{2}_{l} \neq \{ 2 \}$.
Therefore, $\left\{ X^{2}_{1}, X^{2}_{2}, X^{2}_{3} \right\} = \left\{ \{ 1, 2 \}, \{ 3 \}, \{ 4 \} \right\}$.
Hence, $X^{2}_{l_{1}} = \{ 3 \}$ and $X^{2}_{l_{2}} = \{ 4 \}$ for some $l_{1}, l_{2} \in [3]$. This implies that $b_{l_{1}} = x_{3}$ and $b_{l_{2}} = x_{4}$. Note that
\begin{equation*}
1 
= x_{1} + x_{2} + x_{3} + x_{4}
= 0.1 + 0.2 + b_{l_{1}} + b_{l_{2}},
\end{equation*}
which implies that $b_{l_{1}} + b_{l_{2}} = 0.7$. Because $b_{1} = 0.3$, $b_{2} = 0.3$ and $b_{3} = 0.4$, 
$\{ x_{3}, x_{4} \} = \{ b_{l_{1}}, b_{l_{2}} \} = \{ 0.3, 0.4 \}$.

Finally, we consider the partition $\left\{ X^{3}_{1}, X^{3}_{2}, X^{3}_{3} \right\}$ of $[4]$. Note that $3 \notin X^{3}_{1}$; otherwise,
\begin{equation*}
0.2 = c_{1} = \sum_{r \in X^{3}_{1}} x_{r} \geq x_{3} \geq 0.3.
\end{equation*}
Using a similar argument, we can justify that $3 \notin X^{3}_{2}$, $4 \notin X^{3}_{1}$ and $4 \notin X^{3}_{2}$. Hence, $3, 4 \in X^{3}_{3}$. Therefore,
\begin{equation*}
0.6 = c_{3} = \sum_{r \in X^{3}_{3}} x_{r} \geq x_{3} + x_{4} = 0.7.
\end{equation*}
A contradiction is reached. We conclude that $K \neq 4$ and hence $K \geq 5$.
\end{proof}

The decomposition of $P^{\textrm{A}}_{1}$ found using our Python implementation of GER (with the score parameter set to $10$) is
\begin{equation*}
0.1 \langle 	1,	1,	3,	1 \rangle
+ 0.2 \langle	2,	1,	1,	3 \rangle
+ 0.2 \langle	4,	2,	2,	3 \rangle
+ 0.4 \langle	4,	4,	3,	4 \rangle
+ 0.1 \langle	4,	2,	3,	4 \rangle,
\end{equation*}
which involves exactly 5 distinct BN matrices. This shows the sharpness of the bound in Theorem \ref{prop:PA_1} and the effectiveness of GER.

\begin{theorem}\label{prop:PA_2}
The length of any decomposition of $P^{\textrm{A}}_{2}$ is at least $5$.
\end{theorem}

\begin{proof}
The result follows from Proposition \ref{prop:copying_PBN} and Theorem \ref{prop:PA_1}.
\end{proof}

The decomposition of $P^{\textrm{A}}_{2}$ found using our Python implementation of GER (with the score parameter set to $10$) is
\begin{eqnarray*}
& & 
0.1 \langle 1,	1,	3,	1,	5,	5,	7,	5 \rangle
+ 0.2 \langle 2,	1,	1,	3,	6,	5,	5,	7 \rangle
+ 0.2 \langle 4,	2,	2,	3,	8,	6,	6,	7 \rangle \\
& + &
0.4 \langle 4,	4,	3,	4,	8,	8,	7,	8 \rangle
+ 0.1 \langle 4,	2,	3,	4,	8,	6,	7,	8 \rangle,
\end{eqnarray*}
which involves exactly 5 distinct BN matrices. This shows the sharpness of the bound in Theorem \ref{prop:PA_2} and the effectiveness of GER.

\begin{theorem}\label{prop:PB_1}
The length of any decomposition of $P^{\textrm{B}}_{1}$ is at least $4$.
\end{theorem}

\begin{proof}
Let $x_{1}, x_{2}, \ldots, x_{K}$ (positive real numbers) and $A_{1}, A_{2}, \ldots, A_{K}$ (distinct BN matrices) define a decomposition of $P^{\textrm{B}}_{1}$.
Consider the 1\textsuperscript{st} and the 2\textsuperscript{nd} columns of $P^{\textrm{B}}_{1}$.
Because $\mathcal{N}^{+}(P^{\textrm{B}}_{1}(:, 1)) = \mathcal{N}^{+}(P^{\textrm{B}}_{1}(:, 2)) = 2$ and the 1\textsuperscript{st} column is not a permutation of the 2\textsuperscript{nd} column, we deduce that $K > 2$ using Corollary \ref{coroll:permutation}.
Therefore, to prove that $K \geq 4$, it suffices to show $K \neq 3$.

Assume for a contradiction that $K = 3$. Define
\begin{eqnarray*}
a_{1} \coloneqq P^{\textrm{B}}_{1}(1, 1) = 0.1,
& &
a_{2} \coloneqq P^{\textrm{B}}_{1}(4, 1) = 0.9, \\
b_{1} \coloneqq P^{\textrm{B}}_{1}(1, 2) = 0.3,
& &
b_{2} \coloneqq P^{\textrm{B}}_{1}(2, 2) = 0.7, \\
c_{1} \coloneqq P^{\textrm{B}}_{1}(1, 3) = 0.5,
& &
c_{2} \coloneqq P^{\textrm{B}}_{1}(3, 3) = 0.5, \\
d_{1} \coloneqq P^{\textrm{B}}_{1}(1, 4) = 0.6,
& &
d_{2} \coloneqq P^{\textrm{B}}_{1}(4, 4) = 0.4, \\
X^{1}_{1} \coloneqq \{ r \in [K] : A_{r}(1, 1) = 1 \},
& &
X^{1}_{2} \coloneqq \{ r \in [K] : A_{r}(4, 1) = 1 \}, \\
X^{2}_{1} \coloneqq \{ r \in [K] : A_{r}(1, 2) = 1 \},
& &
X^{2}_{2} \coloneqq \{ r \in [K] : A_{r}(2, 2) = 1 \}, \\
X^{3}_{1} \coloneqq \{ r \in [K] : A_{r}(1, 3) = 1 \},
& &
X^{3}_{2} \coloneqq \{ r \in [K] : A_{r}(3, 3) = 1 \}, \\
X^{4}_{1} \coloneqq \{ r \in [K] : A_{r}(1, 4) = 1 \},
& &
X^{4}_{2} \coloneqq \{ r \in [K] : A_{r}(4, 4) = 1 \}.
\end{eqnarray*}
By Lemma \ref{lemma:convenient_notations}, for each $i \in [4]$, $\left\{ X^{i}_{1}, X^{i}_{2} \right\}$ is a partition of $[K] = [3]$. Therefore, one of $X^{i}_{1}, X^{i}_{2}$ has cardinality $1$ and the other one has cardinality $2$.

Note that $a_{2}$, $b_{2}$, $d_{1}$ are greater than $\max(c_{1}, c_{2}) = 0.5$.
Hence, using Proposition \ref{prop:al_is_larger_than_max}, we deduce that
$|X^{1}_{2}| = |X^{2}_{2}| = |X^{4}_{1}| = 2$. 
Therefore, $|X^{1}_{1}| = |X^{2}_{1}| = |X^{4}_{2}| = 1$. 
Write $X^{1}_{1} = \{ r_{1} \}$, $X^{2}_{1} = \{ r_{2} \}$ 
and $X^{4}_{2} = \{ r_{3} \}$.
By Lemma \ref{lemma:convenient_notations},
\begin{eqnarray*}
0.1 = a_{1} & = &
\sum_{r \in X^{1}_{1}} x_{r} = x_{r_{1}}, \\
0.3 = b_{1} & = &
\sum_{r \in X^{2}_{1}} x_{r} = x_{r_{2}}, \\
0.4 = d_{2} & = &
\sum_{r \in X^{4}_{2}} x_{r} = x_{r_{3}}.
\end{eqnarray*}
Because $x_{r_1} , x_{r_2} , x_{r_3} \in \{ x_i : i=1,2,3 \}$ are distinct, $\{ x_i : i=1,2,3 \} = \{ x_{r_1} , x_{r_2} , x_{r_3} \}$ so that
\begin{equation*}
x_{1} + x_{2} + x_{3}
= x_{r_{1}} + x_{r_{2}} + x_{r_{3}}
= 0.1 + 0.3 + 0.4 = 0.8,
\end{equation*}
but $x_{1} + x_{2} + x_{3}$ should equal $1$ because $\sum^{K}_{r = 1} x_{r} A_{r}$ is a decomposition of $P^{\textrm{B}}_{1}$. A contradiction is reached.
Therefore, $K \neq 3$ and hence $K \geq 4$.
\end{proof}

The decomposition of $P^{\textrm{B}}_{1}$ found using our Python implementation of GER (with the score parameter set to $10$) is
\begin{equation*}
0.5 \langle 		4,	2,	1,	1 \rangle
+ 0.1 \langle		1,	1,	3,	1 \rangle
+ 0.2 \langle		4,	1,	3,	4 \rangle
+ 0.2 \langle		4,	2,	3,	4 \rangle,
\end{equation*}
which involves exactly 4 distinct BN matrices. This shows the sharpness of the bound in Theorem \ref{prop:PB_1} and the effectiveness of GER.

\begin{theorem}\label{prop:PB_3}
The length of any decomposition of $P^{\textrm{B}}_{3}$ is at least $4$.
\end{theorem}

\begin{proof}
The result follows from Proposition \ref{prop:max_P(:,j)_lower_bound}.
\end{proof}

The decomposition of $P^{\textrm{B}}_{3}$ found using our Python implementation of GER (with the score parameter set to $10$) is
\begin{eqnarray*}
& & 
0.3 \langle 		1,	7,	7,	5,	3,	7,	5,	8 \rangle
+ 0.2 \langle		1,	7,	7,	1,	3,	7,	5,	8 \rangle \\
& + &
0.3 \langle 		1,	7,	7,	6,	3,	8,	6,	8 \rangle
+ 0.2 \langle 	1,	7,	7,	2,	3,	8,	6,	8 \rangle,
\end{eqnarray*}
which involves exactly 4 distinct BN matrices. This shows the sharpness of the bound in Theorem \ref{prop:PB_3} and the effectiveness of GER.

\begin{theorem}\label{prop:PB_4(d)}
Fix arbitrary $d \in \{ 0.01, 0.02, 0.03, 0.04 \}$.
The length of any decomposition of $P^{\textrm{B}}_{4}(d)$ is at least $5$.
\end{theorem}

\begin{proof}
We are going to prove this result by applying Proposition \ref{prop:trivial_lower_bound} to the 1\textsuperscript{st} and the 4\textsuperscript{th} columns of $P^{\textrm{B}}_{4}(d)$.
Consider the 4 positive entries $a_{1} \coloneqq 0.2$, $a_{2} \coloneqq 0.2$, $a_{3} \coloneqq 0.3$ and $a_{4} \coloneqq 0.3$ of the 4\textsuperscript{th} column, and the two positive entries $b_{1} \coloneqq 1 - d$ and $b_{2} \coloneqq d$ of the 1\textsuperscript{st} column.
Note that there does not exist a partition $\left\{ Y_{1}, Y_{2} \right\}$ of $[4]$ such that
$b_{1} = \sum_{s \in Y_{1}} a_{s}$ and $b_{2} = \sum_{s \in Y_{2}} a_{s}$.
To see this, assume the contrary. Let $s^{*}$ be an arbitrary element of $Y_{2}$. Then,
\begin{equation*}
0.04 \geq d = b_{2} = \sum_{s \in Y_{2}} a_{s} \geq a_{s^{*}} \geq 0.2,
\end{equation*}
which gives a contradiction.
Therefore, by Proposition \ref{prop:trivial_lower_bound}, the theorem is proved.
\end{proof}

For $d = 0.01, 0.02, 0.03, 0.04$, the decomposition of $P^{\textrm{B}}_{4}(d)$ found using our Python implementation of GER (with the score parameter set to $10$) is
\begin{eqnarray*}
& & 
d \langle 		4,	4,	4,	1,	4,	7,	5,	4 \rangle
+ 0.3 \langle 	1,	7,	7,	5,	3,	7,	5,	8 \rangle
+ (0.2 - d) \langle	1,	7,	7,	1,	3,	7,	5,	8 \rangle \\
& + &
0.3 \langle 		1,	7,	7,	6,	3,	8,	6,	8 \rangle
+ 0.2 \langle 	1,	7,	7,	2,	3,	8,	6,	8 \rangle,
\end{eqnarray*}
which involves exactly 5 distinct BN matrices. This shows the sharpness of the bound in Theorem \ref{prop:PB_4(d)} and the effectiveness of GER.

\begin{theorem}\label{prop:PB_6(d)}
Fix arbitrary $d \in \{ 0.01, 0.02, 0.03, 0.04 \}$.
The length of any decomposition of $P^{\textrm{B}}_{6}(d)$ is at least $5$.
\end{theorem}

\begin{proof}
The result follows from Proposition \ref{prop:copying_PBN} and Theorem \ref{prop:PB_4(d)}.
\end{proof}

For $d = 0.01, 0.02, 0.03, 0.04$, the decomposition of $P^{\textrm{B}}_{6}(d)$ found using our Python implementation of GER (with the score parameter set to $10$) is
\begin{eqnarray*}
& &
d 		\langle 4,	4,	4,	1,	4,	7,	5,	4,	12,	12,	12,	9,	12,	15,	13,	12 \rangle \\
& + &
0.3 		\langle 1,	7,	7,	5,	3,	7,	5,	8,	9,	15,	15,	13,	11,	15,	13,	16 \rangle \\
& + &
(0.2 - d)	\langle 1,	7,	7,	1,	3,	7,	5,	8,	9,	15,	15,	9,	11,	15,	13,	16 \rangle \\
& + &
0.3 		\langle 1,	7,	7,	6,	3,	8,	6,	8,	9,	15,	15,	14,	11,	16,	14,	16 \rangle \\
& + &
0.2 		\langle 1,	7,	7,	2,	3,	8,	6,	8,	9,	15,	15,	10,	11,	16,	14,	16 \rangle,
\end{eqnarray*}
which involves exactly 5 distinct BN matrices. This shows the sharpness of the bound in Theorem \ref{prop:PB_6(d)} and the effectiveness of GER.

\end{document}